\documentclass[twoside,11pt]{article}
\RequirePackage{amsmath,amsfonts,amssymb}

\usepackage{amssymb}
\usepackage{mathrsfs}
\usepackage{amsfonts}
\usepackage{amsmath}
\usepackage{graphicx}
\usepackage{multirow}
\usepackage{float}
\usepackage{booktabs}
\usepackage{subfigure}

\usepackage{makecell}
\usepackage[table]{xcolor}

\newtheorem{lemma}{Lemma}
\newtheorem{assumption}{Assumption}

\newcommand{\bc}{\mathbf{c}}

\newcommand{\M}{\mbox{M}}

\usepackage{titlesec}
\titleformat{\section}{\large\bfseries}{\thesection}{1em}{}
\titleformat{\subsection}{\normalsize\bfseries}{\thesubsection}{1em}{}
\titleformat{\subsubsection}{\normalsize\bfseries}{\thesubsubsection}{1em}{}

\usepackage{jmlr2e}
\usepackage{amsmath}
\usepackage{tikz}
\usepackage{subfigure,epic}
\usepackage{multirow}
\usepackage{array}
\usepackage{booktabs}
\usepackage{algorithm}
\usepackage{algpseudocode}
\usepackage{makecell}
\usepackage{multirow}



\usepackage{lastpage}
\jmlrheading
{}{}{1-\pageref{LastPage}}{; Revised
}{}{}
{Xiaotong Liu, Yunwen Lei, Xiangyu Chang, and Shao-Bo Lin}
\ShortHeadings{Adaptive parameter selection for KGD}{Liu, Lei, Chang, and Lin}
\firstpageno{1}

\begin{document}

\title{Beyond Cross-Validation: Adaptive Parameter Selection  for  Kernel-Based Gradient Descents}

\author{\name Xiaotong Liu   \email  	ariesoomoon@gmail.com\\
\addr Center for Intelligent Decision-Making
        and Machine Learning,\\
        School of Management,\\
       Xi'an Jiaotong University, Xi'an, China
       \AND
\name Yunwen Lei
\email leiyw@hku.hk\\
\addr Department of Mathematics,\\
The University of Hong Kong, Hongkong, China
       \AND
       \name Xiangyu Chang \email  
       xiangyuchang@gmail.com\\
       \name Shao-Bo Lin\thanks{Corresponding author} \email sblin1983@gmail.com\\
\addr Center for Intelligent Decision-Making
        and Machine Learning,\\
        School of Management,\\
       Xi'an Jiaotong University, Xi'an, China
}

\editor{}

\maketitle 

\begin{abstract}
This paper proposes a novel parameter selection strategy for kernel-based gradient descent (KGD) algorithms, integrating bias–variance analysis with the splitting method. We introduce the concept of empirical effective dimension to quantify iteration increments in KGD, deriving an adaptive parameter selection strategy that is implementable. Theoretical verifications are provided within the framework of learning theory. Utilizing the recently developed integral operator approach, we rigorously demonstrate that KGD, equipped with the proposed adaptive parameter selection strategy, achieves the optimal generalization error bound and adapts effectively to different kernels, target functions, and error metrics. Consequently, this strategy showcases significant advantages over existing parameter selection methods for KGD.
\end{abstract}


\section{Introduction}

Parameter selection, often referred   as hyperparameter selection, is a critical and longstanding research topic in statistics and machine learning. 
Selecting appropriate parameters in learning algorithms significantly impacts the accuracy, efficiency, and generalization capabilities of the resulting models. For instance, selecting suitable regularization parameters enables support vector machines \citep{steinwart2008support} to maintain a balance between bias and variance, ensuring excellent performance on unseen data, while tuning appropriate step sizes in stochastic gradient descent in deep learning \citep{bengio2017deep} succeeds in balancing the training time and model quality.

Parameter selection generally requires strategic approaches to ensure that the chosen parameters align with the specific goals and constraints of the learning task, triggering substantial research activity over the past two decades \citep{chapelle2002choosing,de2010adaptive,caponnetto2010cross,raskutti2014early,wei2019early,blanchard2019lepskii,lu2020balancing, gizewski2022regularization, dinu2023addressing, lin2024adaptive}. 
Broadly speaking, there are three basic ideas to parameter selection: the information entropy method, the splitting method, and the bias–variance analysis method. The information entropy method seeks to determine optimal parameters by balancing fitting accuracy with entropy-related complexity constraints and is particularly popular in linear methods \citep{reiss2009smoothing}. Typical examples of the information entropy method include the Akaike information criterion \citep{akaike1974new, Clifford1998Smoothing}, the Bayesian information criterion \citep{schwarz1978estimating, Drton2017Bayesian}, and the generalized information criterion \citep{nishii1984asymptotic}. The main advantage of this method is its ease of implementation, while a notable disadvantage is the  difficulty in deriving provable generalization error bounds for non-linear algorithms.
The splitting method focuses on selecting parameters by dividing the samples into a training set and a validation set, using the former to build models and the latter for parameter selection. Typical examples of this method include hold-out \citep{caponnetto2010cross}, cross-validation \cite[Chap.8]{gyorfi2002distribution}, and leave-one-out \citep{zhang2003leave}. The advantages of the splitting method lie in its versatility, as it is applicable to nearly all learning algorithms, while the disadvantages stem from the exclusion of a portion of samples during the training process, which can lead to inflated generalization error. The bias--variance analysis method involves selecting parameters by balancing bias and variance, which are assessed using data-dependent and implementable quantities. Typical examples of this method include the balancing principle \citep{lu2020balancing}, the Lepskii principle \citep{blanchard2019lepskii}, and the discrepancy principle \citep{celisse2021analyzing}. This method is advantageous for theoretical analysis but poses implementation challenges, as deriving precise bounds with optimal constants for bias and variance is  quite difficult.

Due to its implementability, rapidity and versatility, the splitting method, especially cross-validation and hold-out, dominates in the modern statistics and machine learning. Statistical  properties of hold-out, cross-validation and generalized cross-validation have been well developed by \citet{caponnetto2010cross}, \citet{xu2019distributed}, and \citet[Chap.8]{gyorfi2002distribution}, respectively. However, 
the splitting method frequently requires that the empirical excess risk is an accessible unbiased estimate of the population risk \citep{blanchard2019lepskii}, making the corresponding parameter selection weak in tackling covariate shift problem in which the distributions of training and testing samples are different. 
Moreover, the commonly used truncation (or clip) operator that projects the derived estimates onto a finite interval \citep{gyorfi2002distribution}, not only makes parameter selection only available to bounded samples, but also deports the final models from the original hypothesis classes.
The primary reason for these drawbacks  is the lack of bias--variance analysis to highlight the detailed effects of parameters on bias and variance, which is crucial for addressing the covariate shift problem \citep{ma2023optimally} and accessing unbounded samples \citep{blanchard2016convergence}.

This paper aims to equip the classical splitting method with a delicate bias--variance analysis to circumvent the aforementioned drawbacks. 
Taking kernel-based gradient descent (KGD) \citep{lin2018distributed,yao2007early} (also referred to as kernel-based boosting in the literature \citep{raskutti2014early,wei2019early}) as an example,   we propose a novel hybrid selection scheme (HSS) that combines the hold-out \citep{caponnetto2010cross} and Lepskii-type principle  \citep{blanchard2019lepskii,lin2024adaptive}  to adaptively determine the number of iterations for KGD with optimal theoretical verifications.
According to the delicate bias--variance analysis, HSS first quantifies the bias of KGD by the increments between two successive iterations and the variance by  the empirical effective dimension \citep{zhang2002effective}. Then, taking the monotonicity of the involved quantities into account, it searches for the optimal number of iterations in a backward manner, i.e., from $|D|$, $|D|-1$ to 1. In this way, a novel bias--variance analysis method with an unspecified constant independent of the data size is developed. HSS finally determines the constant by using the splitting method such as hold-out (or cross-validation) on a subset of samples.  Under this circumstance, we obtain a parameter selection strategy that adapts to the kernels, target functions, and different metrics without discarding any samples.

Our main contributions can be concluded as follows. From the theoretical perspective,  we prove  that KGD equipped with the proposed HSS succeeds in achieving the optimal generalization error bounds established in \citep{lin2018distributed}, which {demonstrates} the effectiveness and optimality of HSS. This overcomes the bottlenecks of numerous parameter selection schemes such as the balancing principle \citep{lu2020balancing} and discrepancy principle \citep{celisse2021analyzing}
that only lead to suboptimal generalization error bounds.  Our main novelty in the proof is  the introduction of several semi-adaptive stopping rules to quantify the range of the number of iterations determined by HSS.
From the application perspective, the proposed HSS combines the bias–variance analysis and splitting methods to leverage their advantages in parameter selection and performs well in practice. We conduct both toy simulations and real data experiments to show the excellent performance of HSS, compared with hold-out \citep{caponnetto2010cross}, Akaike information criterion \citep{akaike1974new}, Bayesian information criterion \citep{schwarz1978estimating}, balancing principle \citep{lu2020balancing}, Lepskii principle \citep{blanchard2019lepskii},  early stopping rule \citep{raskutti2014early}, and discrepancy principle \citep{celisse2021analyzing}.

The rest of the paper is organized as follows. In the next section, we
introduce HSS to equip KGD and compare our method with existing parameter selection strategies. Section \ref{Sec.Main-result} explores some important properties of HSS and derives the optimal generalization error bounds for KGD equipped with HSS.  In Section \ref{Sec.Experiments}, we conduct  numerical experiments  to demonstrate the feasibility of the proposed stopping rule. Section \ref{Sec.Conclusion} provides further discussion. The proofs of our theoretical results are given in Appendix.

\section{KGD with hybrid selection strategy}\label{sec2:problem_setting}

In this section, we introduce KGD with a hybrid selection strategy and   compare it with some related works.

\subsection{KGD with hybrid selection strategy}

Let $D=\{(x_i,y_i)\}_{i=1}^{|D|}\subset \mathcal{X}\times\mathcal{Y}$ be a sample set with $|D|$ the cardinality of $D$, $\mathcal{X}$ a compact input space and $\mathcal{Y}\subseteq\mathbb{R}$ an output space.  Let $K:\mathcal{X}\times\mathcal{X}\rightarrow \mathbb{R}$ be a Mercer kernel and $\mathcal H_K$ be its corresponding reproducing kernel Hilbert space (RKHS) endowed with the inner product $\langle\cdot,\cdot\rangle_K$ and norm $\|\cdot\|_K$. Since  $\mathcal X$ is compact and $K$ is a Mercer kernel, it is easy to {check}   $\kappa:=\sqrt{\sup_{x\in\mathcal{X}} K(x,x)}<\infty$.
Kernel methods aim at finding a function $f\in\mathcal H_K$ by minimizing the empirical risk
\begin{equation}\label{ERM}
\mathcal L(f):=\frac{1}{|D|}\sum_{i=1}^{|D|}(y_i-f(x_i))^2.
\end{equation}
The gradient of $\mathcal{L}(f)$ with respect to $f\in\mathcal H_K$ can be derived by \citet{yao2007early}
\begin{eqnarray}\label{Gradient}
\nabla_{f}\mathcal{L} (f)
=
\frac{2}{|D|}\sum_{i=1}^{|D|}(f(x_i)-y_i)K_{x_i}.
\end{eqnarray}
Given a step size {$\beta>0$,} KGD for (\ref{ERM}) then can be formulated iteratively as
\begin{eqnarray}\label{Gradient Descent algorithm}
f_{t+1,\beta,D} 
=  f_{t,\beta,D}-\frac{\beta}{|D|}\sum_{i=1}^{|D|}(f_{t,\beta,D}(x_i)-y_i)K_{x_i},
\end{eqnarray}
where $f_{0,\beta,D} =0$ and $K_{x_i}(\cdot)=K(x_i,\cdot)$.
It is obvious that $f_{t,\beta,D}\in\mathcal H_K$ with the special format $f_{t,\beta,D}=\sum_{i=1}^{|D|}c^t_iK_{x_i}$. Let ${\bf c}_t=(c_1^t,\dots,c_{|D|}^t)^{\top}$, $y_D=(y_1,\dots,y_{|D|})^{\top}$ and $\mathbb K=(K(x_i,x_{i'}))_{x_i,x_{i'}\in D}$ be the kernel matrix.
It follows from (\ref{Gradient Descent algorithm}) that
\begin{equation}\label{Parametric_KGD}
\bc_{t+1}=\bc_t-\frac{\beta}{|D|}(\mathbb{K}\bc_t-y_D)
\end{equation}
with $\bc_0=(0,\dots,0)^{\top}$. Therefore, KGD requires $\mathcal O(|D|^2)$ floating-point computations in each iteration. It was shown in \citep{lin2018distributed, yao2007early} that KGD is  an optimal learner {in the  sense of achieving the optimal generalization error bounds}, provided the number of iterations is {appropriately tuned}.  It  has been widely used in regression \citep{yao2007early}, classification \citep{wei2019early}, minimum error entropy principle  analysis \citep{hu2020distributed} and testing \citep{hagrass2024spectral}.

\begin{figure}[!t]
	\centering
	\setlength{\subfigcapskip}{-0.5em}
	\subfigure{\includegraphics[scale=0.30]{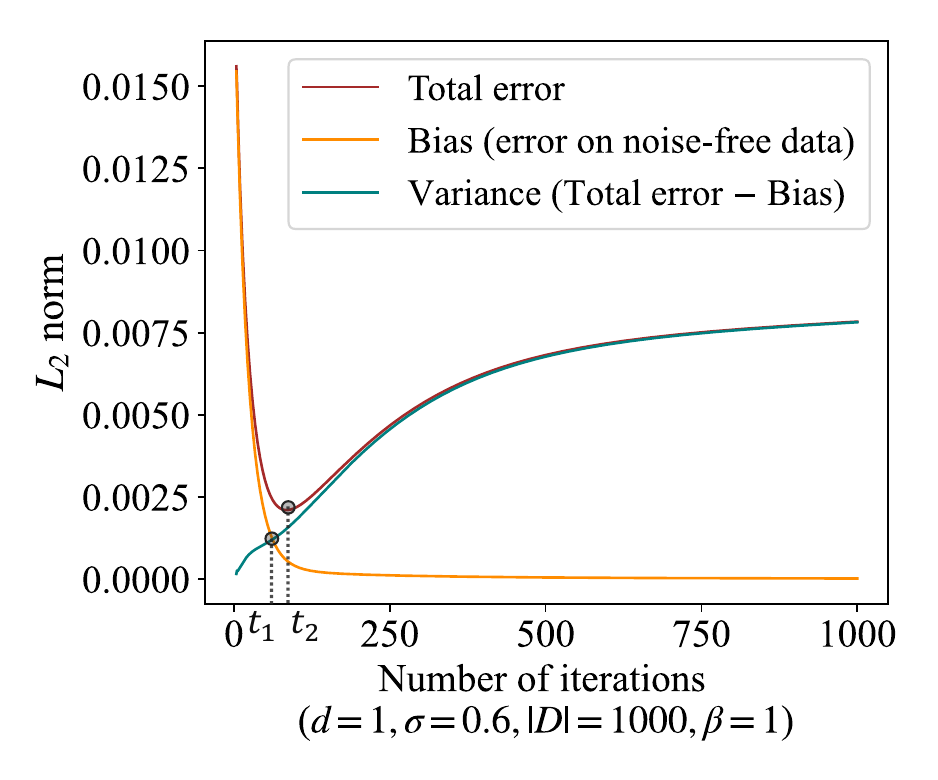}	\label{fig:bias--variance_d1}} 
	\setlength{\subfigcapskip}{-0.5em}
	\hspace{0.01in} 
	\subfigure{\includegraphics[scale=0.30]{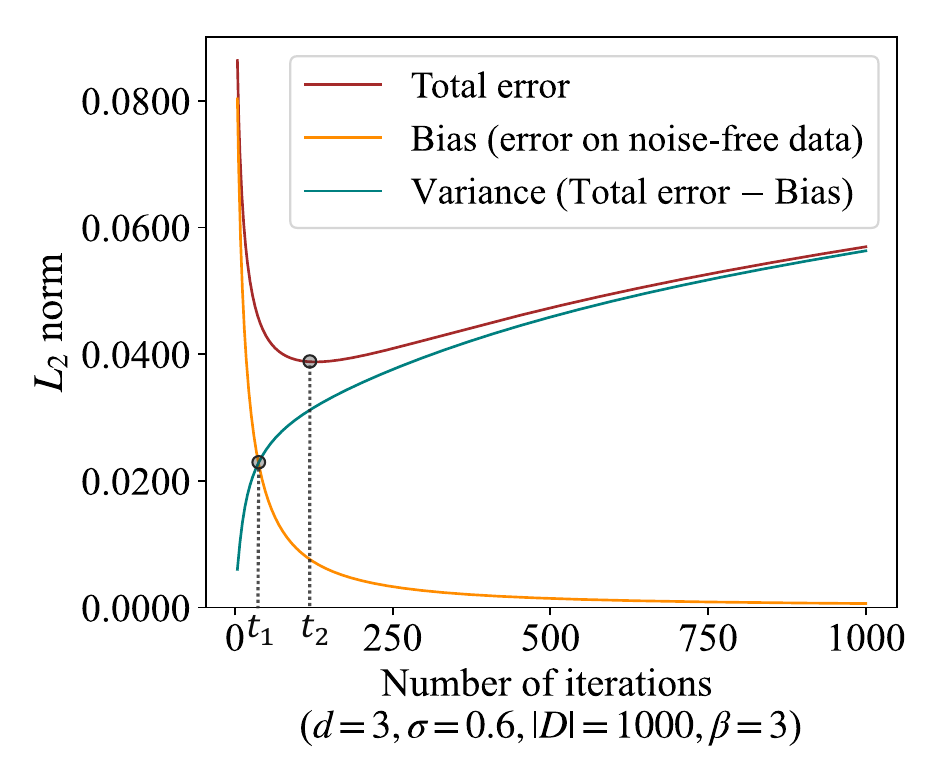}\label{fig:bias--variance_d3}} 
	\vspace{-0.1in}
	\caption{\footnotesize Role of the number of iterations in controlling the bias, variance and total error of KGD under the $L_2$ norm. The training samples $\{x_i\}_{i=1}^{1000}$ are independently  drawn according to the uniform distribution on the (hyper-)cube $[0,1]^d$ with $d=1, 3$. The corresponding outputs are generated by the model $y_i=g_j(x_i)+\varepsilon_i, j=1,2$, where $g_1$ and $g_2$ are defined by (\ref{g1}) and (\ref{g2}), respectively, and $\varepsilon_i$ is the independent Gaussian noise $\mathcal N(0,
		\sigma^2)$ with $\sigma=0.6$. The bias and variance is defined by \eqref{Error-dec.1} below.}\label{fig:bias--variance}
\end{figure}

As shown in Figure \ref{fig:bias--variance}, the number of iterations, $t$, plays a crucial role in the sense that the best generalization performance of KGD is realized by setting the best $t$ to balance the bias and variance. Our purpose is  to propose a novel    strategy  to  adaptively  select  $t$ to approximate the best $t$. Our   approach depends on a novel bias--variance analysis principle, called the  backward selection principle (BSP), which is determined by the so-called
empirical effective dimension of the kernel matrix \citep{zhang2002effective},  defined by
\begin{equation}\label{Definition empi effec}
\mathcal N_{D}(\lambda): ={\rm Tr}[(\lambda|D|I+\mathbb K)^{-1}\mathbb K], \quad\forall \lambda>0,  
\end{equation}
where  ${\rm Tr}$ denotes the trace of a matrix (or  trace-class operator).
Denote 
\begin{eqnarray}\label{Def.WD}
\mathcal W_{D,t}&:=&
\left( \frac{\sqrt{t}}{|D|}+\frac{\sqrt{\max\{\mathcal N_D(t^{-1}),1\}}(1+\sqrt{t/|D|})}{\sqrt{|D|}}\right)
\end{eqnarray}
and
\begin{small}
\begin{align}\label{Def.U}
\mathcal{U}_{D,t,\delta} := & \left(
\frac{\log \left(1+8\log \frac{64}{\delta} \frac{\sqrt{t}}{\sqrt{|D|}} 
\max\{1,\mathcal{N}_D(t^{-1})\}\right)}{t^{-1}|D|}
\right. \nonumber \\
& \quad + \left.
\sqrt{\frac{\log \left(1+8\log \frac{64}{\delta} \frac{\sqrt{t}}{\sqrt{|D|}} 
\max\{1,\mathcal{N}_D(t^{-1})\}\right)}{t^{-1} |D|}}
\right)  
\end{align}
\end{small}
for any given   $\delta\in(0,1)$.
Define further
\begin{equation}\label{Def.T}
T:=T_{\delta,D}:=\max_{1\leq t\leq |D|}\left\{C_1^*\mathcal U_{D,t,\delta}\leq 1/2\right\} 
\end{equation}
with $C_1^*:=\max\{(\kappa^2+1)/3,2\sqrt{\kappa^2+1}\}$. Since $\mathcal U_{D,t,\delta}$ increases with respect to $t$,
 the above  $T$ is well defined.
We then define BSP for KGD as follows.

        
\begin{definition}[Backward Selection Principle (BSP)]\label{def:asr}
	Given the confidence level $\delta\in(0,1)$ and the step size $\beta$, define $\hat{t}:=\hat{t}_{D,\beta}\in[1,{T}]$ as the largest integer satisfying
	\begin{equation}\label{stopping 1}
	t\|f_{t+1,\!\beta,\!D}\!-\!f_{t,\beta,D}\|_{D}
	+t^{1/2}\|f_{t+1,\beta,D}\!-\!f_{t,\beta,D}\|_{K}\!\geq\!
	\tilde{C}\mathcal W_{D, {t}}\!\log^2\frac{16}\delta,
	\end{equation} 
	where $\tilde{C}$  is a constant depending only on $\kappa$ and the noise variance.  
 If there is not any $t\in[1,T]$ satisfying (\ref{stopping 1}), we then set $\hat{t}=T$. 
\end{definition}
        
Figure \ref{fig:C_exist} illustrates the feasibility of BSP in the sense that  a given constant $\tilde C$ follows a specific $t_{\tilde C}$. However, an excessively small or large $\tilde C$ will directly lead to $t_{\tilde C}=T$. 
It is easy to check that the proposed BSP is a special type of Lepskii principle, with the only difference that the item-wise comparison in the classical principle \citep{blanchard2019lepskii, lu2020balancing}  is removed. It should also be highlighted that BSP is not an early stopping scheme like   that in \citep{raskutti2014early}, since BSP requires to run KGD from $t=1$ to $t=T$ at first and then find the largest $t$ satisfying \eqref{stopping 1}.
Due to (\ref{Parametric_KGD}), there are {${\bf c}^{t+1}=(c^{t+1}_1,\dots,c^{t+1}_{|D|})^{\top}$ and ${\bf c}^{t}=(c^{t}_1,\dots,c^{t}_{|D|})^{\top}$} such that
$$
f_{t+1,\beta,D}-f_{t,\beta,D}
=
\sum_{i=1}^{|D|}(c^{t+1}_i-c^{t}_i)K_{x_i}.
$$
Then,
    \begin{small}
\begin{align*}
	\|f_{t+1,\beta,D}\!-\!f_{t,\beta,D}\|_K^2 &\!=\! \sum_{i=1}^{|D|}\sum_{j=1}^{|D|}(c^{t+1}_i \!-\! c^{t}_i)(c^{t+1}_j \!-\! c^{t}_j)K(x_i,x_j) 
	= ({\bf c}^{t+1} \!-\! {\bf c}^{t})^{\top} \mathbb{K} ({\bf c}^{t+1} \!-\! {\bf c}^{t})
\end{align*}
    \end{small}
and
        \begin{small}
\begin{align*}
	\|f_{t+1,\beta,D}\!-\!f_{t,\beta,D}\|_{D}^2 &\!=\! \frac{1}{|D|} \sum_{i=1}^{|D|} \left(f_{t+1,\beta,D}(x_i) \!-\! f_{t,\beta,D}(x_i)\right)^2 
	= \frac{1}{|D|} ({\bf c}^{t+1} \!-\! {\bf c}^{t})^{\top} \mathbb{K}^2 ({\bf c}^{t+1} \!-\! {\bf c}^{t}),
\end{align*}
    \end{small}
which together with \eqref{Definition empi effec} and \eqref{Def.WD} makes the BSP in Definition \ref{def:asr}   implementable.  


 \begin{figure}[!t]
	\centering
	\setlength{\subfigcapskip}{-0.5em}
	\subfigure{\includegraphics[scale=0.3]{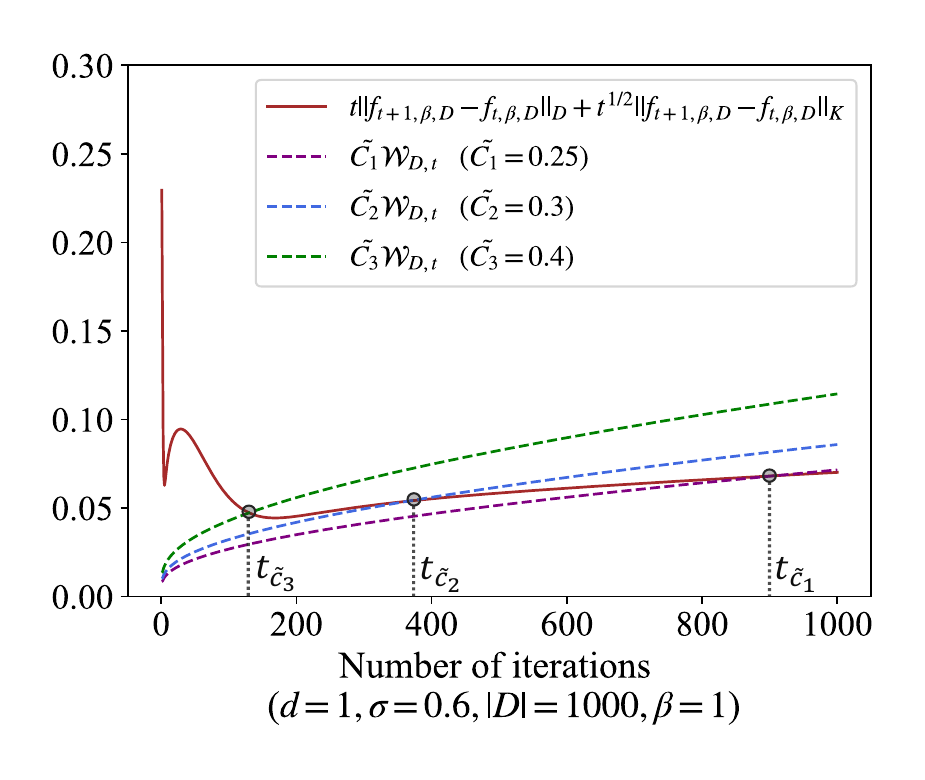}	\label{fig:C_exist_d1}} 
	\setlength{\subfigcapskip}{-0.5em}
	\hspace{0.01in} 
	\subfigure{\includegraphics[scale=0.3]{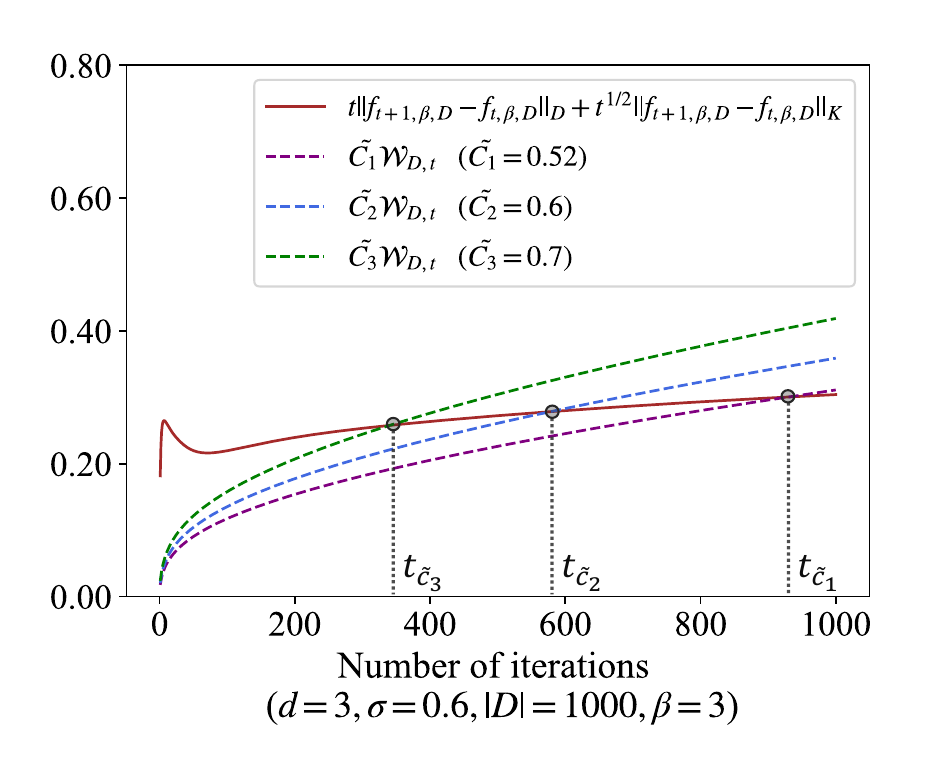}\label{fig:C_exist_d3}} 
	\vspace{-0.05in}
	\caption{\footnotesize Step $t_{\tilde C}$ determined by different values of the constant $\tilde C$ under the BSP. The experimental setting is the same as in Figure \ref{fig:bias--variance}.}\label{fig:C_exist}
\end{figure}

Due to (\ref{Def.WD}), the parameter selection strategy requires to compute the effective dimension $\mathcal N_D(t^{-1})={\rm Tr}[(t^{-1}|D|I+\mathbb K)^{-1}\mathbb K]$ for different $t$. Noting that the eigenvectors for the matrices $\mathbb K$ and $(t^{-1}|D|I+\mathbb K)^{-1}$ are {the same}, if  the eigenvalues of  $\mathbb K$, denoted by $\{\sigma_1^{\bf x},\dots,\sigma_{|D|}^{\bf x}\}$ in a decreasing order, are obtained, then the empirical effective dimension for any $t$ can be obtained by  $\mathcal N_D(t^{-1})=\sum_{i=1}^{|D|}\frac{\sigma_i^{\bf x}}{\sigma_i^{\bf x}+t^{-1}|D|}$. Hence, BSP requires to compute eigenvalues of  $\mathbb K$ and  needs $\mathcal O(|D|^3)$, which is similar to that in the early stopping rule proposed in \citep{raskutti2014early}, although early stopping in \citep{raskutti2014early} forces termination of the iteration before $t=|D|$.

In implementing BSP, we  need to determine the constants $\kappa, M$ and the confidence level $\delta$. It is well known that estimating the noise variance $M$ is not easy, but we can utilize the approach in \citep{raskutti2014early} to approximately compute it, which makes BSP numerically feasible. 
	The main problem is that the derived constant $\tilde{C}$ in (\ref{stopping 1})  is generally not  optimal, 
	leading BSP to be  far from  optimal.
	Fortunately, we find that the constant, as well as the  confidence level $\delta$ that could be specified before the learning process to reflect the user's confidence, are independent of the data size. 
    Therefore, we  can  use the hold-out or cross-validation approaches to select it by running KGD with BSP on a small dataset,  for example, drawing $|D|/10$ samples from $D$. In this way, we actually obtain an adaptive parameter selection strategy by combining BSP with cross-validation. We name this novel strategy the hybrid selection strategy (HSS), as shown in Algorithm \ref{alg:KGD with BSP}.

\begin{algorithm}[t]
	\small 
\begin{algorithmic}\caption{\small KGD with HSS}
		\label{alg:KGD with BSP}
		\State {\bf Inputs}: $D=\{(x_i,y_i)\}_{i=1}^m$, set of constants $\mathcal C_U:=\{\hat{C}_j\}_{j=1}^U$ with $U\in\mathbb N$, step size $0<\beta\leq\frac1\kappa$ and the kernel $K$. 
		\State {\bf Data Split}: For $L\leq |D|$,  select $L$ samples from $D$ randomly according to the uniform distribution. Divide the selected samples into the training set  $D_{tr,L}$ and validation set
		$D_{val,L}$.
		\State {\bf Compute the empirical effective dimension}: Compute   eigenvalues of two kernel matrices $\mathbb K=(K(x_i,x_{i'}))_{x_i,x_{i'}\in D}$ and $\mathbb K_L=(K(x_i,x_{i'}))_{x_i,x_i'\in D_{tr,L}}$ and denote them by $\{\sigma_i^{\bf x}\}_{i=1}^{|D|}$ and $\{\sigma_{i,L}^{\bf x}\}_{i=1}^{|D|}$, respectively. Compute $\mathcal W_{D,t}$ and $\mathcal W_{D_{tr,L},t}$ via (\ref{Definition empi effec}) and (\ref{Def.WD}) for $t=1,\dots,|D|$.
		\State {\bf Compute the sudden stopping value $T$:} Given a confidence level $\delta\in(0,1)$, 
		compute $\mathcal U_{D,t,\delta}$ according to \eqref{Def.U} and then determine  $T$ via \eqref{Def.T}.
		\State {\bf KGD}: Set $f_{0,\beta,D}=0$ and $f_{0,\beta,D_{tr,L}}=0$. Obtain two sequences of KGD estimators $\{f_{t+1,\beta,D}\}_{t=0}^{|D|-1}$ and $\{f_{t+1,\beta,D_{tr,L}}\}_{t=0}^{|D|-1}$ via (\ref{Gradient Descent algorithm}).
		\State {\bf Selecting the constant}: Define $\hat{t}_j\in[1,T]$ to be the largest integer satisfying
\begin{multline}\label{stopping 1.1}
t\|f_{t+1,\beta,D_{tr,L}}\!-\!f_{t,\beta,D_{tr,L}}\|_{D_{tr,L}} + 
t^{1/2} \|f_{t+1,\beta,D_{tr,L}}\!-\!f_{t,\beta,D_{tr,L}}\|_{K} \geq \hat{C}_j \mathcal{W}_{D_{tr,L},t}.
\end{multline}
		If (\ref{stopping 1.1}) does not hold for all $t\in[1,T]$, set $\hat{t}_j=T$.
		Define
	\begin{equation}\label{validation-coeffi}
		j^*:=\arg\min_{j=1,\dots,U} \frac{1}{{|D_{val,L}|}}\sum_{(x,y)\in D_{val,L}}(f_{\hat{t}_j,\beta,D_{tr,L}}(x)-y)^2.
		\end{equation}
		\State {\bf BSP for final parameter}: Define $\hat{t}^*\in[1,|D|]$ to be the largest integer satisfying
		\begin{equation}\label{stopping 1.2}
		t\|f_{t+1,\beta,D}-f_{t,\beta,D}\|_{D}
		+t^{1/2}\|f_{t+1,\beta,D}-f_{t,\beta,D}\|_{K}\geq \hat{C}_{j^*}
		\mathcal W_{D, {t}}.
		\end{equation} 
		\State {\bf Output}: $f_{\hat{t}^*,\beta,D}$.    
	\end{algorithmic}
\end{algorithm}

We then present several comments on Algorithm \ref{alg:KGD with BSP} as follows.

$\diamond$ \textbf{Constant candidates}: 
HSS should   quantify the set $\mathcal C_U$ before the training. Recalling (\ref{stopping 1}), we note that $\mathcal C_U$ possesses an upper bound $4(1+\beta)\tilde{C}\log^4\frac{16}\delta$. A favorable setting is to set $\mathcal C_U=\{4(1+\beta)\tilde{C}\log^4\frac{16}\delta q^k:k=0,1,\dots,\} $ for some $0<q< 1$. The problem is, however, that  $4(1+\beta)\tilde{C}\log^4\frac{16}\delta$ itself is not easy to compute, since it requires to obtain the noise of the data. Numerically, we can set 
\begin{equation}\label{CU}
\mathcal C_U:=\left\{\tilde{C}_0q^k:k=0,1,\dots,U\right\}
\end{equation}
for a sufficiently large $\tilde{C}_0$, i.e., $\tilde{C}_0=100$, $q=9/10$ and $U=20$   for example. 

$\diamond$ \textbf{Computation reduction}: since the constant $4(1+\beta)\tilde{C}\log^4\frac{16}\delta$ in (\ref{stopping 1}) is independent of the data and the selection of $L\leq |D|$ samples in the parameter selection strategy is only for selecting a suitable constant to feed BSP,   the  subsampling scheme in the data split step  does not  affect the selection of the final constants in theory. In computing the empirical effective dimension step, it is not necessary to compute all $\mathcal W_{D,t}$ and $\mathcal W_{D_{tr,L},t}$. Indeed, we can select $t$ backward from $|D|$ to 1 and stop the search once (\ref{stopping 1.1}) and (\ref{stopping 1.2}) are satisfied. Under this circumstance, we 
only need to compute $\mathcal W_{D,t}$ and  $\mathcal W_{D_{tr,L},t}$ for $t$ not smaller than the value determined by (\ref{stopping 1.1}) and (\ref{stopping 1.2}).  Recalling that $\mathcal N_D(t^{-1})=\sum_{i=1}^{|D|}\frac{\sigma_i^{\bf x}}{\sigma_i^{\bf x}+t^{-1}|D|}$, computing $\mathcal W_{D,t}$ and $\mathcal W_{D_{tr,L},t}$ for all $t$ does not incur much additional computation, we suggest computing all of them in the algorithm. Once the constant $C_j^*$ is determined, we then use the whole sample to derive a final parameter $\hat{t}^*$ to feed KGD. If we set $L\leq |D|/U^{1/3}$, it is easy to check that there are totally $\mathcal O(|D|^3)$ floating-point computations in implementing KGD with HSS.

$\diamond$ \textbf{Early-stopping and BSP}: Different from \citep{raskutti2014early} that equips KGD with an early stopping scheme, our proposed strategy in Algorithm \ref{alg:KGD with BSP} requires to run KGD over all $t\in [1,T]$. However, it should be highlighted that all existing early-stopping parameter selection strategies \citep{raskutti2014early,wei2019early} require to compute the eigenvalues of  $\sigma_i^{\bf x}$ for $i=1,\dots,|D|$, needing $\mathcal O(|D|^3)$ computational complexity. Under this circumstance, whether an early stopping parameter selection strategy is employed  does not affect the computational complexity essentially. Our proposed HSS strategy, compared with those in \citep{raskutti2014early,wei2019early}, dominates in reflecting the regularity of data precisely. 

$\diamond$ \textbf{Lepskii principle versus discrepancy principle}: Lepskii and discrepancy principles are two popular principles used to design parameter selection strategy for kernel methods \citep{de2010adaptive,lu2020balancing,blanchard2019lepskii,celisse2021analyzing,lin2024adaptive}. The former focuses on selecting parameters by bounding  differences of two successive estimates while the latter is devoted to  quantifying the fitting error by some computable quantities. They are  essentially different approaches for general kernel methods \citep{blanchard2019lepskii,celisse2021analyzing}. However, due to the special spectral property of KGD \citep{lin2018distributed}, i.e., $  
    f_{t+1,\beta,D}-f_{t,\beta,D} = 
       \beta \left(L_{K,D} f_{t,\beta,D} - S_D^{\top} y_D \right),    
$
Lepskii and discrepancy principles are the same for KGD. That is, the philosophy behind HSS is both  the Lepskii principle and discrepancy principle.

\subsection{Related works}
In this part, we compare the proposed adaptive  stopping rule with some related works. KGD   was initially studied  in \citep{yao2007early,gerfo2008spectral} as a special class of kernel-based spectral algorithms to conquer the  saturation of kernel ridge regression \citep{caponnetto2007optimal}.  Optimal generalization error bounds of KGD were established in \citep{yao2007early,raskutti2014early,lin2018distributed}  under different assumptions on  data and kernels.   However, all these  optimal generalization error bounds were established on the condition that the number of iterations is  {appropriately selected}. Although several stopping rules including the discrepancy principle \cite[Chap.6]{engl1996regularization} and balancing principle \citep{lazarov2007balancing} have been proposed to determine the number of iterations in the  framework of inverse problems, it is highly non-trivial to develop stopping rules for KGD in the nonparametric regression community. In the following, we introduce existing strategies  to select the number of iterations  for KGD.

$\bullet$ {Baseline (BS)}:  The number of iterations is determined by
\begin{equation}\label{Eq:oracle_stopping_rule}
t_{BS}:={\arg\min}_{t\in[0,|D|]}\|f_{t,\beta,D}-f_\rho\|_D^2.
\end{equation}
{The baseline defined by (\ref{Eq:oracle_stopping_rule})  is a theoretically optimal  {parameter selection} strategy to determine $t$ in KGD under the assumption that the exact knowledge of the target   function $f_\rho$ is known. It
	numerically performs  well  but requires the access of $f_\rho$, which is  practically infeasible.  We introduce this method only to provide
	a reference to judge the quality of other implementable parameter selection strategies.}

$\bullet$  {Hold-out (HO)}: The number of iterations is determined by
\begin{equation}\label{Eq:houldout_stopping_rule}
t_{HO}:={\arg\min}_{t\in[0,|D|]} \M_{HO}(f_{t,\beta,D_{tr}}).
\end{equation}
To be detailed, the hold-out (or cross-validation CV) \citep{caponnetto2010cross,lin2018optimal, zhang2003leave} method is a numerically feasible method for selecting $t$ from $[0,|D|]$ to achieve  optimal generalization error bounds of KGD. It requires splitting the sample set $D$ into training and validation sets, denoted as $D_{tr}$ and $D_{val}$, respectively, where $D=D_{tr}\cup D_{val}$, $D_{tr}\cap D_{val}=\emptyset$ and $D_{tr}$ contains half of the whole sample data. It then evaluates the performance of KGD on $D_{val}$ via $\M_{HO}(f_{t,\beta,D_{tr}}):=\frac{1}{|{D_{val}}|}\sum_{(x_i, y_i)\in {D_{val}}}(y_i-f_{t,\beta,D_{tr}}(x_i))^2$ and {selects} $t$   to minimize the validation error.
The optimality of HO for KGD can be easily derived by using  the same method as in \citep{lin2018optimal}. The problem is, however, {that} the split of samples  reduces the numerical performance of KGD itself.

$\bullet$ {Balancing principle (BP)}: Given the confidence level $\delta$, the number of iterations is determined by
{\small 
\begin{multline}\label{Eq:bp_stopping_rule}
t_{BP}:=
\min \{t\in[0,|D|]:\|f_{t',\beta,D}-f_{t,\beta,D}\|_D \leq 
\left. C_{BP} \mathcal W_{D, {t'}}\log^4\frac{16}\delta, \; t'=t+1,\dots,|D| \right\},
\end{multline}
}
where  $\mathcal W_{D,t}$ is defined by (\ref{Def.WD}) and $C_{BP}>0$ is a constant independent of $|D|$, $\delta$ or $t$. 
The balancing principle was initially proposed in \citep{de2010adaptive} in the framework of learning theory to select a regularization parameter for KRR and then developed in \citep{lu2020balancing} for determining parameters of kernel-based spectral algorithms. 
It can be found in (\ref{Eq:bp_stopping_rule}) that the balancing principle does not require a split of the sample set. 

$\bullet$ {Lepskii principle (LP)}:  Given $q>1$ and $0<\delta<1$, define    
$T_q:=\{t_i= q^i/(\kappa^2):i\in\mathbb N\}$,     $L_{\delta,q}=2\log\left(\frac{8\log|D|}{\delta\log q}\right)$ and
    {\small
$
T_{\bf x}:=\left\{t\in T_q: t\leq \max\left\{\frac{|D|}{100\kappa^2L_{\delta,q}^2},\frac{|D|}{3\kappa^2(\mathcal N_{D}(t^{-1})+1)}\right\}\right\}.
$
}
The number of iterations is determined by
    {\small
\begin{equation}\label{Eq:lp_stopping_rule}
t_{LP} \!:=\! \min \left\{t \in T_{\bf x} : \|(L_{K,D} \!+\! t'^{-1}I)^{1/2} (f_{t',\beta,D} \!-\! f_{t,\beta,D})\|_K \!\leq\! 
 C_{LP} \mathcal W^*_{D, {t'}}, \; t' \geq t, \; t' \in T_{\bf x} \right\}.
\end{equation}}
where $\mathcal  W^*_{D, {t}}=\frac{\sqrt{t}(\mathcal N_D(t^{-1})+1)}{\sqrt{|D|}}.$
The Lepskii principle was first proposed in \citep{blanchard2019lepskii} to improve the performance of balancing principle in supervised learning.
The oracle property  of the Lepskii principle was verified in  \citep{blanchard2019lepskii}. In particular,  it can be found in \citep{blanchard2019lepskii} that KGD equipped with the Lepskii principle can {achieve} the almost optimal generalization error bounds. Similar to BP, LP also needs  item-wise comparisons  for a given constant $C_{LP}$. Since $C_{LP}$ is practically unknown, it frequently requires heavy computation.

$\bullet$ {Early stopping rule (ESR)}:   Define the local empirical Rademacher complexity of the kernel matrix $\mathbb{K}$ to be
$\hat{\mathcal{R}}_{\mathbb{K}}(\epsilon):=\Big[\frac{1}{|D|}\sum_{i}\min\{\sigma^{\bf x}_i,\epsilon^2\}\Big]^{1/2}.$ Denote $\eta_t=t\beta$. {Let} $t_{ESR}$ be the first positive number satisfying
\begin{equation}\label{Eq:dsr_stopping_rule}
\hat{\mathcal{R}}_{\mathbb{K}}(1/\sqrt{\eta_t})>C_{ESR}(\tau\eta_t)^{-1},
\end{equation}
where $\tau>0$ is the standard deviation of noise and  can {be estimated} by the classical method  in \citep{hall1990variance}, and $C_{ESR}>0$ is a tuning parameter, suggested to be $2e$ by \citet{raskutti2014early}.
The work in \citep{raskutti2014early} is the first result, to the best of our knowledge, that designs an exclusive  early stopping rule for KGD with theoretical guarantees. Different from the above four methods, ESR does not require to run KGD over $[1,|D|]$. The problem is, however, that such an early stopping rule is not adaptive to the regularity of the  target function in the sense that it only holds under the assumption that the target function is in $\mathcal H_K$.

$\bullet$  Discrepancy principle (DP): 
Given a noise level $\sigma^2$, the discrepancy principle, which is a classical and widely used strategy to select parameters in inverse problems, is defined 
by \citet[Sec.4.3]{engl1996regularization},
	$$
t_{DP}:=\inf_{t\in[1,|D|]}\{\|y_D-S_Df_{t,\beta,D}\|_{\ell^2}\leq C_{DP}\sigma^2\},
$$	
where $C_{DP}$ is a constant independent of $t$ or $|D|$.  However, it is non-trivial to quantify the noise level $\sigma^2$ for supervised learning, which hinders the usage of DP. The interesting work \citep{celisse2021analyzing}  modified   the discrepancy principle 
for the   learning purpose by introducing  the computation of the noise level and an emergency stop rule.
Theoretical analysis has also been conducted in \citep{celisse2021analyzing} to verify the effectiveness of the discrepancy principle in the framework of learning theory. The problem is, however, that the derived generalization error bounds are suboptimal  {\cite[Theorem 20]{celisse2021analyzing}.}

$\bullet$   Akaike information criterion (AIC): Let $W_{D,t}$ be defined by \eqref{Def.WD} to mimic the degree of freedom in the kernel learning framework. AIC,  introduced by \citet{akaike1974new} for linear methods, aims to introduce an additional   penalty to balance the fitting accuracy and model complexity and can be mathematically defined by 
$$
    t_{AIC}:=\inf_{t\in[1,|D|]}\|y_D-S_Df_{t,\beta,D}\|_{\ell^2}+C_{AIC}\mathcal W_{D,t}.
$$
	
Although the defined penalty in the above definition is slightly  different from that in \citep{demyanov2012aic}, the main idea remains the same, i.e., using some implementable quantities to measure the degree of freedom to act as the penalty. 
Some statistical properties of this type of AIC  have been investigated in \citep{claeskens2008information}, but deriving provable generalization error bounds for the corresponding algorithms remains challenging. 
In this case, we cannot theoretically verify the adaptivity of the defined AIC.

$\bullet$ Bayesian information criterion (BIC): BIC was originally introduced by \citet{schwarz1978estimating} for linear methods and can be extended to KGD via 
$$
    t_{BIC}:=\inf_{t\in[1,|D|]}\|y_D-S_Df_{t,\beta,D}\|_{\ell^2}+C_{BIC}\mathcal W_{D,t}\log |D|.
$$
Compared with AIC, BIC includes an additional $\log|D|$ term in the penalty, and its statistical properties, similar to those of AIC, have been discussed in \citep{demyanov2012aic}. Mathematically, it is easy to derive that $t_{BIC}\leq t_{AIC}$. However, similar to AIC, it is also difficult to theoretically  verify the adaptivity of BIC in terms of deriving optimal generalization error bounds for the corresponding KGD. 

\section{Theoretical behaviors}\label{Sec.Main-result}
In this section, we present theoretical verifications  of  the proposed   HSS for KGD.  
    
\subsection{Bias and variance analysis of KGD}
We conduct the bias--variance analysis of  KGD in the framework of  nonparametric regression \citep{gyorfi2002distribution}, in which
the {samples} $\{(x_i,y_i)\}_{i=1}^{|D|}$ are assumed to be independently and identically drawn from an unknown distribution $\rho:=\rho_X\times\rho(y|x)$ with $\rho_X$ the marginal distribution and $\rho(y|x)$ the conditional distribution on $x$. The aim of learning is to find an estimator to approximate the  regression function $f_\rho=\int_{\mathcal Y}yd\rho(y|x)$, which minimizes the generalization error $\mathcal E(f):=\int(f(x)-y)^2d\rho$. 
Let $L_{\rho_X}^2$ denote the space of $\rho$-square integrable functions endowed with norm $\|\cdot\|_\rho$. It is easy to check 
    \begin{small}
\begin{equation}\label{equality}
\mathcal E(f)-\mathcal E (f_\rho)=\|f-f_\rho\|_\rho^2,\qquad f\in  L_{\rho_X}^2.
\end{equation}
    \end{small}

The learning performance of KGD has been extensively studied in \citep{yao2007early, raskutti2014early,lin2018distributed}
under  assumptions on noise, regularity of the regression function, and  capacity of the kernel.   
\begin{assumption}\label{Assumption:boundedness}
	({\bf Noise assumption:})  Assume $\int_{\mathcal Y}
	y^2d\rho<\infty$ and
        {\small
	\begin{equation}\label{Boundedness for output}
	\int_{\mathcal Y}\!\left(e^{\frac{|y-f_\rho(x)|}M}\!-\!\frac{|y-f_\rho(x)|}M\!-\!1\right)d\rho(y|x)\leq
	\frac{\gamma^2}{2M^2},  \forall x\!\in\mathcal\!\! X,
	\end{equation}}
	where $M$ and $\gamma$ are positive constants. 
\end{assumption}

Condition
(\ref{Boundedness for output}) is the standard Bernstein assumption on the noise \citep{caponnetto2007optimal,blanchard2016convergence}, which can be  
satisfied when the noise is uniformly bounded,
Gaussian or sub-Gaussian. To introduce the  regularity assumption, we define the following positive operator
$L_K:\mathcal H_K\rightarrow\mathcal H_K$ (or $L_{\rho_X}^2\rightarrow L_{\rho_X}^2$ when there is no ambiguity) as
$
L_K(f):=\int_X f(x)K_x d\rho_X.
$
The positivity is guaranteed  since $K$ is a  Mercer kernel. 
\begin{assumption}\label{Assumption:regularity}
	({\bf Regularity assumption:}) There exists an $r>0$ such that 
            \begin{small}
	\begin{equation}\label{regularitycondition}
	f_\rho=L_K^r (h_\rho),\qquad h_\rho\in L_{\rho_X}^2,\quad r>0,
	\end{equation}
        \end{small}
	where $L_K^r$ is the  $r$-th power of $L_K$.
    	
\end{assumption}
Assumption \ref{Assumption:regularity} has been widely used in  literature \citep{caponnetto2007optimal,yao2007early,blanchard2016convergence,guo2017learning,lin2017distributed} and the exponent $r$ quantifies the regularity of $f_\rho$. It should be noted {that} (\ref{regularitycondition}) with $r=1/2$ implies $f_\rho\in\mathcal H_K$ and vice versa \citep{caponnetto2007optimal}. Moreover, a larger $r$ indicates better regularity of the regression functions.

We  then study the role of iterations in the learning process  from the bias and variance perspective to pursue the selection of optimal $t$ for KGD.
Let
$
f_{t,\beta,D}^\diamond:=E[f_{t,\beta,D}|x]
$
be the noise-free version of $f_{t,\beta,D}$. {It} follows from (\ref{Gradient Descent algorithm}) that
    \begin{small}
    \begin{equation}\label{noise-free}
f^\diamond_{t+1,\beta,D}
= f^\diamond_{t,\beta,D}-\frac{\beta}{|D|}\sum_{i=1}^{|D|}(f^\diamond_{t,\beta,D}(x_i)-f_\rho(x_i))K_{x_i}.
\end{equation}
\end{small}
Then, the triangle inequality yields
        \begin{small}
\begin{equation}\label{Error-dec.1}
\|f_{t+1,\beta,D}\!-\!f_\rho\|_\rho   \!\!\leq   \!\! \overbrace{\|f^\diamond_{t+1,\beta,D}\!-\!f_\rho\|_\rho}^{Bias}+\overbrace{\|f^\diamond_{t+1,\beta,D}\!-\!f_{t+1,\beta,D}\|_\rho}^{Variance}.
\end{equation}
\end{small}

The following lemma which can be found in \citep{lu2020balancing,lin2019boosted} establishes a relation between $\|f\|_\rho$, $\|f\|_D$ and $\|f\|_K$.
\begin{lemma}\label{Lemma: in-sample--out-sample}
	Let $f\in\mathcal H_K$. Then
        {\small 
	\begin{equation}\label{in-sample to out-sample}
	\|f\|_\rho \leq \mathcal Q_{D,t}(\|f\|_D+
	t^{-1/2}\|f\|_K),
	\end{equation}}
    	
	where
        {\small 
	\begin{equation}\label{Def.Q}
	\mathcal Q_{D,t}:=\|(L_{K,D}+t^{-1} I)^{-1/2} (L_K+t^{-1}
	I)^{1/2}\| 
	\end{equation}}
	and $\|A\|$ denotes the spectral norm of the operator $A$.
\end{lemma}

Due to the above lemma,
to study the role of $t$ in bounding  $\|f_{t,\beta,D}-f_\rho\|_\rho$, it suffices to analyze its effect on
$\|f_{t,\beta,D}-f_\rho\|_D+ t^{-1/2}\|f_{t,\beta,D}-f_\rho\|_K$.
From (\ref{Error-dec.1}), the bias and variance can be bounded by
\begin{equation}\label{DEF.BIAS}
\mathcal B_{t,\beta,D}:=\|f^\diamond_{t,\beta,D}-f_\rho\|_D+ t^{-1/2}\|f^\diamond_{t,\beta,D}-f_\rho\|_K
\end{equation}
and
\begin{equation}\label{Def.Variance}
\mathcal V_{t,\beta,D}:=\|f^\diamond_{t,\beta,D}-f_{t,\beta,D}\|_D+
t^{-1/2}\|f^\diamond_{t,\beta,D}-f_{t,\beta,D}\|_K.
\end{equation}

The following proposition illustrates how $t$ controls the bias and variance via \eqref{Def.WD} and the effective dimension $\mathcal{N}(\lambda)$ {\citep{zhang2002effective}} defined by
$\mathcal{N}(\lambda):={\rm Tr}[(\lambda I+L_K)^{-1}L_K]$ with $\lambda>0$.

\begin{proposition}\label{Proposition:bias--variance-via-t}
	Let $0<\delta<1$, $f_{t,\beta,D}$ and $f^\diamond_{t,\beta,D}$ be defined by (\ref{Gradient Descent algorithm}) and (\ref{noise-free}) with  $t\leq T$ and $0<\beta\leq {\kappa^{-1}}$.
Under Assumption \ref{Assumption:boundedness} and Assumption \ref{Assumption:regularity} with $r\geq 1/2$, with confidence $1-\delta$, there hold
\begin{small}
\begin{equation}\label{bias-via-t}
	\mathcal{B}_{t,\beta,D} \!\leq\! \log^2\!\frac{16}{\delta}\! 
	\left\{
	\begin{aligned}
	\!	&2^{r-\frac{1}{2}} \, t^{-r}, && \!\!\!\!\!\hspace{-1.5em} \textnormal{if } \frac{1}{2} \leq \!r \!\leq \!1,\\
		\! &t^{-r}\!+\! 4\kappa^2 t^{-\frac{1}{2}} |D|^{-\min\{1/2,(r/2-1/4)\}}, && \hspace{-0.7em} \textnormal{if } r\! >\! 1
	\end{aligned}
	\right.
\end{equation}
\end{small}
and
        {\small
\begin{eqnarray}\label{Variance-via-t}
    \mathcal V_{t,\beta,D}&\leq&
     2\sqrt{2}(\kappa M +\gamma)(1+\beta) \mathcal A_{D,\lambda} \log
		\frac{8}{\delta} 
        \leq
	2(1+\beta)\mathcal W_{D,t}\log^2\frac{16}\delta,
\end{eqnarray}}
where
    {\small
\begin{equation}\label{Adlambda}
     \mathcal A_{D,t}:=
   \frac1{\sqrt{|D|}}\left(\frac{\sqrt{t}}{\sqrt{|D|}}+\sqrt{\mathcal
		N(t^{-1})}\right).
\end{equation}}
\end{proposition}

Proposition \ref{Proposition:bias--variance-via-t} presents the trends of bias and variance as the number of iterations $t$ changes. It can be found in (\ref{Def.WD}) and  (\ref{Variance-via-t}) that the {upper bound of variance} increases with the number of iterations.  On the contrary, 
	 (\ref{bias-via-t}) shows
    that the {upper bound of bias}  is monotonically decreasing with   respect to $t$.  Therefore, there exists an optimal $t$ minimizing the generalization error of KGD. Let $\{(\sigma_\ell,
\varphi_\ell)\}_{\ell}$ be a set of normalized eigenpairs of $L_K$ on
$\mathcal H_K$ with $\{\varphi_\ell\}_{\ell}$ being an
orthonormal basis of $\mathcal H_K$ and $\{\sigma_\ell\}_{\ell}$ arranged in a non-increasing order. 
Based on Proposition \ref{Proposition:bias--variance-via-t}, we can derive the following corollary to show the generalization error bounds for KGD with appropriately tuned parameter $t$ under different decaying rates of $\sigma_\ell$.
\begin{corollary}\label{Colloary:generalization-error-abc}
Let $0<\delta<1$, $f_{t,\beta,D}$    be defined by (\ref{Gradient Descent algorithm})  with  $t\leq T$ and $0<\beta\leq {\kappa^{-1}}$.
Under Assumption \ref{Assumption:boundedness} and Assumption \ref{Assumption:regularity} with $r\geq 1/2$,  
if
    {\small
\begin{equation}\label{def.t0-aaa}
  t\sim\left\{\begin{array}{cc}
      (|D|/L)^{\frac1{2r}},    &  \textnormal{if} \ \sigma_\ell=0,\ell\geq L+1,\\
       |D|^{\frac{1}{2r+s}},   & \textnormal{if} \ \sigma_\ell\leq c_0\ell^{-1/s},\\
      (|D|/\sqrt{\log|D|})^{\frac{1}{2r}}, &  \textnormal{if} \ \sigma_\ell\leq c_1e^{-c_2\ell^2},
     \end{array}\right.  
\end{equation}}
then with confidence $1-\delta$, there holds
    {\small
\begin{eqnarray}\label{optimal-optimal-parameter}
  &&\max\{\|f_{D,t}-f_\rho\|_\rho,  \|f_{D,t}-f_\rho\|_D, t^{1/2}\|f_{D,t}-f_\rho\|_K\} \nonumber\\
  &\leq&
  \tilde{C}
  \log^2\frac{16}{\delta}\left\{\begin{array}{cc}
     \!\!\!\! (|D|/L)^{-1/2},  & \textnormal{if} \ \sigma_\ell=0,\ell\geq L+1,    \\
    \!\! \!\! |D|^{-r/(2r+s)} & \!\!\!\!\!\! \!\!\! \textnormal{if} \ \sigma_\ell\leq c_0\ell^{-1/s},\\
     \!\! (|D|/\sqrt{\log|D|})^{-1/2},& \!\!\!\!\! \!\!\! \textnormal{if} \ \sigma_\ell\leq c_1e^{-c_2\ell^2},
  \end{array}\right.
\end{eqnarray}}
where $\tilde{C}$ is a constant independent of $\delta,|D|$, or $t$.
\end{corollary}

It can be found in  \citep{caponnetto2007optimal,raskutti2012minimax}   that the established generalization error bound for  KGD is optimal in the minimax sense, and the learning performance of KGD is independent of the step size $\beta$, provided $0<\beta\leq \kappa^{-1}$. The aforementioned theoretical verification indicates the superiority of KGD,  but selecting an appropriate parameter $t$ to optimize its learning capability  remains open since the theoretically optimal parameter   $t_0$ satisfying \eqref{def.t0-aaa}   cannot be directly obtained in practice. From \eqref{def.t0-aaa}, we find that the theoretically optimal $t$ for the finite rank and exponential decaying  $L_K$ is independent of $s$. Hence, if the kernel and $\rho_X$ are specified to make $L_K$ be finite rank or exponential decaying, it is not necessary to make the stopping rule adaptive to kernels. Recalling \eqref{optimal-optimal-parameter}, the generalization error bounds under the $\|\cdot\|_\rho$ norm for finite rank and exponential decaying  $L_K$ are independent of $r$. However, for the $\mathcal H_K$ norm, the index $r$ is crucial. For example, if $L_K$ is of finite rank $L$, then we can derive from \eqref{optimal-optimal-parameter} that
$
    \|f_{D,t}-f_\rho\|_K\leq \tilde{C}\left(\frac{L}{|D|}\right)^\frac{1-1/2r}{2}\log^2\frac{16}{\delta},
$
which has not been considered in  \citep{raskutti2014early,zhang2015divide} as only    $r=1/2$  is analyzed.

Since the computation of bias and variance in Proposition \ref{Proposition:bias--variance-via-t} requires access to the regression function that is unknown in practice, 
we adopt an alternative approach by leveraging the iterative nature of KGD.
Our basic idea is  to control the error increments of successive iterations in the following proposition.

\begin{proposition}\label{proposition:iterative error}
	Let $0<\delta<1$ and  $f_{t,\beta,D}$ be defined by (\ref{Gradient Descent algorithm}) with $t\in\mathbb N$ and $0<\beta\leq{\kappa^{-1}}$.  Under Assumption \ref{Assumption:boundedness} and Assumption \ref{Assumption:regularity}
	with $r\geq 1/2$, with confidence $1-\delta$, there holds
\begin{small}
\begin{align}\label{err-doc-itera-eef1}
    & \|f_{t+1,\beta,D}-f_{t,\beta,D}\|_{D}
    + t^{-1/2} \|f_{t+1,\beta,D}-f_{t,\beta,D}\|_{K} \nonumber \\
    & \leq 2\sqrt{2} \log^2\frac{16}{\delta} \, (\beta+1) \, t^{-1} \, \mathcal{W}_{D,t} 
     + C_1' \log^2\frac{16}{\delta} \left\{
    \begin{array}{ll}
       \! \!\! 2^{r-1/2} \, t^{-r-1}, &\! \!\!\hspace{-2.0em} \textnormal{if } \frac{1}{2} \leq r \leq 1, \\
     \! \!\!  t^{-r-1} \!+\! 4\kappa^2 t^{-3/2} |D|^{-\min\{1/2, r/2 \!-\! 1/4\}}, & \hspace{-0.5em} \textnormal{if } r > 1,
    \end{array}
    \right.
\end{align}
\end{small}
	where $C_1'$ is a constant independent of $|D|$, $t$ or $\delta$.
\end{proposition}

Comparing Proposition \ref{proposition:iterative error} with Proposition \ref{Proposition:bias--variance-via-t}, we find that   $t\|f_{t+1,\beta,D}-f_{t,\beta,D}\|_{D}
+t^{1/2}\|f_{t+1,\beta,D}-f_{t,\beta,D}\|_{K}$ behaves the same as
$\mathcal V_{t,\beta,D}+\mathcal B_{t,\beta,D}$ with respect to $t$.   The only difference is that $\|f_{t+1,\beta,D}-f_{t,\beta,D}\|_{D}$ and $\|f_{t+1,\beta,D}-f_{t,\beta,D}\|_{K}$ are numerically realizable but $\|f_{t,\beta,D}-f_\rho\|_D$ or $\|f_{t,\beta,D}-f_\rho\|_K$ {is incomputable}. Therefore, it is possible to quantify the increments of successive iterations to develop a parameter selection strategy like \eqref{stopping 1} to yield perfect theoretical results of the generalization error $\|f_{t,\beta,D}-f_\rho\|_\rho$.

\subsection{Semi-adaptive stopping rules and upper bound of iterations}
To study   the property of HSS, especially the upper bound of $\hat{t}$,  we   introduce   a semi-adaptive stopping rule for KGD based on Proposition \ref{proposition:iterative error}.   Define $t^*:=t^*_{D,\beta}\in[1,T]$ to be the  largest integer satisfying
\begin{equation}\label{Def.prelimiary-tstar}
    {\small
\mathcal W_{D,t}
\leq  
C_2'\left\{\begin{array}{cc}
2^{r-1/2}  t^{-r}, & \textnormal{if}\ \frac12\leq r\leq
\frac32,\\
t^{-r}+4\kappa^2t^{-1/2}|D|^{-1/2},& \textnormal{if}\ r>\frac32,
\end{array}\right. }
\end{equation}
where $C_2':=\max\{\tilde{C},(2(\beta+1) \tilde{C})^{-1}C_1'\}$.
It is easy to see that the left side of \eqref{Def.prelimiary-tstar} increases with $t$ while the right side of \eqref{Def.prelimiary-tstar} decreases with $t$. Therefore, the definition of $t^*$ shows that 
\begin{equation}\label{pre-1.small}
    {\small
\mathcal W_{D,t}
\leq  
C_2'\left\{\begin{array}{cc}
2^{r-1/2}  t^{-r}, & \textnormal{if}\ \frac12\leq r\leq
\frac32,\\
t^{-r}+4\kappa^2t^{-1/2}|D|^{-1/2},& \hspace{-1.0em}\textnormal{if}\ r>\frac32,
\end{array}\right. \hspace{-0.5em} \forall t\leq t^*,}
\end{equation}
and
\begin{equation}\label{pre-1.large}
    {\small
\mathcal W_{D,t}
\geq  
C_2'\left\{\begin{array}{cc}
2^{r-1/2}  t^{-r}, & \hspace{-1.0em}\textnormal{if}\ \frac12\leq r\leq
\frac32,\\
t^{-r}+4\kappa^2t^{-1/2}|D|^{-1/2},& \textnormal{if}\ r>\frac32,
\end{array}\right.  \forall t>t^*.}
\end{equation}
The following proposition presents the relation between $\hat{t}$ and $t^*$ and  provides an upper bound for $\hat{t}$.

\begin{proposition}\label{Prop:bound-t-semi}
	Let $0<\delta<1$.  For 
\begin{eqnarray}\label{def.t0bbbbbb}
    {\small
     t_0=\tilde{c}\log^{a(\sigma)}\frac{16}{\delta}\left\{\begin{array}{cc}
     \!\!\!\! (|D|/L)^{\frac1{2r}},    & \!\! \textnormal{if} \ \sigma_\ell=0,\ell\geq L+1,\\
     \!\!\!\!  |D|^{\frac{1}{2r+s}},   &\!\! \!\! \!\!\!\textnormal{if} \ \sigma_\ell\leq c_0\ell^{-1/s},\\
      (|D|/\sqrt{\log|D|})^{\frac{1}{2r}}, &  \!\!\!\!\!\!\textnormal{if} \ \sigma_\ell\leq c_1e^{-c_2\ell^2},\!\!
     \end{array}\right. } 
\end{eqnarray} 
	with some $\tilde{c}$ depending only on $\kappa,M$, $\delta$, and $\gamma$
and $a(\sigma):=-\frac{s}{2r+s}$ when $\sigma_\ell\leq c_0\ell^{-1/s} $ and $a(\sigma)=0$ otherwise,   
 if Assumptions \ref{Assumption:boundedness} and \ref{Assumption:regularity}     hold with $r\geq 1/2$, then with confidence $1-\delta$, there holds
	 \begin{equation}\label{relation-1-t}
	\max\{\hat{t},t_0\}\leq t^*\leq T 
	\end{equation}
 and
\begin{small}
        \begin{eqnarray}\label{optimal-semi}
	 &&\max\{\|f_{D,t^*}-f_\rho\|_\rho,  \|f_{D,t^*}-f_\rho\|_D, (t^*)^{1/2}\|f_{D,t^*}-f_\rho\|_K\} \nonumber\\
  &\leq&
  \tilde{C}'
  \log^2\frac{16}{\delta}\left\{\begin{array}{cc}
      (|D|/L)^{-1/2},  & \textnormal{if} \ \sigma_\ell=0,\ell\geq L+1,    \\
      |D|^{-r/(2r+s)}, &\!\!\!\!\! \textnormal{if} \ \sigma_\ell\leq c_0\ell^{-1/s},\\
      (|D|/\sqrt{\log|D|})^{-1/2},& \!\!\!\!\! \textnormal{if} \ \sigma_\ell\leq c_1e^{-c_2\ell^2},
  \end{array}\right.
\end{eqnarray}
\end{small}
 where $\tilde{C}'$ is a constant independent of $|D|,t$  or $\delta$.
\end{proposition}

Since $t_0\leq t^*$, we get from \eqref{optimal-semi} that 
 {\small
\begin{eqnarray*}
     &&\|f_{D,t^*}-f_\rho\|_K\leq t_0^{-1/2}
 \tilde{C}'\log^2\frac{16}{\delta}\left\{\begin{array}{cc}
      (|D|/L)^{-1/2},  & \textnormal{if} \ \sigma_\ell=0,\ell\geq L+1,    \\
      |D|^{-r/(2r+s)}, & \textnormal{if} \ \sigma_\ell\leq c_0\ell^{-1/s},\\
      (|D|/\sqrt{\log|D|})^{-1/2},& \textnormal{if} \ \sigma_\ell\leq c_1e^{-c_2\ell^2} 
  \end{array}\right.\\
  &\leq&
  \tilde{C}'_1\log^2\frac{16}{\delta}\left\{\begin{array}{cc}
      (|D|/L)^{-\frac{(1-1/2r)}{2}},  & \mbox{if} \ \sigma_\ell=0,\ell\geq L+1,    \\
      |D|^{-\frac{r-1/2}{2r+s}}, & \mbox{if} \ \sigma_\ell\leq c_0\ell^{-1/s},\\
      \left(\frac{|D|}{\sqrt{\log|D|}}\right)^{-\frac{(1-1/2r)}{2}},& \mbox{if} \ \sigma_\ell\leq c_1e^{-c_2\ell^2},
  \end{array}\right.
\end{eqnarray*}}
which coincides with the bound derived in Corollary \ref{Colloary:generalization-error-abc}, where $\tilde{C}'_1$ is a constant independent of $\delta,|D|,t$.
Proposition \ref{Prop:bound-t-semi} derives optimal generalization error bounds for KGD equipped with the proposed semi-adaptive stopping rule. Furthermore, it presents an upper bound of $\hat{t}$, which will play a crucial role in verifying the optimality of the BSP rule in Definition \ref{def:asr}.

\subsection{Optimal generalization error bounds for KGD with HSS}
In this subsection, we provide theoretical verifications for the proposed parameter selection strategy. First, we derive the optimal generalization error bounds for KGD with BSP defined by  Definition \ref{def:asr}.

\begin{theorem}\label{Theorem:KGD with stop111}
	Let $0<\delta<1$ and $0<\beta\leq {\kappa^{-1}}$. 
 Under Assumption \ref{Assumption:boundedness} and Assumption \ref{Assumption:regularity} with $r\geq 1/2$, if $\hat{t}$ is defined by Definition \ref{def:asr} with $\tilde{C}\geq 32\sqrt{2}(1+\beta)(\kappa M+\gamma)$,
	then with confidence $1-\delta$, there holds
     {\footnotesize
	\begin{eqnarray}\label{Error-estimate1}
	&&\max\left\{\|f_{\hat{t},\beta,D}-f_\rho\|_\rho,
	\|f_{\hat{t},\beta,D}-f_\rho\|_D \right\}
	\leq
	C\log^4\frac{16}\delta\left\{\begin{array}{cc}
      \frac{\sqrt{L}\log|D|}{\sqrt{|D|}},  & \textnormal{if} \ \sigma_\ell=0,\ell\geq L+1,    \\
      |D|^{-r/(2r+s)}, & \textnormal{if} \ \sigma_\ell\leq c_0\ell^{-1/s},\\
      \frac{\log|D|}{\sqrt{|D|}},& \textnormal{if} \ \sigma_\ell\leq c_1e^{-c_2\ell^2},
  \end{array}\right.
	\end{eqnarray}}
	and
     {\footnotesize
	\begin{eqnarray}\label{Error-estimate2}
	&&\|f_{\hat{t},\beta,D}-f_\rho\|_K 
    \leq C \log^4\frac{16}\delta\left\{\begin{array}{cc}
      \left(\frac{\sqrt{L}\log|D|}{\sqrt{|D|}}\right)^{1-1/2r},  & \textnormal{if} \ \sigma_\ell=0,\ell\geq L+1,    \\
      |D|^{-\frac{r-1/2}{2r+s}}, & \textnormal{if} \ \sigma_\ell\leq c_0\ell^{-1/s},\\
      \left(\frac{\log|D|}{\sqrt{|D|}}\right)^{{(1-1/2r)} },& \textnormal{if} \ \sigma_\ell\leq c_1e^{-c_2\ell^2},
  \end{array}\right. 
	\end{eqnarray}}
	where $C$ is a constant independent of $|D|$ or $\delta$.
\end{theorem}

Theorem \ref{Theorem:KGD with stop111} shows that, for polynomial decayed $L_K$, KGD with BSP achieves the optimal generalization error bounds of KGD in Corollary \ref{Colloary:generalization-error-abc}. 
Table \ref{tab:comparison} provides a comparison of different parameter selection strategies for KGD in terms of theoretical performance. 
We have made the following observations. All methods are adaptive to the kernel, but only HO \citep{caponnetto2010cross}, CV \citep{gyorfi2002distribution}, ESR \citep{raskutti2014early}, and BSP achieve optimal generalization error bounds. Notably, the splitting method does not adapt to the error metric. ESR for KGD does not adapt to the regularity of $f_\rho$ since it only holds for $r=1/2$.  In fact, it behaves similarly to the semi-adaptive stopping rule defined by \eqref{Def.prelimiary-tstar}. For BP proposed in \citep{de2010adaptive,lu2020balancing} and DP proposed in \citep{celisse2021analyzing}, neither adapts to the error metrics, suggesting that different norms require distinct parameter selection strategies. Only the LP proposed in \citep{blanchard2019lepskii} and BSP can adapt to the regularity index  $r$, capacity index $s$ and also the error metrics. However, LP involves recurrent comparisons on the numerical side and introduces an additional logarithmic term in the generalization error bound, resulting in a gap between Corollary \ref{Colloary:generalization-error-abc} and \cite[Theorem 1]{blanchard2019lepskii}.
 In a recent work \citep{lin2024adaptive}, a novel Lepskii-type principle was designed for kernel ridge regression (KRR), deriving optimal generalization error bounds. The problem is, however, that due to the saturation phenomenon of KRR \citep{bauer2007regularization,gerfo2008spectral}, the generalization error bounds saturate at $r\leq 1$, meaning that larger $r$ no more implies better generalization error bounds. 
Our theoretical results in Theorem \ref{Theorem:KGD with stop111}  show that BSP adapts to the regularity index $r\in[1/2,\infty)$, capacity index $s\in (0,1]$ and different metrics of error.  This implies that with a unified parameter selection strategy, optimal generalization error bounds hold for all  $r\in[1/2,\infty)$, $s\in (0,1]$ under three different  norms:  $\|\cdot\|_\rho$, $\|\cdot\|_D$ and $\|\cdot\|_K$,  simultaneously.
This demonstrates the power of BSP, which fully exploits the iterative nature of KGD.

\begin{table*}[!t]
	\renewcommand\arraystretch{1.2} 
	\small
	\centering
	\caption{Comparison of different parameter selection methods}\label{tab:comparison}
	\scalebox{0.62}{ 
		\begin{tabular}{cl|l|c|c|c}
			\hline
			\rowcolor{gray!20}
			\multicolumn{6}{c}{\textbf{Comparison of parameter selection methods for KGD
			}} \\ \cline{1-6}
			\multicolumn{2}{c|}{\textbf{Methods}}&generalization error bound &Ada. to kernel& Ada. to function &	Ada. to norm $\|\cdot\|_\rho$ and  $\|\cdot\|_K$	  \\
			\hline
			\multirow{2}*{{\makecell[c]{\textbf{Information}\\\textbf{entropy method}}}  }&\textit{AIC \citep{demyanov2012aic}} &N/A  & \checkmark	& $\times$& $\times$
			\\
			\cline{2-6}
			&\textit{BIC \citep{demyanov2012aic}}&N/A	&\checkmark&$\times$ &$\times$
			\\
			\cline{2-6}
			\hline
			\multirow{2}*{\textbf{Splitting method}}&\textit{HO \citep{caponnetto2010cross}}&   Optimal  
			&\checkmark
			& \checkmark&$\times$
			\\
			\cline{2-6}
			&\textit{CV \citep{gyorfi2002distribution}}&  Optimal
			&\checkmark& \checkmark &$\times$
			\\
			\cline{2-6}
			&\textit{LOO \citep{zhang2003leave}}&   Sub-optimal  
			&\checkmark
			& \checkmark&$\times$
			\\
			\cline{2-6}
			\hline
			\multirow{4}*{{\makecell[c]{\textbf{Bias--variance}\\\textbf{analysis method}}}  
			}&\textit{BP \citep{lu2020balancing}}& Near optimal
			&\checkmark
			&\checkmark&	$\times$\\
			\cline{2-6}
			&\textit{LP \citep{blanchard2019lepskii}}& Near-optimal   & \checkmark&\checkmark&\checkmark\\
			\cline{2-6}
			
			&\textit{ESR \citep{raskutti2014early}}&Optimal 
			&	\checkmark
			&$\times$ &$\times$
			\\
			\cline{2-6}
			
			&\textit{DP \citep{celisse2021analyzing} }&Sub-optimal&	\checkmark
			&\checkmark &$\times$
			\\
			\cline{2-6}
			\hline
			\multirow{1}*{\textbf{Ours}}&\textit{BSP} &Optimal&\textbf{\checkmark}&\checkmark&\checkmark  \\
			\hline
	\end{tabular}}
	\vspace{3pt} 
	\footnotesize
	\begin{tabular}{@{}l@{}}
		\multicolumn{1}{@{}p{0.94\linewidth}@{}}{
			\scriptsize
			\textbf{Note}: 
			``Ada.'' represents adaptive. ``N/A'' indicates that no relevant literature related to KGD was found.
		}
	\end{tabular}
\end{table*}

	Theorem \ref{Theorem:KGD with stop111} presents the feasibility and optimality for KGD with BSP. A key point is that the assertions hold for any
$\tilde{C}\geq \tilde{C}_0:=32\sqrt{2}(1+\beta)(\kappa M+\gamma)$. On one hand, it can be found in our proof  that the constant $\tilde{C}_0$  is a little bit pessimistic. We believe that it can be reduced by using a more delicate analysis technique. On the other hand, the best choice of $\tilde{C}$ is doomed to optimize the parameter selection strategy via minimizing the constant $C$ in \eqref{Error-estimate1} and \eqref{Error-estimate2}.
Recalling that BSP adapts to different metrics of error, HSS then focuses on selecting a good candidate of $\tilde{C}$ to minimize the error under the $\|\cdot\|_\rho$ metric \citep{caponnetto2010cross}. The following  corollary that can be derived directly from Theorem \ref{Theorem:KGD with stop111} presents the theoretical guarantee for KGD with HSS in Algorithm \ref{alg:KGD with BSP}. 
\begin{corollary}\label{Corollary:KGD with HSS}
	Let $0<\delta<1$ and $0<\beta\leq {\kappa^{-1}}$. Under Assumption \ref{Assumption:boundedness} and Assumption \ref{Assumption:regularity} with $r\geq 1/2$, if $\hat{t^*}$ is given by Algorithm \ref{alg:KGD with BSP} with $\min_{\hat{C}_j\in C_U}\geq \tilde{C}_0$, then with confidence $1-\delta$, 
    \eqref{Error-estimate1} and \eqref{Error-estimate2}
     hold for $f_{\hat{t^*}}$.
\end{corollary}
Corollary \ref{Corollary:KGD with HSS} demonstrates the optimality of the proposed HSS strategy in Algorithm \ref{alg:KGD with BSP}. Though an additional assumption  $\min_{\hat{C}_j\in C_U}\geq \tilde{C}_0$ is imposed to the parameter selection strategy that inevitably  narrows the range of $\tilde{C}$ in theory, the estimate of $\tilde{C}_0$, as discussed above, is too pessimistic and can be  significantly reduced numerically. 
All the above theoretical results show  that without splitting the data, it is possible to design a delicate and fully adaptive parameter selection strategy to equip KGD to achieve its optimal generalization error bounds. 

Since our results hold for $\|\cdot\|_K$ and $\|\cdot\|_\infty\leq \kappa\|\cdot\|_K$, we obtain from Theorem \ref{Theorem:KGD with stop111} the following corollary  directly.

\begin{corollary}\label{Corollary:mis-match}
	Let $0<\delta<1$ and $0<\beta\leq {\kappa^{-1}}$. Under Assumption \ref{Assumption:boundedness} and Assumption \ref{Assumption:regularity} with $r\geq 1/2$,   then  
    {\small
\begin{eqnarray}\label{mis-match}
	\|f_{\hat{t},\beta,D}-f_\rho\|_\infty 
    &\leq& C\kappa \log^4\frac{16}\delta\left\{\begin{array}{cc}
      \left(\frac{\sqrt{L}\log|D|}{\sqrt{|D|}}\right)^{1-1/2r},  & \textnormal{if} \ \sigma_\ell=0,\ell\geq L+1,    \\
      |D|^{-\frac{r-1/2}{2r+s}}, & \textnormal{if} \ \sigma_\ell\leq c_0\ell^{-1/s},\\
      \left(\frac{\log|D|}{\sqrt{|D|}}\right)^{{(1-1/2r)} },& \textnormal{if} \ \sigma_\ell\leq c_1e^{-c_2\ell^2}.
  \end{array}\right. 
	\end{eqnarray}}
\end{corollary}
 
Different from describing the generalization error in $\|\cdot\|_\rho$, the above corollary derives error bounds in $\|\cdot\|_\infty$, which implies the capability of covariate shift problem  \citep{ma2023optimally}, in which the marginal distributions for the test data and training data are different. Actually, for any $\rho'_X$, it is easy to derive $\|\cdot\|_{\rho'}\leq \|\cdot\|_\infty$. We should highlight that the derived bounds for $\|\cdot\|_{\rho'}$ is worse than \cite{ma2023optimally} which involves distribution information in the algorithms. Our results are different from \cite{ma2023optimally} for tackling the covariant-shift problems. On one hand, our algorithm is distribution-independent and results are independent of the similarity between $\rho_X$ and $\rho'_X$. On the other hand, we introduce a stricter metrics, $\|\cdot\|_K$ or $\|\cdot\|_\infty$, to measure the generalization error which is, at least for $\|\cdot\|_K$ and polynomial decayed $L_K$, optimal in the mini-max sense \citep{fischer2020sobolev}. 
    


\section{Numerical analysis}\label{Sec.Experiments}
In this section, we perform both toy simulations and two real-world data experiments to demonstrate the numerical properties of KGD with the HSS parameter selection strategy as shown in Algorithm 1, and to verify Theorem \ref{Theorem:KGD with stop111}, Corollary \ref{Corollary:KGD with HSS}, and Corollary \ref{Corollary:mis-match}.  Except for the comparison experiments in Simulation 2, all other experiments were executed using Python 3.7 on a PC with an Intel Core i5 2GHz processor. Scripts reproducing these experiments are available at \url{https://github.com/Ariesoomoon/HSS_experiments}.

\subsection{Simulation experiments}
We conduct three simulations. The first simulation demonstrates the feasibility and power of the proposed adaptive stopping rule BSP. The second simulation illustrates the effectiveness and superior performance of HSS by comparing it with other parameter selection strategies, including  the baseline strategy BS, information entropy methods AIC and BIC, the splitting type strategy HO, and the bias--variance analysis strategies such as BP, LP, ESR, and DP. The third highlights HSS's capability in overcoming covariate shift. To fully showcase the numerical optimality of HSS, we employ two error metrics:  $L_2$ norm and  $L_\infty$ norm. Each simulation is conducted 10 times for averaging.



We generate $|D|\in\{200, 400, \dots, 6000\}$ samples for training, where $|D|$ denotes the cardinality of the dataset $D=\{(x_i,y_i)\}_{i=1}^{|D|}$. The
inputs $\{x_i\}_{i=1}^{|D|}$ are independently  drawn according to the uniform distribution on the (hyper-)cube $[0,1]^d$ with $d = 1, 3$. The corresponding outputs are generated from the regression models
$y_i=g_j(x_i)+\varepsilon_i,i=1,\dots,|D|, j=1,2$, where
$\varepsilon_i$ is the independent Gaussian noise $\mathcal N(0,
\sigma^2)$ with $\sigma=0.6$,
    \begin{small}
\begin{equation}\label{g1}
g_1(x)\!:=\!\left\{
\begin{array}{cc}
\!\!x, & \mbox{if } 0\leq x\leq 0.5, \\
\!\!1-x, & \mbox{if } 0.5<x,
\end{array}%
\right.
\end{equation}
\end{small}
and
\begin{small}
\begin{equation}\label{g2}
g_2(\|x\|_2\!)\!\!:=\!\!\left\{\ \ \!\!\!
\begin{array}{cc}
\hspace{-0.3em}\!\!(1\!\!-\!\!\|x\|_2)^6(35\|x\|_2^2\!\!+\!\!18\|x\|_2\!\!+\!\!3),& \hspace{-0.3em}\mbox{if }  \!\ \!0\leq\!\|x\|_2\!\leq \!1, x\!\in\!\mathbb{R}^3, \\
0, & \hspace{-0.3em}\mbox{if } \!\|x\|_2\!>\!1.
\end{array}%
\right.
\end{equation}
\end{small}
Furthermore, testing samples $\{(x_i',y_i')\}_{i=1}^{|D^{\prime}|}$ with $|D^{\prime}|=500$ are generated, where  $\{x_i'\}_{i=1}^{|D^{\prime}|}$ are drawn independently according to the uniform distribution and
$y_i'=g_j(x_i'), j=1, 2$. The adopted kernel function is  $K_1(x,x')=1+\min(x,x')$ for $d=1$  and
$K_2(x,x')=h(\|x-x'\|_2)$ with
    \begin{small}
\begin{equation}
h(\|x\|_2):=\left\{
\begin{array}{cc}
\!\!(1-\|x\|_2)^4(4\|x\|_2+1),&\!\! \mbox{if }\ 0\!<\!\|x\|_2\!\leq\! 1,\! x\in\mathbb{R}^3, \\
0, & \mbox{if } \|x\|_2>1,
\end{array}%
\right.
\end{equation}
\end{small}
for $d=3$.
It can be found in \citep{wu1995compactly,schaback2006kernel} that $g_1\in W_1^1$ and $g_2\in W_3^4$, where $W_d^\alpha$ denotes the $\alpha$-order Sobolev space on $[0,1]^d$; 
$K_1$ and $K_2$ are
reproducing kernels for $W_1^1$ and $W_3^2$, and $g_1\in\mathcal{H}_{K_1}$ and $g_2\in\mathcal{H}_{K_2}$, respectively.
Then, the KGD algorithm using the kernels $K_1$ and $K_2$ is applied to fit the training samples.

We mention that in all simulations, the ratio of training samples $D_{tr,L}$ to validation samples $D_{val, L}$ among the $L$ selected samples for HSS is fixed at 7:3. The same ratio is applied to other bias--variance analysis methods that require constant parameter selection.
For HO,  the ratio of training to validation samples  is fixed at 1:1. To ensure a fair comparison, the step size $\beta$ for parameter selection methods in all toy simulations is set to the same value: $\beta = 1$ for $d=1$ and $\beta = 3$ for $d=3$.

\subsubsection{Simulation 1: Feasibility and power of BSP}
In this simulation, we conducted two sets of experiments. First, we fixed the total data size $|D|=2000$ and observed how two different norms varied with the constant $\tilde{C}$ of BSP. Second, we examined the range where the optimal $\tilde{C}$ (denoted as $\hat{C}_{j^*}$) and the optimal $t$ (denoted as $\hat{t}^*$) are located across different data scales. There are two main observations from this simulation.

\begin{figure*}[htbp]
	\centering
	\subfigure{\includegraphics[scale=0.33]{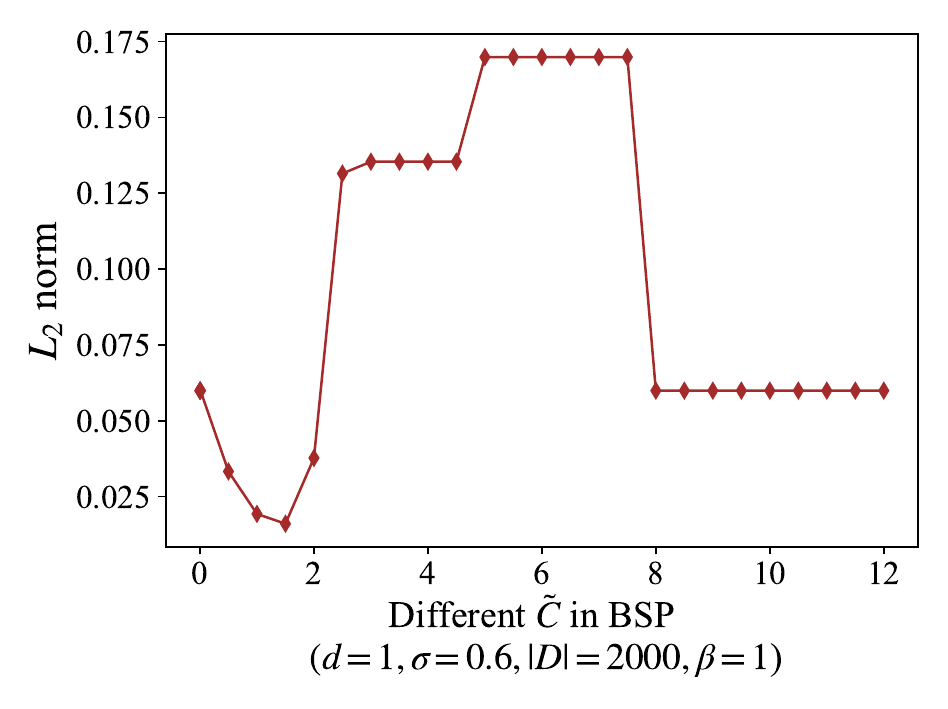}	\label{fig:C_selection_d1_L2norm}} 
	\setlength{\subfigcapskip}{-0.6em}
	\hspace{-0.02in} 
	\subfigure{\includegraphics[scale=0.33]
{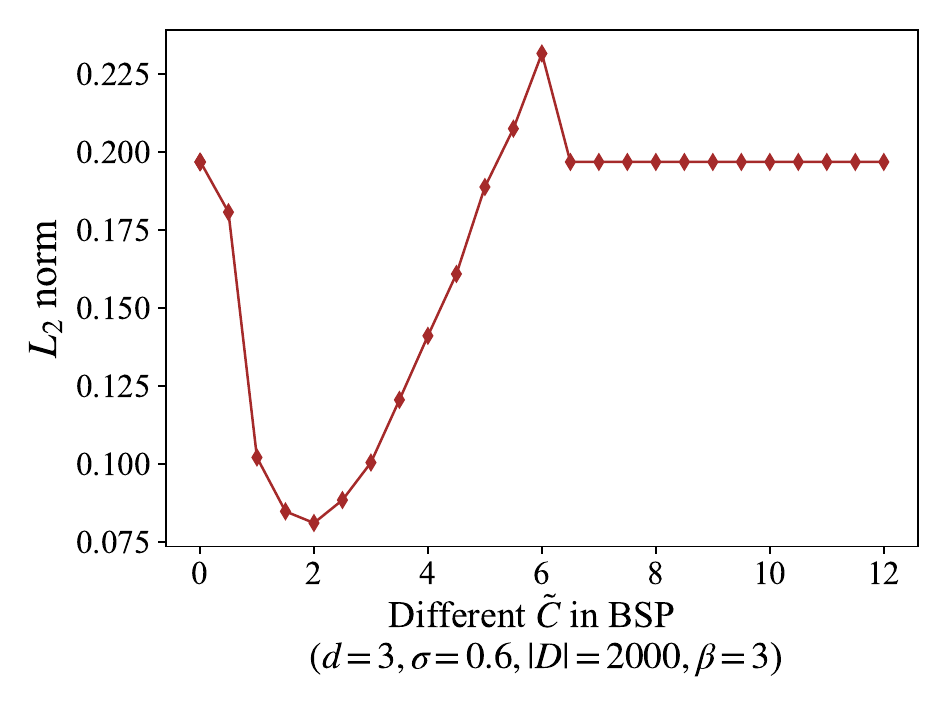}\label{fig:C_selection_d3_L2norm}} 
	\subfigure{\includegraphics[scale=0.33]{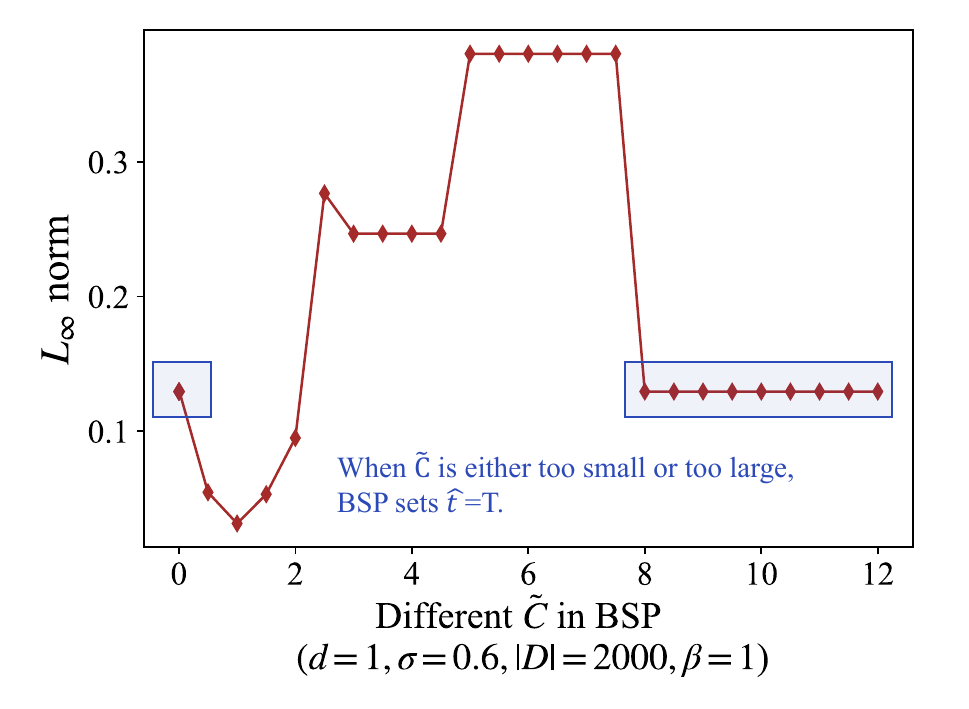}	\label{fig:C_selection_d1_Linftynorm}} 
	\setlength{\subfigcapskip}{-0.6em}
	\hspace{-0.02in} 
	\subfigure{\includegraphics[scale=0.33]{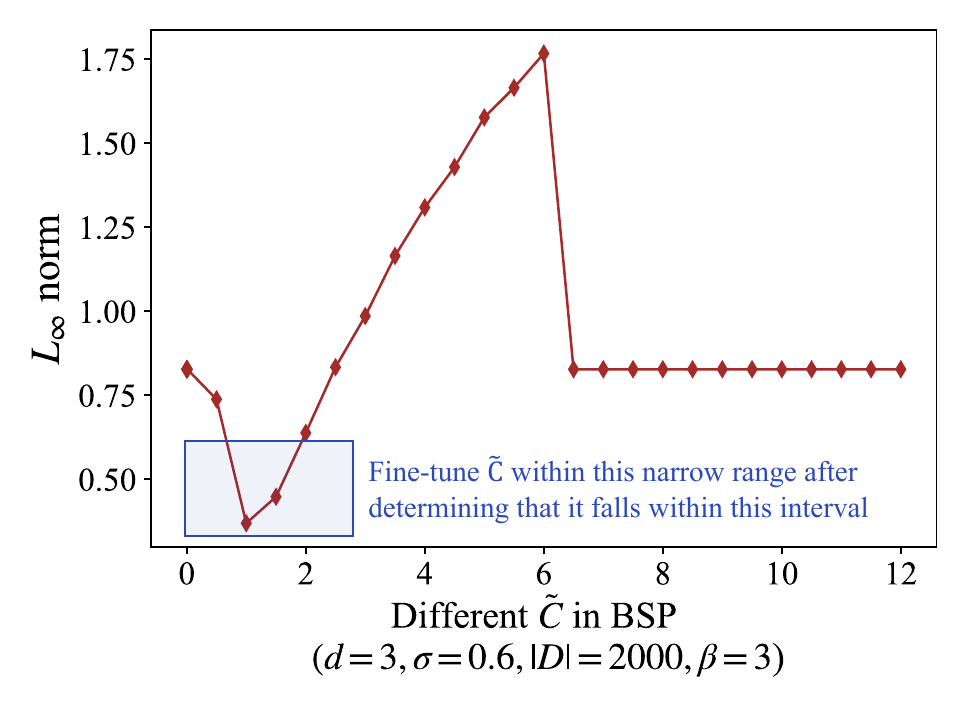}\label{fig:C_selection_d3_Linftynorm}} 
	\vspace{-0.05in}
	\caption{\footnotesize Relation between the $L_2$ norm/$L_{\infty}$ norm and the constant  $\tilde C$.}\label{fig:C_selection}
\end{figure*}

First, in Figure \ref{fig:C_selection}, both the $L_2$ norm and $L_\infty$ norm exhibit a global minimum with respect to $\tilde{C}$, demonstrating the feasibility of using BSP to equip KGD by selecting $\tilde{C}$. The similar behavior across different norms highlights the power of BSP. Note that  in Figure \ref{fig:C_selection}, the two sides of each subplot represent cases where $\tilde{C}$ is either too large or too small, resulting in $\hat{t}^*$ being determined as $T$ by BSP. 

Second, we observe that the optimal $\tilde{C}$ falls within a relatively narrow range ($[0, 4]$), making it easier to pinpoint. This suggests an efficient parameter selection approach:  first, use a logarithmic scale to quickly identify the rough range of $\tilde{C}$ (i.e., selecting $\tilde{C} \in \{0, 2^0, 2^1, 2^2, \dots\}$), and then refine the selection within that range using a uniform scale (i.e., selecting $\tilde{C} \in \{0.001, 0.002, \dots\})$. This process, referred to as log-uniform parameter selection, was used in our experiments, with the final selection interval for $\tilde{C}$ set to $2^{-10}$ (under the uniform scale).

It is important to emphasize that this parameter selection process is not suitable for directly selecting $\hat{t}^*$, as we observe that $\hat{t}^*$ spans  a much wider range [0, 800] compared to $\hat{C}_{j^*}$'s [0, 4], as shown in Figure \ref{fig:C_t_range}, and the generalization error changes gradually with respect to $t$, as seen in Figure \ref{fig:bias--variance}.

\begin{figure*}[htbp]
	\centering
	\setlength{\subfigcapskip}{-0.6em}
	\subfigure{\includegraphics[scale=0.35]{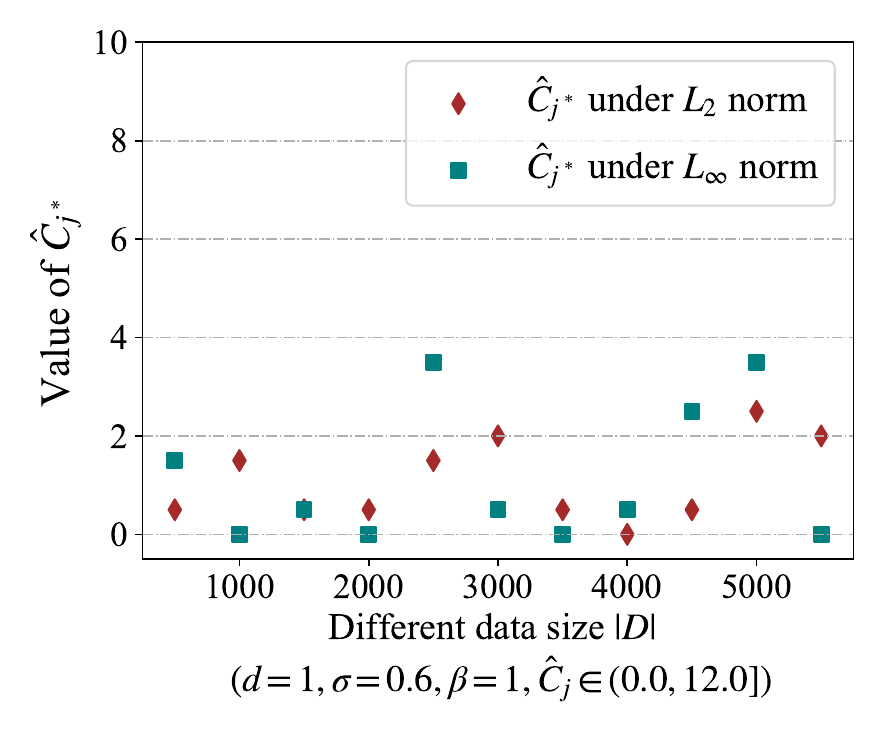}	\label{fig:C_range_d1_L2norm}} 
	\setlength{\subfigcapskip}{-0.6em}
	\hspace{0.01in} 
	\subfigure{\includegraphics[scale=0.35]{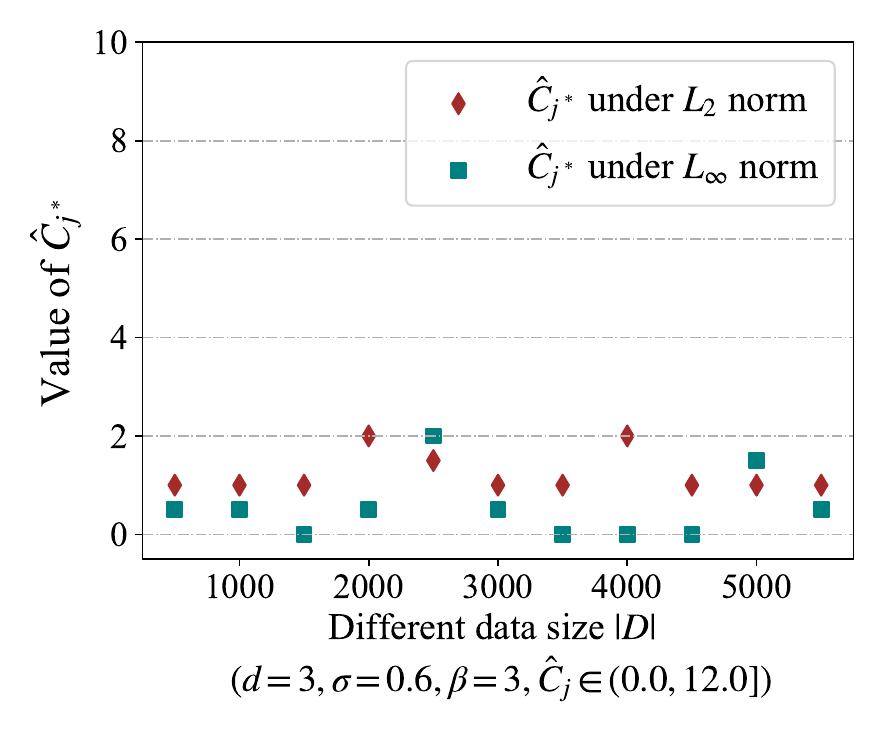}\label{fig:C_range_d1_Linftynorm}} 
	\subfigure{\includegraphics[scale=0.35]{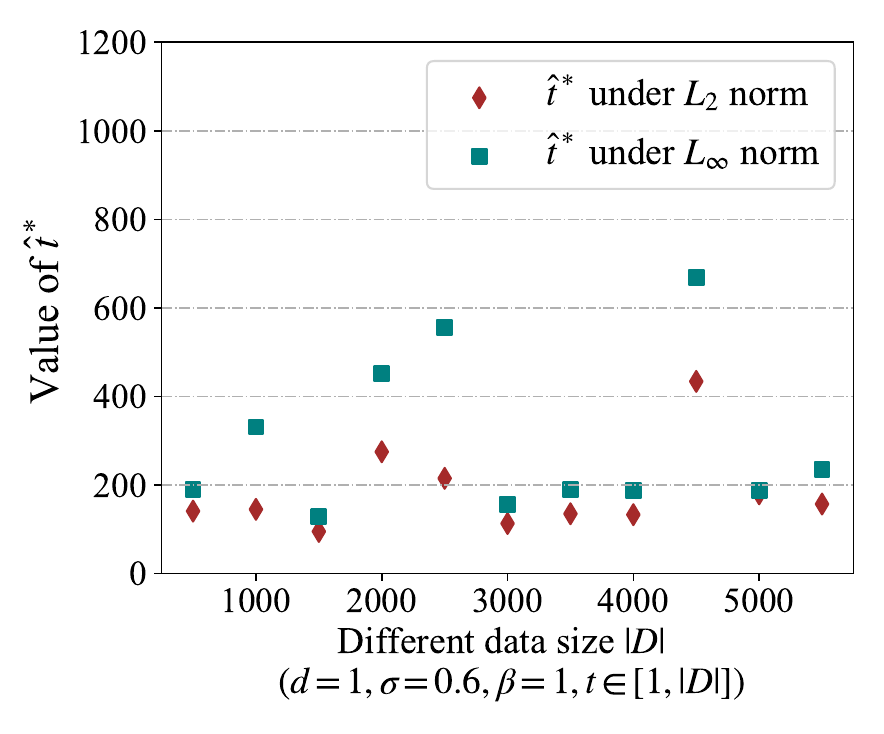}	\label{fig:t_range_d1_L2norm}} 
	\setlength{\subfigcapskip}{-0.6em}
	\hspace{0.01in} 
	\subfigure{\includegraphics[scale=0.35]{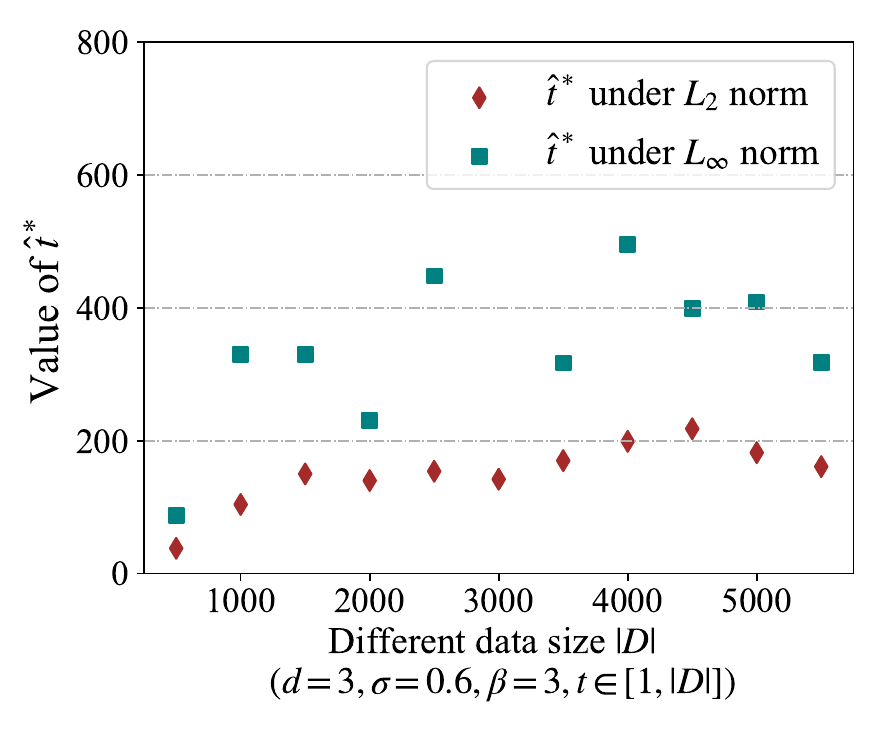}\label{fig:t_range_d1_Linftynorm}} 
	\vspace{-0.05in}
	\caption{\footnotesize Range of $\hat{C}_{j^*}$ under BSP and range of $\hat{t}^*$ under BS (results from a single experiment).}\label{fig:C_t_range}
\end{figure*} 

\begin{figure*}[htbp]
	\centering
	\setlength{\subfigcapskip}{-0.6em}
	\subfigure{\includegraphics[scale=0.35]{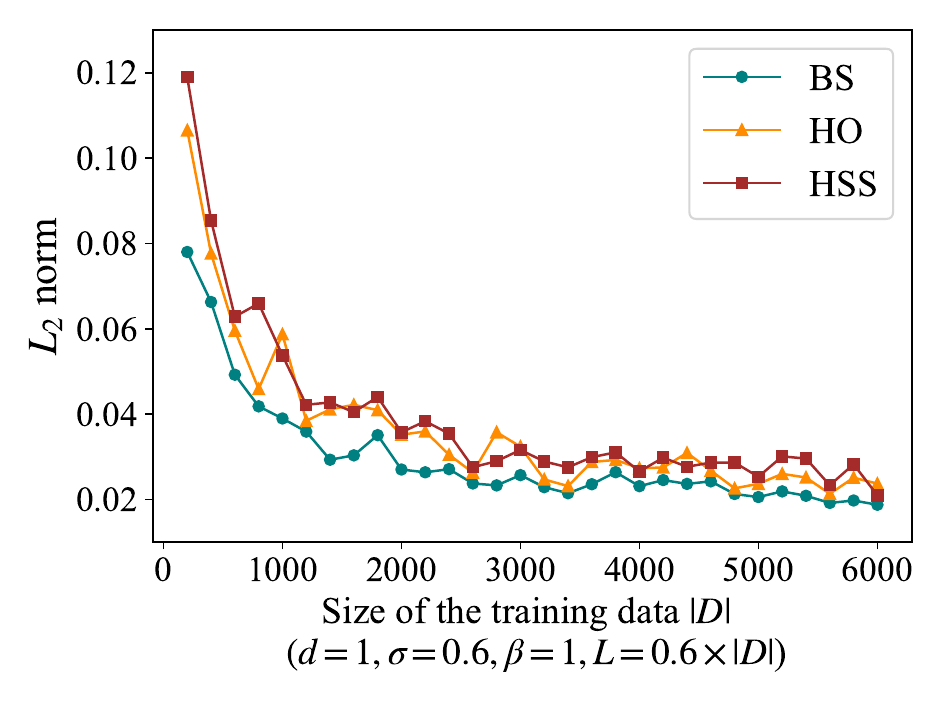}	\label{fig:HSS_advantage_d1_L2norm}} 
	\setlength{\subfigcapskip}{-0.6em}
	\hspace{0.01in} 
	\subfigure{\includegraphics[scale=0.35]{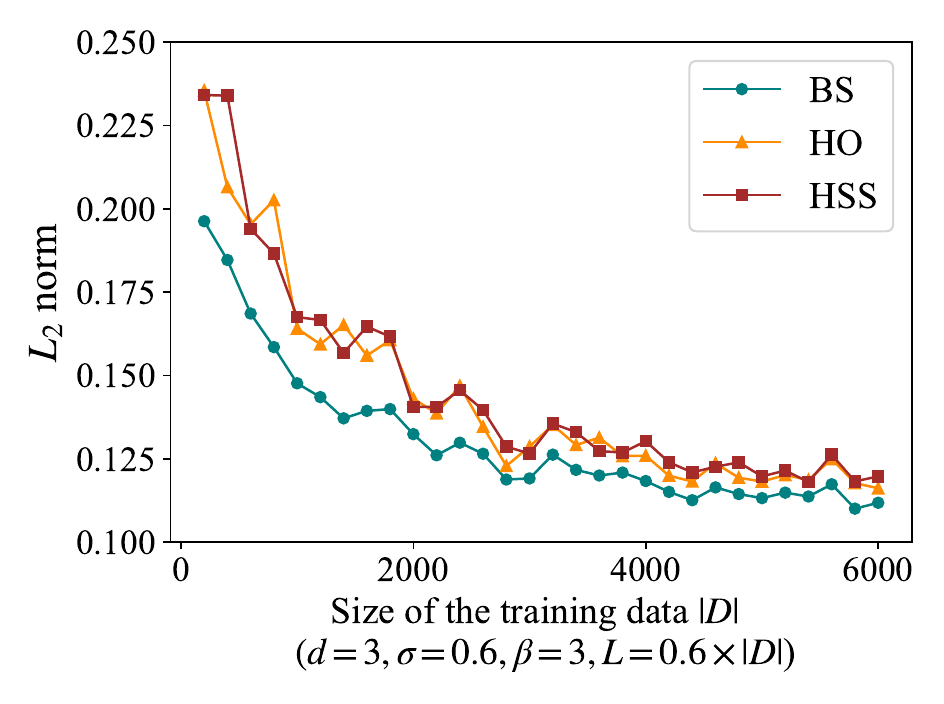}\label{fig:HSS_advantage_d3_L2norm}} 
	\subfigure{\includegraphics[scale=0.35]{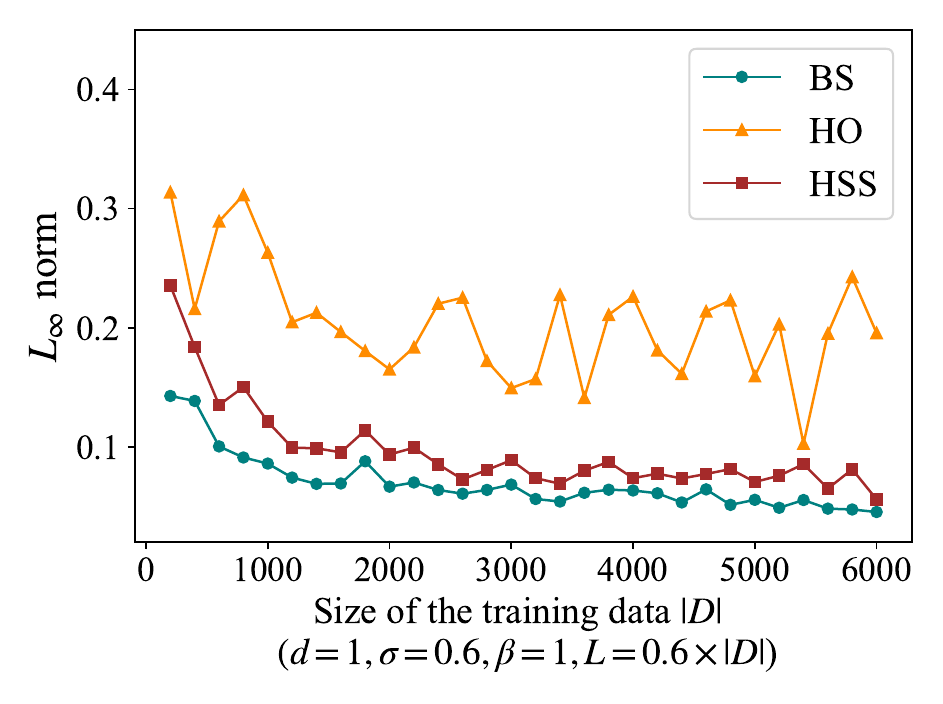}	\label{fig:HSS_advantage_d1_Linftynorm}} 
	\setlength{\subfigcapskip}{-0.6em}
	\hspace{0.01in} 
	\subfigure{\includegraphics[scale=0.35]{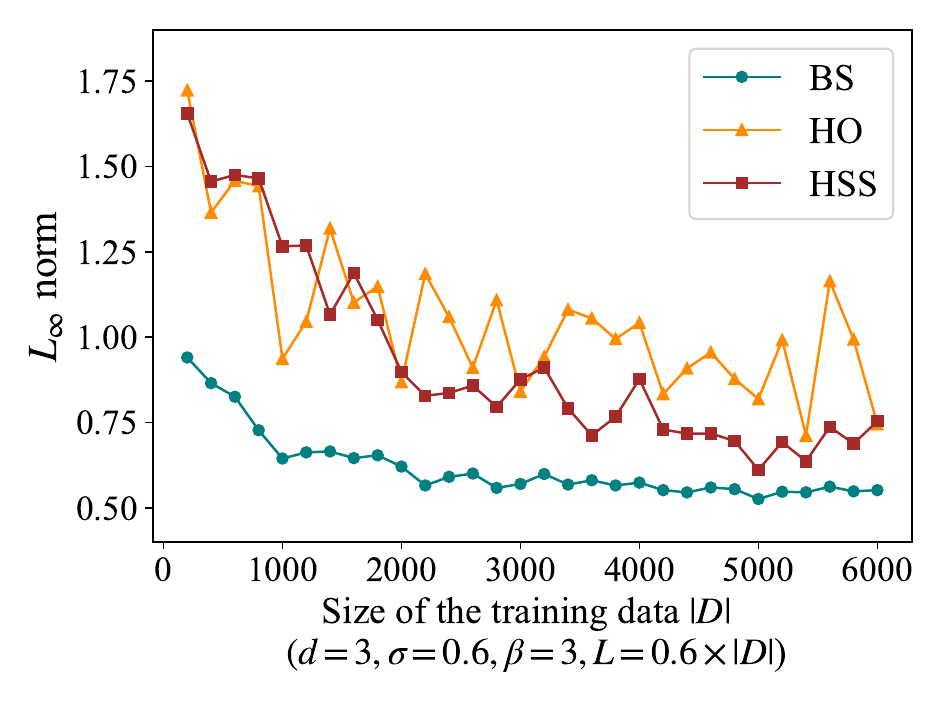}\label{fig:HSS_advantage_d3_Linftynorm}} 
	\vspace{-0.1in}
	\caption{\footnotesize Generalization performance of BS, HO, and HSS (where $L$ in HSS is fixed at $0.6|D|$).}\label{fig:HSS_advantage}
\end{figure*}

\subsubsection{Simulation 2: Effectiveness and superior performance  of HSS}

In this simulation, we demonstrate the computational efficiency and prediction performance of HSS, thereby validating Theorem \ref{Theorem:KGD with stop111} and Corollary \ref{Corollary:KGD with HSS}.

Figure \ref{fig:HSS_advantage} compares the prediction performance of BS, HO, and HSS across different data scales, with 
$L$ in HSS fixed at $0.6|D|$. We observe that under the $L_2$ norm,  the performance of HSS is comparable to HO, both of which are close to BS. However, under the $L_\infty$ norm, HSS significantly outperforms HO and is closer to the performance of BS. This demonstrates that the adaptive stopping rule BSP  can determine a high-quality step $\hat{t^*}$  based on a well-chosen constant $\tilde C$, thereby validating our Theorem \ref{Theorem:KGD with stop111}. It highlights the effectiveness and advantages of HSS, particularly its adaptability to different norms, thereby supporting Corollary \ref{Corollary:KGD with HSS}.

Furthermore, we compared HO, AIC, BIC, HSS and other bias--variance analysis methods in terms of efficiency and accuracy. Results of BP with the constant $C_{BP}$ from \citet{lu2020balancing} and  ESR with the constant $C_{ESR}$ from \citet{raskutti2014early} are also presented.  Due to the significant memory consumption required for BP and LP, all results  in the following were generated using Python 3.7 (Ubuntu 18.04) on a system equipped with an NVIDIA RTX 4090 (24GB) GPU, a 16-core Intel Xeon Platinum 8352V processor (2.10GHz), and 120GB of RAM. The comparison is limited to data scales of 1000 and 1200, with dimensions $d=1$ and $d=3$. Each result is  conducted 10 times for averaging. 

\begin{table*}[!htbp]
	\renewcommand\arraystretch{1.2} 
\scriptsize
	\centering
	\caption{Experiment settings for all bias--variance analysis methods in Table \ref{tab:comparison_d1_d3_1000} and Table \ref{tab:comparison_d1_d3_1200}}
	\label{tab:experiment_settings} 
    		\scalebox{0.93}{ 
	\begin{tabular}{l|c|l|c} 
		\toprule 
		\textbf{Settings} & \textbf{Details} & \textbf{Settings} & \textbf{Details}  \\
		\midrule  
		Number of alternative parameters & 24   & Step size &  $\beta=1$ for $d=1$, $\beta=3$ for $d=3$   \\
		Parameter selection interval & 0.05  & Training data size $|D|$& 1000, 1200     \\
		Optimal parameter was selected & Yes & Testing data size $|D^{\prime}|$& 500  \\
		Number of trials &10 & Accuracy metric & $L_2$ norm, $L_\infty$ norm  
		\\
		Maximum number of iterations  &  $T=|D|$ &Efficiency metric &  Maximum memory usage (MMU), Time (s) \\
		\bottomrule 
	\end{tabular}}
\end{table*}

We carefully designed the experiments to ensure a fair comparison, with the specific settings detailed in Table \ref{tab:experiment_settings}. For accuracy, we used the $L_2$ norm and $L_\infty$ norm, while for efficiency, we measured the execution time and the maximum memory usage (MMU) and  during the execution of each algorithm. Both execution time and MMU were recorded individually.

\begin{table*}[!htbp]
	\renewcommand\arraystretch{1.2} 
	\scriptsize
	\centering
	\caption{Numerical performance of different parameter selection methods when $|D|=1000$}\label{tab:comparison_d1_d3_1000}
		\scalebox{0.75}{ 
	\begin{tabular}{c|c|c|c|c||c|c|c|c}
		\hline
		\rowcolor{gray!70}
		\multicolumn{9}{c}{\textbf{Numerical performance of different parameter selection methods
		}} \\ \cline{1-9}
		\multicolumn{1}{c}{ \multirow{2}*{\textbf{Methods}
		} }& \multicolumn{4}{|c||}{$d=1$
		} & \multicolumn{4}{c}{$d=3$
		}
		\\
		\cline{2-9}	
		\multicolumn{1}{c|}{}&$L_2$ norm&$L_\infty$ norm&Time (s)&MMU (GB) &$L_2$ norm&$L_\infty$ norm&Time (s)& MMU (GB)
		\\
		\hline
		HO&0.0587&0.2630&1.03&0.43&0.1641&0.9370&15.98&0.54\\
		\hline
  		AIC& 0.0742 & 0.2407& 7.26&0.66 & 0.1738& 1.0100 & 110.76 & 0.80 \\
		\hline
    		BIC& 0.0750&0.2995 &8.85 &0.65 &0.2356 &  1.0573 &109.75 &0.81 \\
		\hline
		BP	& 0.0643&0.1817&924.71&52.40&0.1680 &1.0030&1081.37&53.38\\
		\textit{Use \textit{$C_{BP}$ in \citep{lu2020balancing} }}	&\textcolor{blue}{\textbf{0.0528}} &\textcolor{blue}{\textbf{0.1299}}&512.08&52.62&0.2095&1.7912&533.16&52.94\\
		\hline
		LP&0.0658&0.2394& 1618.29&76.65&\textcolor{teal}{\textbf{0.1631}}&0.9624& 1543.62&75.79\\
		\hline
		ESR &0.0598&0.2604&4.88&0.92&0.1632&1.0001&93.09&1.03	\\
		\textit{{\makecell[c]{Use $C_{ESR}$ in  \citet{raskutti2014early}}} } &0.0660&\textcolor{teal}{\textbf{0.1599}}&4.00&0.47&0.2085&\textcolor{red}{\textbf{0.7784}}&24.23&27.48\\
		\hline
		DP&\textcolor{teal}{\textbf{0.0566}}&0.2788&4.32&0.49&\textcolor{blue}{\textbf{0.1617}}&\textcolor{teal}{\textbf{0.9307}}&87.01&0.62\\
		\hline
    {\makecell[c]{HSS\\with $L=|D|$}}  &\textcolor{red}{\textbf{0.0506}}&\textcolor{red}{\textbf{0.1216}}&8.71&0.51& \textcolor{red}{\textbf{0.1571}}&\textcolor{blue}{\textbf{0.8633}}&108.17&0.62\\
		\hline
	\end{tabular}}
	\vspace{3pt} 
	\footnotesize
	\begin{tabular}{@{}l@{}}
		\multicolumn{1}{@{}p{0.9\linewidth}@{}}{
			\scriptsize
			\textbf{Note}: 
			The results for each accuracy metric are highlighted in red, blue, and green to denote the best, second-best, and third-best performances, respectively. This also applies to Table \ref{tab:comparison_d1_d3_1200}.
		}
	\end{tabular}
\end{table*}

\begin{table*}[!htbp]
	\renewcommand\arraystretch{1.2} 
	\scriptsize
	\centering
	\caption{Numerical performance of different parameter selection methods when $|D|=1200$}\label{tab:comparison_d1_d3_1200}
		\scalebox{0.75}{ 
	\begin{tabular}{c|c|c|c|c||c|c|c|c}
		\hline
			\rowcolor{gray!70}
		\multicolumn{9}{c}{\textbf{Numerical performance of different parameter selection methods
		}} \\ \cline{1-9}
		\multicolumn{1}{c}{ \multirow{2}*{\textbf{Methods}
		} }& \multicolumn{4}{|c||}{$d=1$
		} & \multicolumn{4}{c}{$d=3$
		}
		\\
		\cline{2-9}	
		\multicolumn{1}{c|}{}&$L_2$ norm&$L_\infty$ norm&Time (s)& MMU (GB)&$L_2$ norm&$L_\infty$ norm&Time (s)& MMU (GB)
		\\
		\hline
		HO&\textcolor{red}{\textbf{0.0384}} &0.2046 &1.52&0.45& 0.1594& 1.0457&21.77&0.60\\
		\hline
    		AIC& 0.0439& 0.1548&11.58& 0.68 & 0.1503& 1.0330 &158.02 &0.87 \\
		\hline
    		BIC& 0.0485&0.2439 & 12.50&0.68 &0.2169 &\textcolor{teal}{\textbf{1.0292}}  &163.58 &0.92 \\
		\hline
		BP	&0.0423 &0.1346 &1524.22& 91.37&0.1513 & 1.0835&1788.85& 91.48\\
		\textit{Use \textit{$C_{BP}$ in \citep{lu2020balancing} }}	&0.0456 &\textcolor{blue}{\textbf{0.1164}}&886.35&90.97 &0.2050 &1.7753&882.45&90.20 \\
		\hline
		LP&0.0418 &\textcolor{teal}{\textbf{0.1249}}&1604.03*& 76.32*& 0.1531& 1.0692 &1750.27*& 76.34* \\
		\hline
		ESR &\textcolor{blue}{\textbf{0.0391}}&0.1906&7.42& 0.95&\textcolor{blue}{\textbf{0.1494}}&1.0864&133.70& 1.06	\\
		\textit{{\makecell[c]{Use $C_{ESR}$ in  \citet{raskutti2014early}}} }&0.0566&0.1398&6.50& 0.54&0.2062&\textcolor{red}{\textbf{0.7983}}&38.86&0.66	\\
		\hline
		DP&0.0606&0.1635&6.90&0.51&\textcolor{teal}{\textbf{0.1503}}&1.1473&123.31&0.65
		\\
		\hline
		{\makecell[c]{HSS\\with $L=|D|$}}   &\textcolor{teal}{\textbf{0.0393}}& \textcolor{red}{\textbf{0.1137}} & 12.26&0.51&\textcolor{red}{\textbf{0.1492}}& \textcolor{blue}{\textbf{0.8180}} & 158.00&0.64\\
		\hline
	\end{tabular}}
	\vspace{3pt} 
	\footnotesize
	\begin{tabular}{@{}l@{}}
		\multicolumn{1}{@{}p{0.9\linewidth}@{}}{
			\scriptsize
			\textbf{Note}: 
			When the data size is 1200, the LP calculation fails to run due to excessive memory usage (117 GB, which is 97.5\% of the system's total 120 GB of memory). We modified the code to release as much memory as possible during execution, allowing it to run normally. Results from this modified code for LP are denoted with $*$.
		}
	\end{tabular}
\end{table*}

\begin{figure}[t]
	\centering
	\setlength{\subfigcapskip}{-0.5em}
	\subfigure{\includegraphics[scale=0.32]{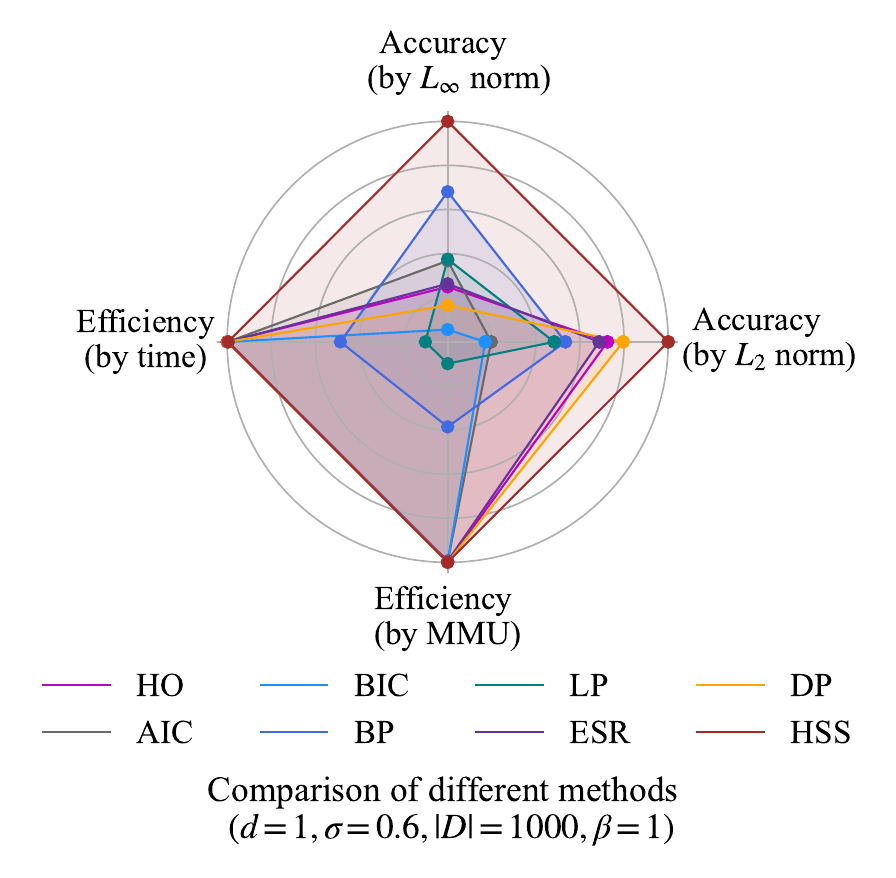}	\label{fig:Radar_Chart1}} 
	\setlength{\subfigcapskip}{-0.5em}
	\hspace{0.01in} 
	\subfigure{\includegraphics[scale=0.32]{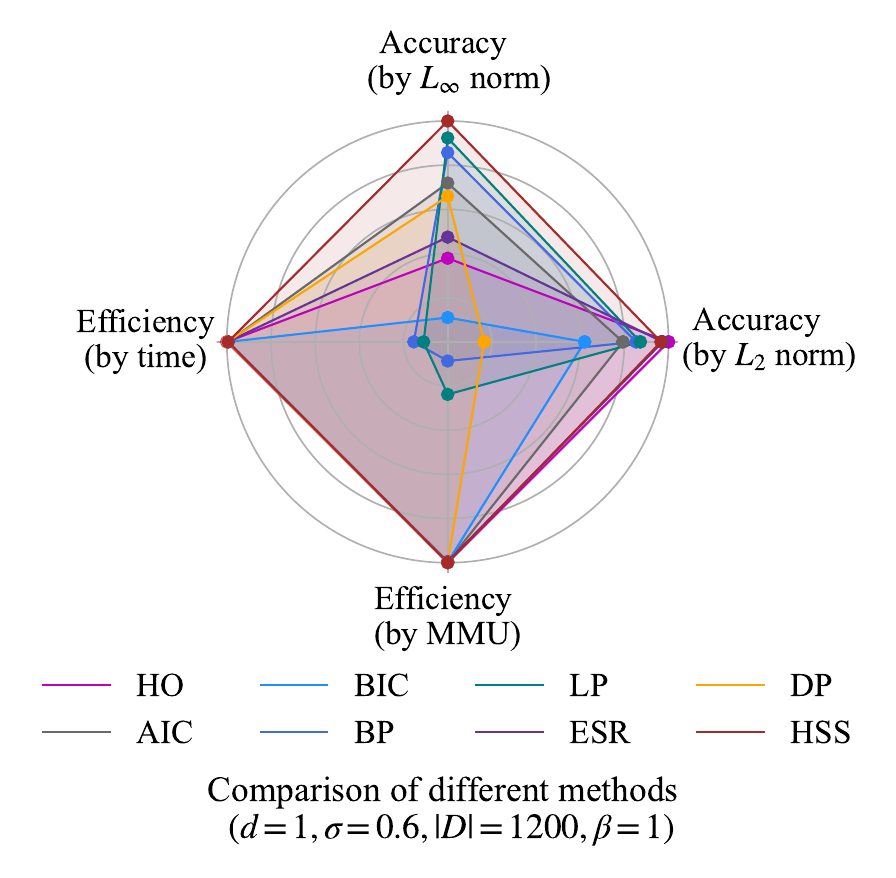}\label{fig:Radar_Chart2}} 
	\subfigure{\includegraphics[scale=0.32]{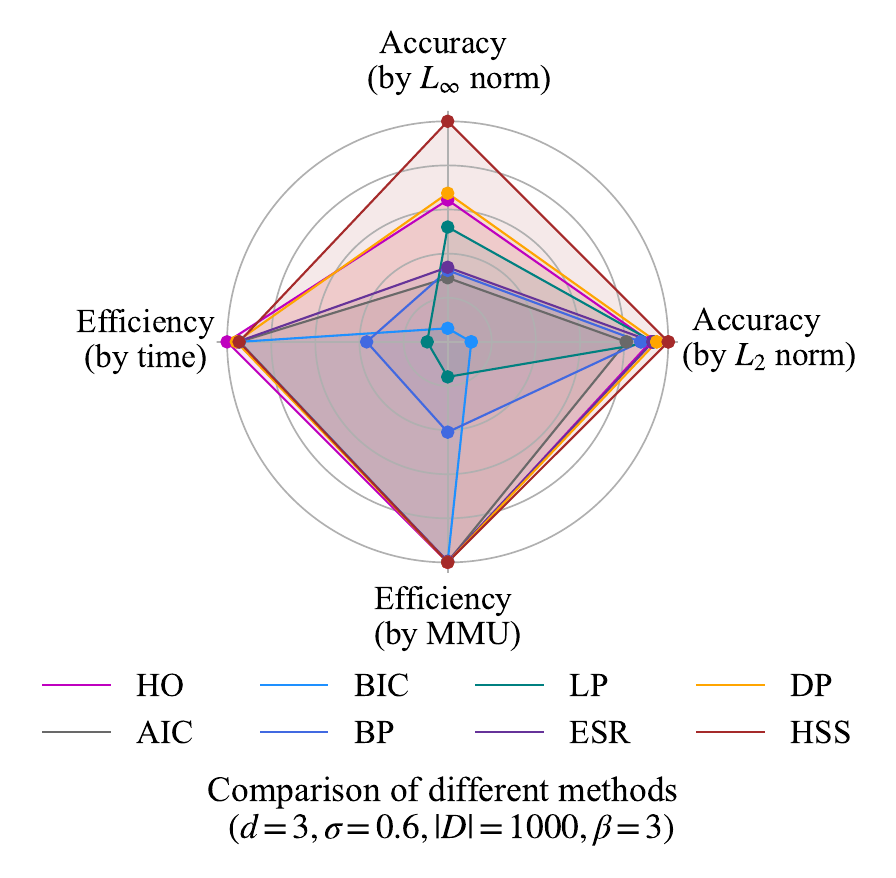}	\label{fig:Radar_Chart3}} 
	\setlength{\subfigcapskip}{-0.5em}
	\hspace{0.01in} 
	\subfigure{\includegraphics[scale=0.32]{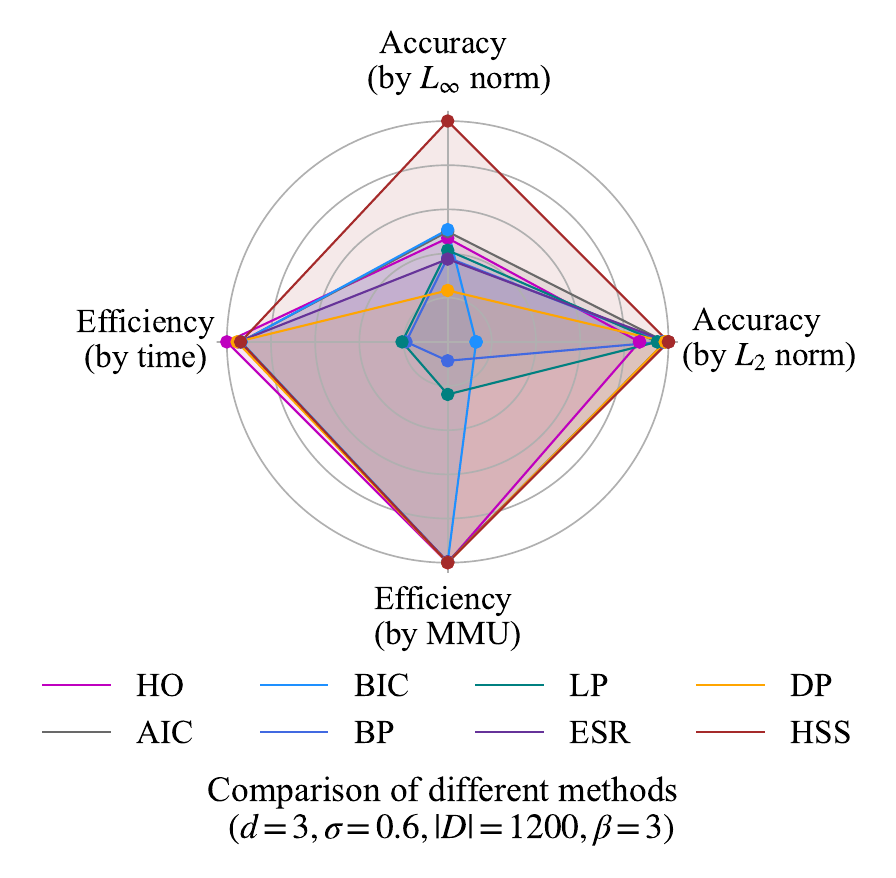}\label{fig:Radar_Chart4}} 
	\vspace{-0.05in}
	\caption{\footnotesize Illustration of the numerical performance of methods presented in Tables \ref{tab:comparison_d1_d3_1000} and \ref{tab:comparison_d1_d3_1200}.}\label{fig:Radar_chart_all}
\end{figure}

The detailed results are shown in Table \ref{tab:comparison_d1_d3_1000} and Table \ref{tab:comparison_d1_d3_1200}, where $L$
in HSS is fixed at $|D|$.  
In terms of accuracy, HSS performs similarly to HO under the $L_2$ norm but significantly better under the 
$L_\infty$ norm, further demonstrating the superior performance of HSS. 
Regarding efficiency, while HSS is slightly less efficient than HO, its performance remains within an acceptable range. We also observe that the running time of BP is substantially higher than that of HSS, as BP requires item-wise comparisons at each step $t$ for every constant, whereas HSS only performs two successive iterations.
When comparing the results of ESR and BP using the selected  constant $\tilde C$ with those using the fixed constants $C_{ESR}$ and $C_{BP}$, we found that in most cases, using fixed $\tilde C$ did not yield results   as good as those obtained through parameter selection. This indicates that while ESR and BP exhibit good theoretical performance, their numerical effectiveness depends heavily on the quality of $\tilde C$.

To provide a clear overview of the comparative results for all methods presented in the above two tables, we created a radar chart, as shown in Figure \ref{fig:Radar_chart_all}. We modified the indicators MMU and Time to be expressed similarly to accuracy; in this chart, values closer to the edge (or larger areas) represent better performance. From Figure \ref{fig:Radar_chart_all}, we make the following observations: (1) HSS demonstrates the best overall performance across all four indicators of accuracy and efficiency. Its advantage over other algorithms is particularly evident in $L_\infty$ norm. (2) While HO, AIC, BIC, ESR and DP perform well in terms of efficiency, their accuracy is generally inferior to HSS. (3) BP and LP perform significantly worse in terms of both running time and MMU, primarily due to their reliance on item-wise comparisons for each given constant (see (\ref{Eq:bp_stopping_rule}) and (\ref{Eq:lp_stopping_rule})).

The results of the radar chart demonstrate that among information entropy methods, splitting methods and bias--variance analysis methods, HSS exhibits superior performance in both efficiency and  accuracy.

\subsubsection{Simulation 3: HSS overcoming covariate shift}



This simulation demonstrates that HSS effectively addresses the covariate shift problem and thus validates Corollary~\ref{Corollary:mis-match}, thereby contributing to the existing literature on covariate shift in RKHS-based nonparametric regression \citep{ma2023optimally}.
Specifically, \citet{ma2023optimally} focuses on kernel ridge regression (KRR) and proposes a variant known as the truncated-reweighted KRR estimator, which achieves minimax rate-optimality \cite[Theorem 4]{ma2023optimally}. In contrast, we focus on KGD.

For simplicity, we set $d=1$ and  assume $\{x_i'\}_{i=1}^{|D^{\prime}|}$ are independently  drawn according to the uniform distribution on the (hyper-)cube $[0,b]$ with $b \in \{1.1, 1.2, \dots, 1.5\}$.
We use Kullback-Leibler (KL) divergence to quantify the distributional difference. For continuous probability distributions $P$ and $Q$, the KL divergence from $Q$ to $P$, denoted as $\mathbb{D}_{\mathrm{KL}}(P \| Q)$, is defined as
$
\mathbb{D}_{\mathrm{KL}}(P \| Q) = \int_{-\infty}^{\infty} p(\mathbf{x}) \ln \left(\frac{p(\mathbf{x})}{q(\mathbf{x})}\right) d \mathbf{x},
$
where $p(\mathbf{x})$ and $q(\mathbf{x})$ are the probability density functions of $P$ and $Q$, respectively. To calculate KL divergence (hereafter denoted as $\mathbb{D}_{\mathrm{KL}}$), we use kernel density estimation (with a bandwidth of 0.05) to obtain the probability density functions and Gauss-Legendre quadrature for integration.

Figure \ref{fig:HSS_HO_covariate_shift} presents the results under covariate shift. Figures \ref{fig:covariate_shift_L2norm1} and \ref{fig:covariate_shift_Linftynorm1} show that HSS consistently outperforms HO, where ``$\Delta L_2$ norm (\%)'' represents the change in the $L_2$ norm under covariate shift relative to the 
$L_2$ norm without shift, and ``$\Delta L_\infty$ norm (\%)'' carries a similar meaning. Figures \ref{fig:covariate_shift_L2norm2} and \ref{fig:covariate_shift_Linftynorm2} provide detailed results for both HSS and HO across 10 trials using boxplots, revealing that HSS exhibits smaller fluctuations as $\mathbb{D}_{\mathrm{KL}}$ varies. These results confirm HSS's 
robustness in handling covariate shift, thereby supporting the theoretical optimal generalization error bound guarantee under the $\|\cdot\|_\infty$ metric as stated in Corollary~\ref{Corollary:mis-match}.

\begin{figure*}[!htbp]
	\centering
	\setlength{\subfigcapskip}{-0.2em}
	\subfigure{\includegraphics[scale=0.35]{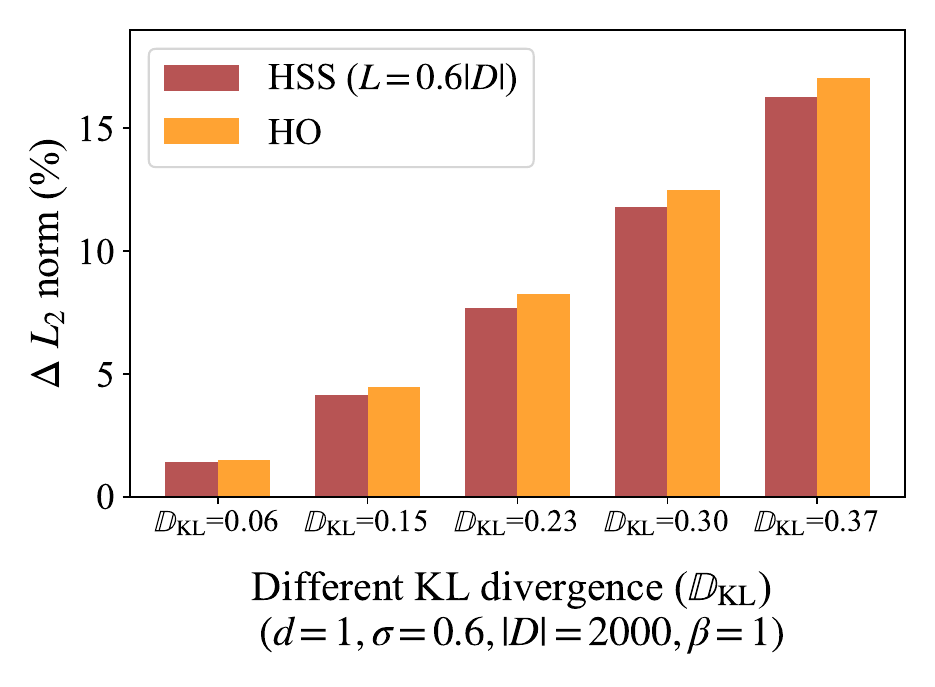}	\label{fig:covariate_shift_L2norm1}} 
	\setlength{\subfigcapskip}{-0.2em}
	\hspace{0.005in} 
	\subfigure{\includegraphics[scale=0.35]{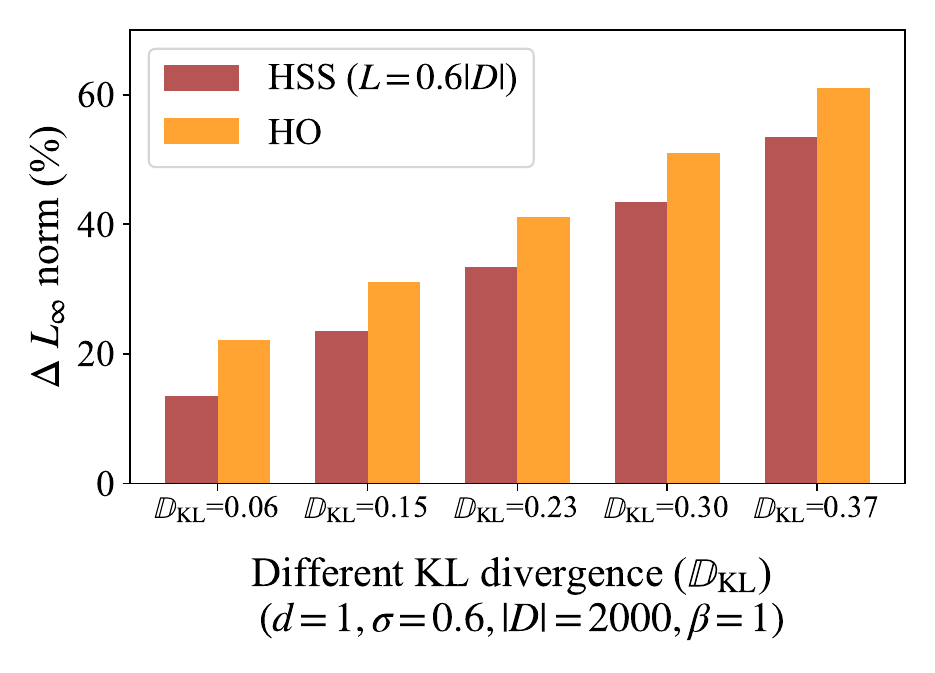}\label{fig:covariate_shift_Linftynorm1}} 
		\subfigure{\includegraphics[scale=0.35]{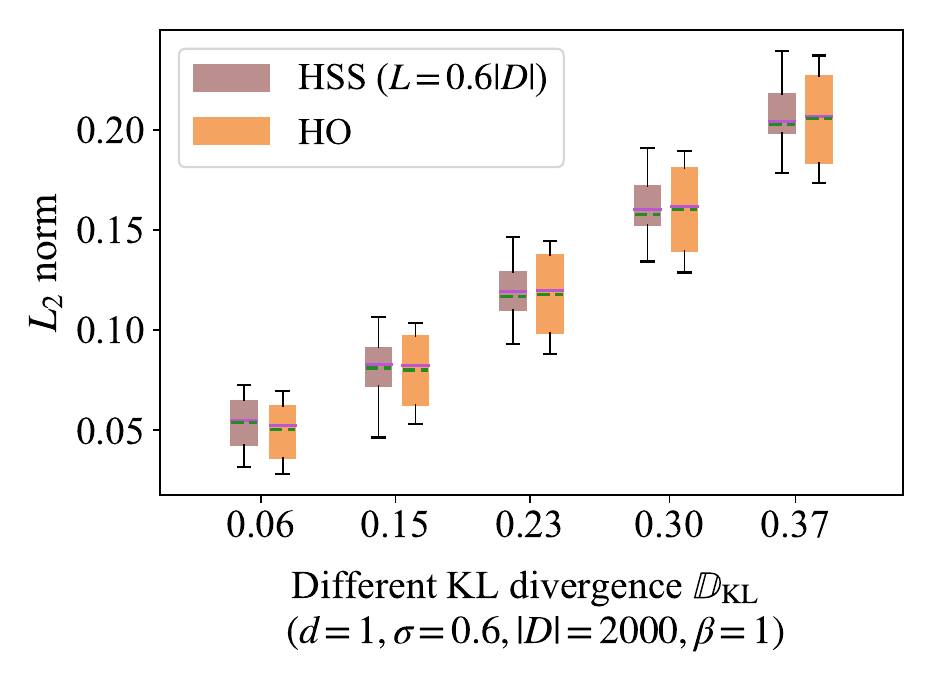}	\label{fig:covariate_shift_L2norm2}} 
	\setlength{\subfigcapskip}{-0.5em}
	\hspace{0.005in} 
	\subfigure{\includegraphics[scale=0.35]{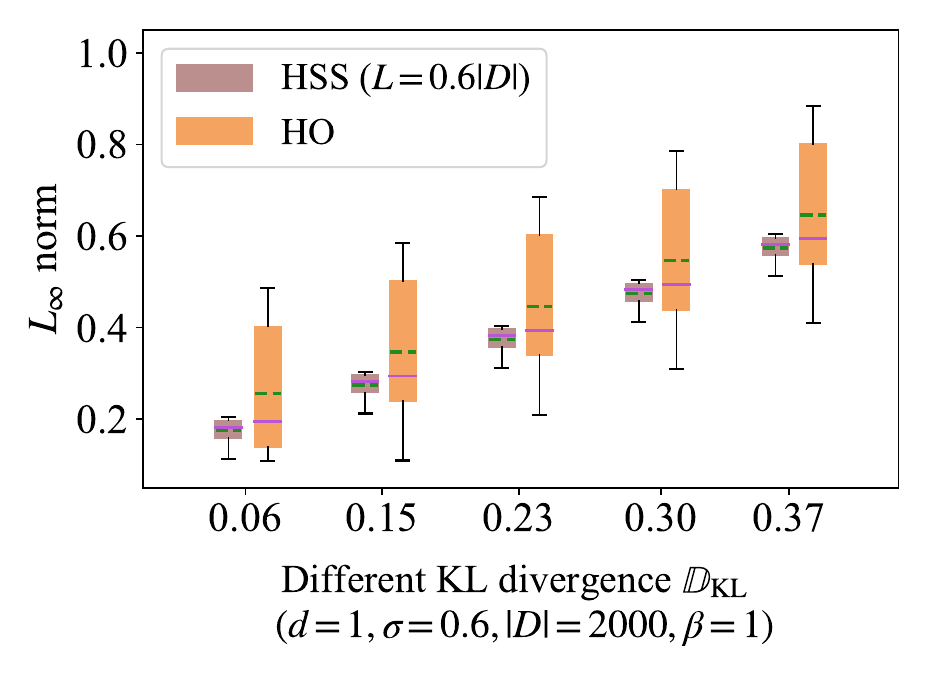}\label{fig:covariate_shift_Linftynorm2}} 
	\vspace{-0.05in}
	\caption{\footnotesize Results under covariate shift.}\label{fig:HSS_HO_covariate_shift}
\end{figure*}

We further explored HSS's performance with different $L$  under covariate shift, as shown in Figure \ref{fig:HSS_covariate_shift}. The dashed line represents HSS's performance when $L=0.5|D|$. We found that using a smaller $L$ often provides a slight advantage in prediction performance compared to a larger $L$  under covariate shift, especially in the $L_{\infty}$ norm. This may be because using less training data to select  $L$  may help avoid overfitting, thus improving performance under covariate shift.

\begin{figure}[H]
	\centering
	\setlength{\subfigcapskip}{-0.5em}
	\subfigure{\includegraphics[scale=0.28]{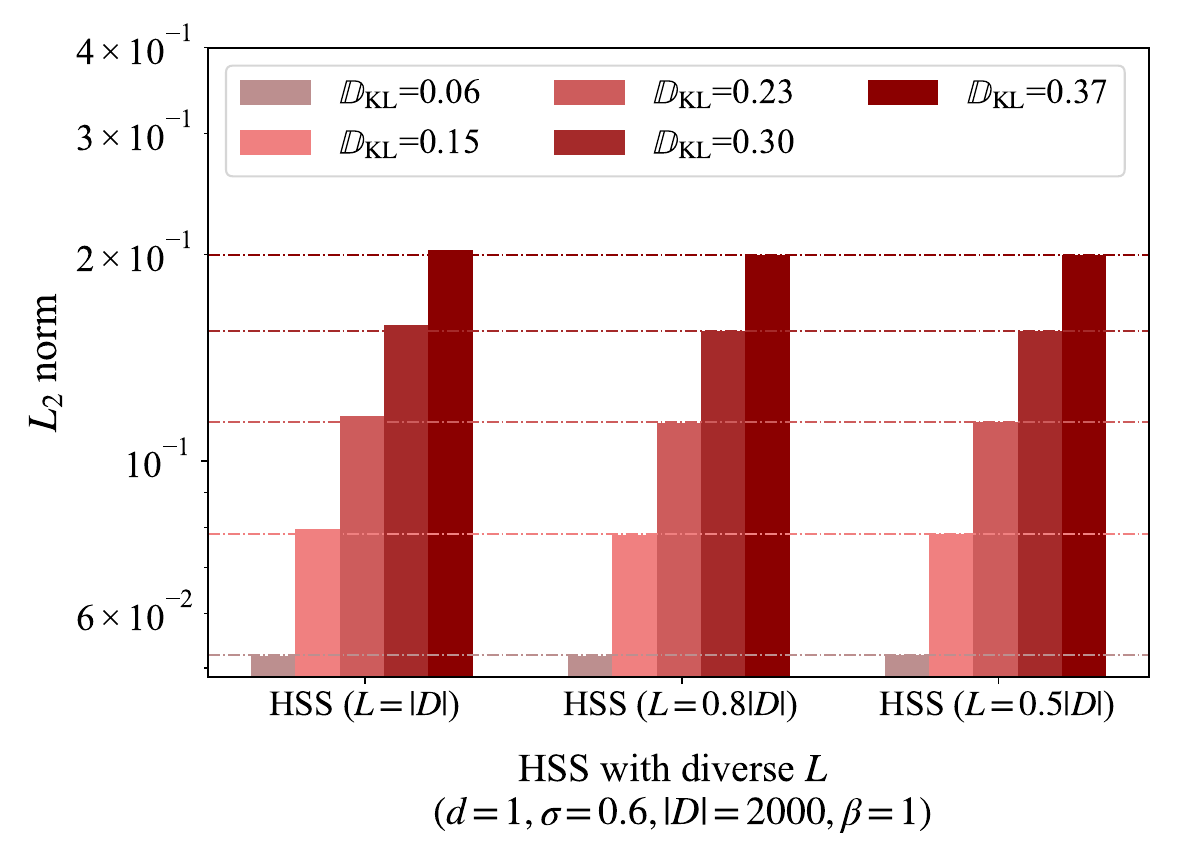}	\label{fig:covariate_shift_L2norm3}} 
	\setlength{\subfigcapskip}{-0.5em}
	\hspace{0.01in} 
	\subfigure{\includegraphics[scale=0.28]{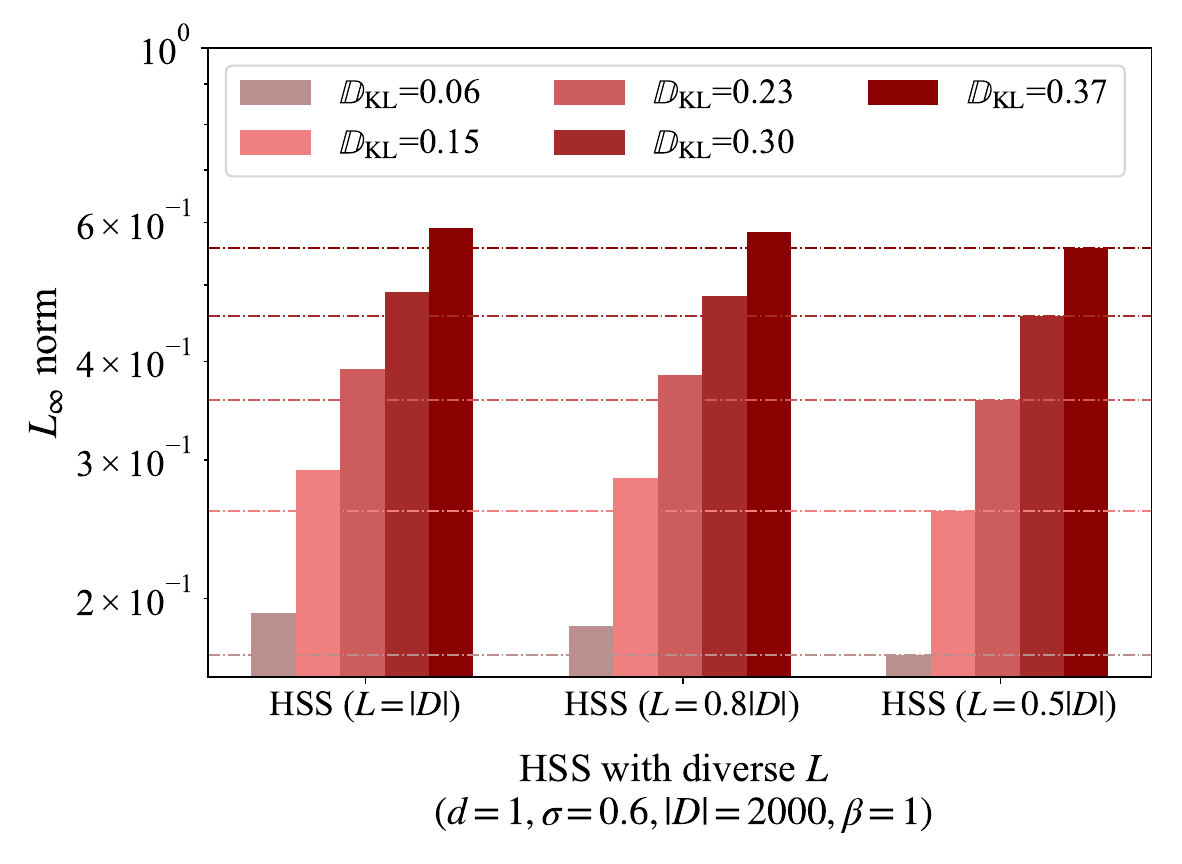}\label{fig:covariate_shift_Linftynorm3}} 
	\vspace{-0.05in}
	\caption{\footnotesize Results under covariate shift.}\label{fig:HSS_covariate_shift}
\end{figure}

\subsection{Real data examples}
Lastly, we evaluate the proposed HSS on two real-world datasets under the $L_2$ norm. These datasets pertain to the magnetic field information on the Earth's surface at varying latitudes, longitudes, and altitudes, specifically the magnetic total intensity dataset and the magnetic declination dataset.  They are crucial for various applications, including navigation, positioning, deological exploration, and more.
These two datasets are collected from the website \url{https://geomag.bgs.ac.uk/data_service/models_compass/igrf_calc.html}, whose geomagnetic ﬁeld information is provided by the 13th generation of IGRF (IGRF-13) \citep{alken2021international} based on observations recorded by satellites and ground observatories. We only use the latitude $\phi$, longitude $\theta$, altitude $h$, corresponding total intensity (nT)  and declination (degree) attributes on August 15th, 2024 in experiment.

For generating training samples, we first draw 2000 samples according to the uniform distribution on the (hyper-)cube $[-1, 1]^d$ with $d=3$, and then use the same data processing and collection strategy as in \citep{liu2024weighted} to collect total intensity and declination data from the above website. 
We collect 2664 samples as testing samples from locations where $\phi$ and $\theta$ are sampled every ﬁve degrees and $h$ is ﬁxed at 0km for visualizing them in Miller cylindrical projection. Similar to \citet{liu2024weighted}, we introduce truncated Gaussian noise with a standard deviation $\sigma= 500$ to the training samples of total intensity data and $\sigma=20$ to the training samples of declination data, resulting in noisy data.  We apply the same parameter selection for the constant $\tilde C$ in HSS as used in the toy simulations. For a fair comparison among HSS, BS, and HO, their step size $\beta$ is set to 45 for total intensity data and set to 20 for declination data.
Each experiment is conducted ﬁve times for averaging.

We set $L=0.8 \times 2000$ for the magnetic total intensity data and $L=0.9 \times 2000$ for the magnetic declination data in HSS. Figure \ref{fig:geo_bar} shows that HSS performs nearly the same as BS and better than HO (especially on the total intensity data), which demonstrates the advantage of HSS. Figure \ref{fig:geo_total_intensity_visualization} presents a visualization of global maps related to total intensity, including the total intensity data given by IGRF-13 \citep{alken2021international}, total intensity fitted data by HO and HSS. We have highlighted two areas with light white rectangles in each map to highlight their difference. Compared with the ground truth, HSS provides predictions closer to those given by IGRF-13.

\begin{figure*}[t]
	\centering
	\setlength{\subfigcapskip}{-0.5em}
	\subfigure{\includegraphics[scale=0.35]{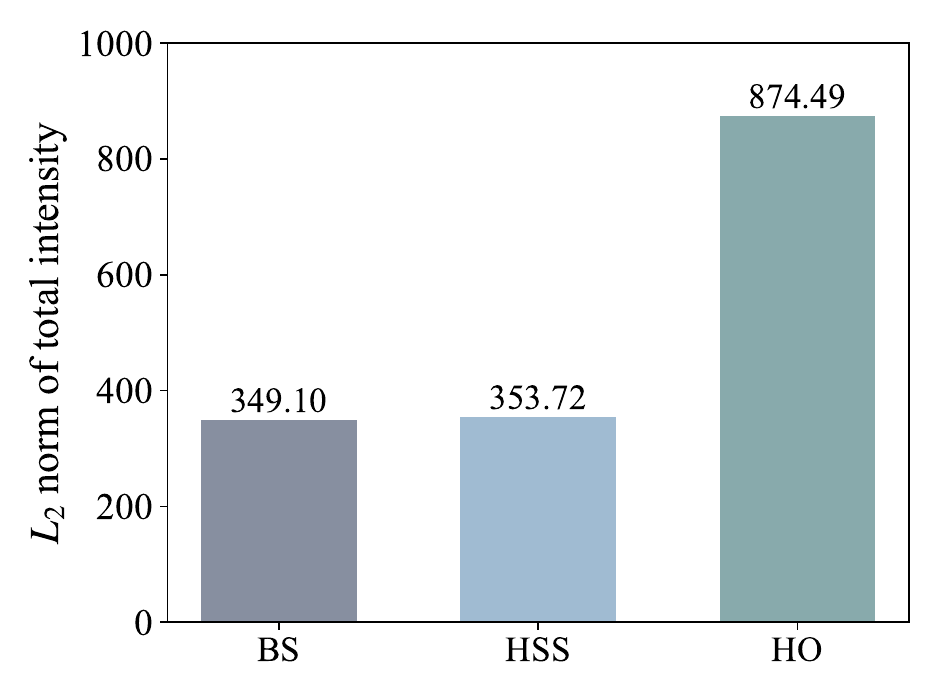}	\label{fig:geo_total_intensity}} 
	\setlength{\subfigcapskip}{-0.5em}
	\hspace{0.01in} 
	\subfigure{\includegraphics[scale=0.35]{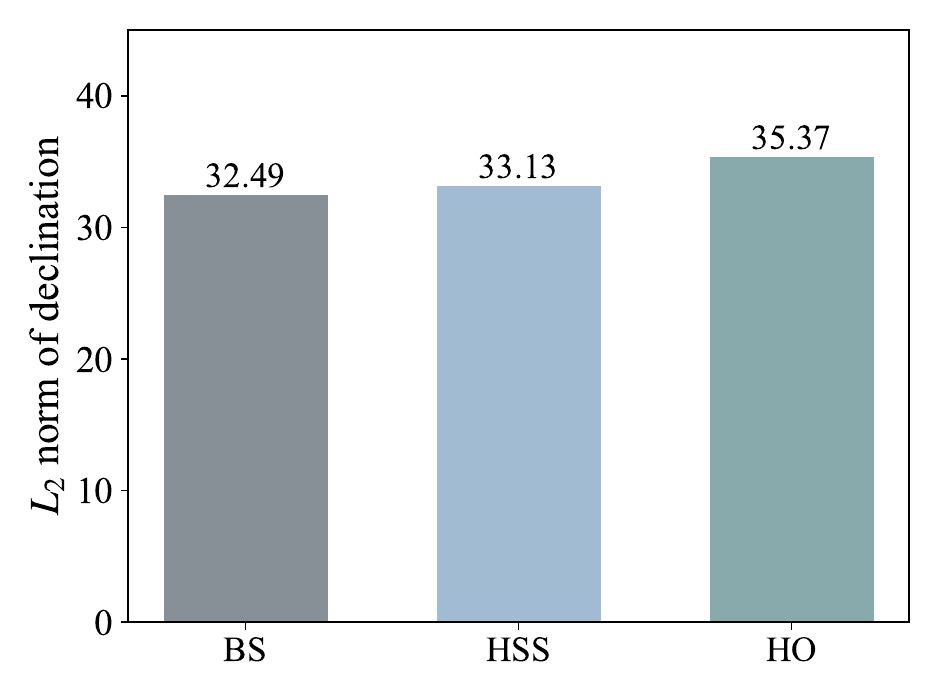}\label{fig:geo_declination}} 
	\vspace{-0.05in}
	\caption{\footnotesize Generalization performance of BS, HO, and HSS for the total intensity data and declination data of magnetic ﬁeld. 
	 In HSS, $L$ is set to $0.8|D|$ for the total intensity data and $0.9|D|$ for the declination data.}\label{fig:geo_bar}
\end{figure*}

\begin{figure*}[t]
	\centering
	\setlength{\subfigcapskip}{-0.5em}
	\subfigure{\includegraphics[scale=0.45]{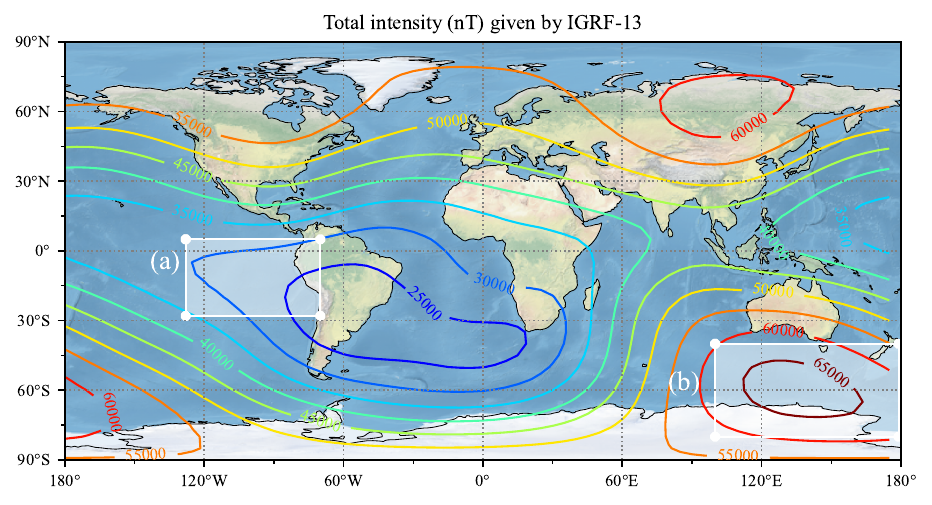}	\label{fig:geo_total_intensity_test}} 
	\setlength{\subfigcapskip}{-0.5em}
	\hspace{0.01in} 
	\vspace{-0.05in}
	\subfigure{\includegraphics[scale=0.45]{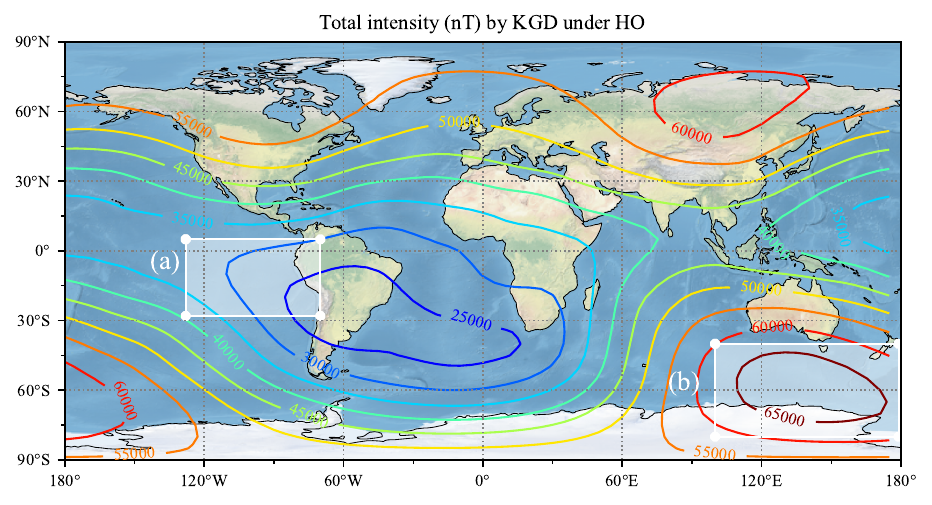}\label{fig:geo_total_intensity_HO}} 
	\hspace{-0.01in} 
	\vspace{-0.05in}
	\subfigure{\includegraphics[scale=0.45]{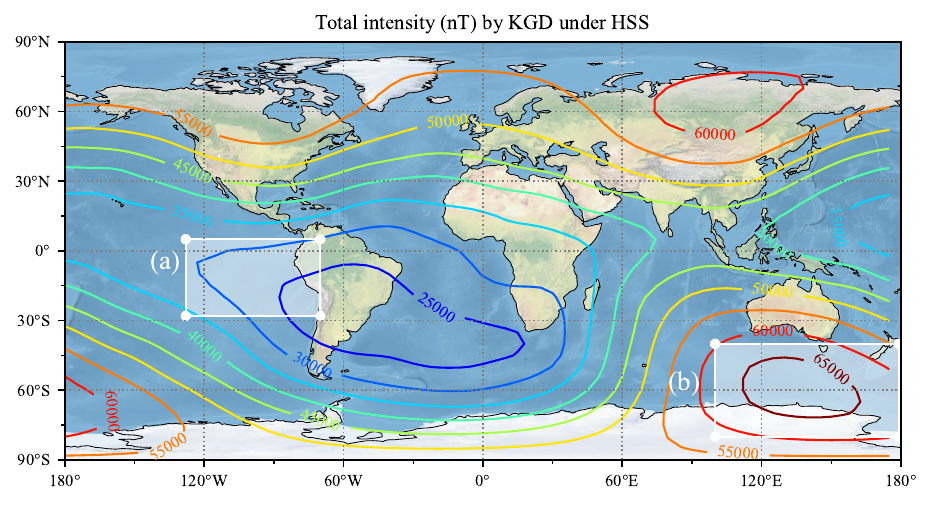}	\label{fig:geo_total_intensity3_HSS}} 
	\setlength{\subfigcapskip}{-0.5em}
	\vspace{0.01in}
	\caption{\footnotesize Maps of magnetic total intensity given by IGRF-13 (top panel) and predicted by KGD under HO (middle panel) and  KGD under HSS (bottom panel) at the WGS84 ellipsoid surface on August 15th, 2024. The projection used is the Miller cylindrical projection.}\label{fig:geo_total_intensity_visualization}
\end{figure*}




\section{Further discussion}\label{Sec.Conclusion}
The semi-adaptive nature of bias–variance analysis methods and the sub-optimality of information entropy methods limit their effectiveness in parameter selection. 
In this context, splitting methods—particularly cross-validation and hold-out—remain dominant in practice.
However, splitting methods often perform poorly in practice due to wasting a portion of samples to create the validation set. 
Without discarding any samples, we propose the HSS method that combines the advantages of both bias–variance analysis and splitting methods, achieving optimal generalization error bounds for all the regularity index $r\in[1/2,\infty)$, capacity index $s\in (0,1]$, and diﬀerent metrics of error ($\|\cdot\|_\rho$, $\|\cdot\|_D$ and $\|\cdot\|_K$ ).

Notably, as HSS does not rely on the entire dataset for constant selection, future research could explore its application in distributed learning systems, where each local agent selects constants based on its own local data. Building on the theoretical results in this paper, combining HSS with distributed learning has the potential to contribute to the development of privacy-preserving distributed kernel-based gradient descent algorithms, while maintaining high prediction accuracy.

Furthermore, given the substantial body of research on kernel methods for spherical data (such as the magnetic total intensity data used in our real-data experiments) and the fact that many real-world spherical datasets are acquired through deterministic sampling (e.g., satellite deterministic sampling), investigating a fully adaptive parameter selection strategy for KGD specifically designed for such spherical data would be an interesting direction for future research.

\section*{Acknowledgment}
The work is supported partially by
the National Natural Science Foundation of China
({Grant Numbers 62276209)}.

\bibliography{reference}

 \newpage

 \section*{Appendix A: Operator representations and operator concentrations}
\setcounter{lemma}{0}
\renewcommand{\thelemma}{A.\arabic{lemma}} 
\renewcommand{\theequation}{A\arabic{equation}} 
\setcounter{equation}{0} 
\renewcommand{\theproposition}{A.\arabic{proposition}} 
\renewcommand{\thecorollary}{A.\arabic{corollary}} 

The main tool in our analysis is the integral operator approach developed in \citep{smale2007learning,lin2017distributed,guo2017learning}. Let $S_{D}:\mathcal H_K\rightarrow\mathbb R^{|D|}$ be the sampling
operator \citep{smale2007learning} defined by
$$
S_{D}f:=(f(x))_{(x,y)\in D}.
$$
Its scaled adjoint $S_D^{\top}:\mathbb R^{|D|}\rightarrow
\mathcal H_K$  is
given by
$$
S_{D}^{\top}{\bf c}:=\frac1{|D|}\sum_{i=1}^{|D|}c_iK_{x_i},\qquad {\bf
	c}:=(c_1,c_2,\dots,c_{|D|})^{\top}
\in\mathbb
R^{|D|}.
$$
Define $L_{K,D}:\mathcal H_K\rightarrow\mathcal H_K$   the empirical version of the integral operator $L_K$ by
$$
L_{K,D}f:=\frac1{|D|}\sum_{(x,y)\in D}f(x)K_x.
$$
{For an} arbitrary $f\in\mathcal H_K$, it is easy to check
\begin{equation}\label{emprical-norm}
\|f\|_{D} = \|L_{K,D}^{1/2}f\|_K,\qquad\textnormal{and}\quad
\|f\|_{\rho}=\|L_{K}^{1/2}f\|_K.
\end{equation}
The gradient descent algorithm
(3)   can be rewritten as
\begin{equation}\label{Gradient Descent algorithm111}
\begin{aligned}
    f_{t+1,\beta,D} = & \ f_{t,\beta,D} - \beta \left(L_{K,D} f_{t,\beta,D} - S_D^{\top} y_D \right) 
    =  \ \left(I - \beta L_{K,D}\right) f_{t,\beta,D} + \beta S_D^{\top} y_D.
\end{aligned}
\end{equation}
Direct computation \citep{yao2007early,lin2018distributed}  yields
\begin{equation}\label{Spectral gradient}
f_{t,\beta,D}=\sum_{k=0}^{t-1}\beta\pi_{k+1,\beta}'(L_{K,D})S_D^{\top}y_D
=g_{t,\beta}(L_{K,D})S_D^{\top}y_D,
\end{equation}
where
$\pi_{k+1,\beta}'(u):=\Pi_{\ell=k+1}^{t-1}(1-\beta u)=(1-\beta
u)^{t-k-1}$, $\pi_{t,\beta}':=1$
and
\begin{equation}\label{filter op for GD}
g_{t,\beta}(u)=\sum_{k=0}^{t-1}\beta\pi_{k+1,\beta}'(u)=\frac{1-(1-\beta
	u)^t}{u}, \qquad \forall u>0.
\end{equation}

The operator representation of KGD by (\ref{Spectral gradient}) is significant, which will be used {throughout the analysis} of KGD. {The following lemma that is well known in inverse problems  \cite[Chap.6]{engl1996regularization}
	and nonparametric regression \citep{bauer2007regularization,yao2007early,lin2018distributed}
	shows some important properties for this representation.
	
	\begin{lemma}\label{Lemma:Spectral property 1}
		Let $g_{t,\beta}$ be defined by (\ref{filter op for GD}). We have
		\begin{equation}\label{spectral property 1}
		\|L_{K,D}g_{t,\beta}(L_{K,D})\|\leq 1,\qquad
		\frac1t\|g_{t,\beta}(L_{K,D})\|\leq \beta,
		\end{equation}
		and
		\begin{equation}\label{spectral property}
		\|L_{K,D}^v(I-L_{K,D}g_{t,\beta}(L_{K,D}))\|\leq
		\left(\frac{v}{t\beta}\right)^v ,\qquad v\geq 0.
		\end{equation}
\end{lemma}}

Based on the operator representation, we use the operator difference to describe the bias, variance and generalization.  The following lemma combines several standard concentration inequalities in \citep{caponnetto2007optimal,blanchard2016convergence,guo2017learning,blanchard2019lepskii} (see also
\cite[Lemma 22]{lin2019boosted}). 

\begin{lemma}\label{Lemma:Q}
	Let $\delta\in(0,1)$  and $D$ be a dataset with samples drawn independently according to $\rho$. Under (15), with confidence at least
	$1-\delta$,
    \begin{small}
	\begin{eqnarray*}
		\mathcal S_{D,t}\!\!\! &\leq& \!\!C_1^*\left(
		\frac{\log \max\{1,\mathcal N(t^{-1})\}}{t^{-1}|D|}\!+ \!\sqrt{\frac{\log\max\{1,\mathcal N(t^{-1})\}}{t^{-1} |D|}}\!\right),  \\
		\mathcal R_D   &\leq&
		\frac{4\kappa^2}{\sqrt{|D|}}\log\frac8\delta,\\
		\mathcal P_{D,t} &\leq&  2(\kappa M +\gamma) \mathcal A_{D,\lambda} \log
		\bigl(8/\delta\bigr), 
	\end{eqnarray*}
    \end{small}
and
\begin{small}
\begin{align*}
    & (1+4\eta_{\delta/4})^{-1} \sqrt{\max\{\mathcal N(t^{-1}),1\}}  
     \leq \sqrt{\max\{\mathcal N_D(t^{-1}),1\}} \\
    & \quad \leq (1+4\sqrt{\eta_{\delta/4}}\vee\eta_{\delta/4}^2)\sqrt{\max\{\mathcal N(t^{-1}),1\}}
\end{align*}
    \end{small}
	hold  simultaneously, where  
	$C_1^*:=\max\{(\kappa^2+1)/3,2\sqrt{\kappa^2+1}\}$,
	\begin{equation}\label{Def.S}
	\mathcal S_{D,t}:=\|(L_K+t^{-1} I)^{-1/2}(L_K-L_{K,D})(L_K+t^{-1}I)^{-1/2}\|,
	\end{equation}
	\begin{equation}\label{Def.R}
	\mathcal R_D:=\|L_K-L_{K,D}\|_{HS},
	\end{equation}
	\begin{equation}\label{Def.P}
	\mathcal P_{D,t}:=
	\left\|(L_K+t^{-1}
	I)^{-1/2}(L_Kf_\rho-S^{\top}_{D}y_D)\right\|_K,
	\end{equation}
$\eta_\delta:=2\log(16/\delta)\sqrt{t}/\sqrt{|D|}$ and  $\|A\|_{HS}$ is  the  Hilbert-Schmidt norm of the  Hilbert-Schmidt operator $A$.	
\end{lemma}

From the above lemma, we can derive an upper bound of $\mathcal Q_{D,t}$.

\begin{lemma}\label{Lemma:Q1}
	For any $\delta\in (0,1)$, if $C_1^*\mathcal U_{D,t,\delta}\leq 1/2$,
	then with confidence $1-\delta$, there holds
	$$
	\mathcal Q_{D,t}\leq \sqrt{2}.  
	$$
\end{lemma}

\begin{proof}  Direct computation yields
\begin{small}
\begin{align*}
    &(L_K+\lambda I)^{1/2}(L_{K,D}+\lambda I)^{-1}(L_K+\lambda I)^{1/2} \\
    &= (L_K+\lambda I)^{1/2} \big[(L_{K,D}+\lambda I)^{-1} - (L_{K}+\lambda I)^{-1}\big] (L_K+\lambda I)^{1/2} + I \\
    &= I + (L_K+\lambda I)^{-1/2}(L_K - L_{K,D})(L_K+\lambda I)^{-1/2} (L_K+\lambda I)^{1/2}(L_{K,D}+\lambda I)^{-1}(L_K+\lambda I)^{1/2}.
\end{align*}
\end{small}
	We then have from \eqref{Def.S} that
\begin{small}
\begin{align*}
    & \|(L_K+\lambda I)^{1/2}(L_{K,D}+\lambda I)^{-1}(L_K+\lambda I)^{1/2}\| \\
    & \quad \leq 1 + \mathcal S_{D,t} \, 
    \|(L_K+\lambda I)^{1/2}(L_{K,D}+\lambda I)^{-1}(L_K+\lambda I)^{1/2}\|.
\end{align*}
\end{small}
	The only thing remaining is to present a restriction on $t$ so that $\mathcal S_{D,t}<1$. For this purpose, recall  Lemma \ref{Lemma:Q} and get that with confidence $1-\delta$, there holds
    \begin{small}
\begin{eqnarray*}
    \mathcal S_{D,t} &\leq& C_1^*\left(
    \frac{\log\max\{1,\mathcal N(t^{-1})\}}{t^{-1}|D|} 
    + \sqrt{\frac{\log\max\{1,\mathcal N(t^{-1})\}}{t^{-1} |D|}}\right) 
    \leq C_1^* \mathcal U_{D,t,\delta}.
\end{eqnarray*}
\end{small}
	Therefore, for any $t$ satisfying 
	$$
	C_1*\mathcal U_{D,t,\delta}\leq 1/2,
	$$
	we have $\mathcal Q_{D,t}\leq \sqrt{2}$. This completes the proof of Lemma \ref{Lemma:Q1}.
\end{proof}


From Lemma \ref{Lemma:Q} and Lemma \ref{Lemma:Q1}, we can establish a relation between operator differences and $\mathcal W_{D,t}$, {which is similar as that in \cite[Lemma 23]{lin2019boosted}}

\begin{lemma}\label{Lemma:important for stopping}
	For any $0<\delta<1$ and $1\leq t\leq T$, under (15), with
	confidence $1-\delta$, there holds
    \begin{small}
	\begin{equation}\label{important-for-stopping}
	\mathcal P_{D,t} \mathcal Q_{D,t}
	\leq 2\sqrt{2}(\kappa M +\gamma) \mathcal A_{D,\lambda} \log
		\frac{8}{\delta} \leq\sqrt{2}\mathcal W_{D,t}\log^2\frac{16}\delta.
	\end{equation}
    \end{small}
\end{lemma}
    
\begin{proof}   The first inequality follows from Lemma \ref{Lemma:Q} and Lemma \ref{Lemma:Q1} directly.
We then turn to proving the second inequality in \eqref{important-for-stopping}.
It follows from Lemma \ref{Lemma:Q} and
	(30) that  with confidence $1-\delta$,
	$$
	\mathcal A_{D,t}\leq\frac{\sqrt{t}}{|D|}+\frac{\sqrt{\max\{\mathcal N_D(t^{-1}),1\}}(1+8\sqrt{t/|D|})}{\sqrt{|D|}}\log\frac{16}\delta.
	$$
	Then, we obtain from Lemma \ref{Lemma:Q}    that with confidence $1-\delta$, 
    \begin{small}
    $$
\mathcal P_{D,t} \!\leq\!
2(\kappa M+\gamma)\!\left(\!
\frac{\sqrt{t}}{|D|} \!+\! \frac{\sqrt{\max\!\{\mathcal N_D(t^{-1}),1\}}\!(1\!+\!8\!\sqrt{t/|D|})}{\sqrt{|D|}}
\!\right)\!\log^2\!\frac{16}{\delta}.
$$
\end{small}
	Thus, for any $1\leq t\leq T$, we have from Lemma \ref{Lemma:Q1} that with confidence $1-\delta$, there holds
    \begin{eqnarray*}
    && \mathcal P_{D,t} \mathcal Q_{D,t} \leq
    2\sqrt{2}(\kappa M+\gamma) \log^2\frac{16}{\delta}  \left(
    \frac{\sqrt{t}}{|D|}
    + \frac{\sqrt{\max\{\mathcal N_D(t^{-1}),1\}}(1+8\sqrt{t/|D|})}{\sqrt{|D|}}
    \right).
\end{eqnarray*}
	This together with (6) completes the proof of Lemma
	\ref{Lemma:important for stopping}.
\end{proof}

\section*{Appendix B: Proof of propositions and corollaries}

\setcounter{lemma}{0}
\renewcommand{\thelemma}{B.\arabic{lemma}}
\renewcommand{\theequation}{B\arabic{equation}} 
In this section, we focus on proving   propositions.

\begin{proof}[Proof of Proposition \ref{Proposition:bias--variance-via-t}]
	Due to (\ref{noise-free})  and (\ref{Spectral gradient}),
	we have
	\begin{equation}\label{spectal-noisefree}
	f^\diamond_{t,\beta,D}=g_{t,\beta}(L_{K,D})L_{K,D}f_\rho.
	\end{equation}
	Then (\ref{DEF.BIAS}), (16) with $r\geq 1/2$ and (\ref{emprical-norm})  yield
        \begin{small}
	\begin{eqnarray}\label{case1:bias}
	\mathcal B_{t,\beta,D} &\!\leq\!&
	2\|(L_{K,D}\!+\!t^{-1}I)^{1/2}(g_{t,\beta}(L_{K,D})L_{K,D}\!-\!I)f_\rho\|_K \nonumber\\
	&\!\leq\!&
	2\|(L_{K,D}\!+\!t^{-1}I)^{1/2}(g_{t,\!\beta}\!(L_{K,D})L_{K,D}\!-\!I)L_K^{r-1/2}\|\|h_\rho\|_\rho.\nonumber
	\end{eqnarray}
    \end{small}
	If $\frac12\leq r\leq 1$, we then have from Lemma \ref{Lemma:Spectral property 1} and the Cordes inequality \citep{fujii1993norm}
	\begin{equation}\label{Cordes inequality}
	\|A^\alpha B^\alpha\|\leq
	\|AB\|^\alpha,\qquad 0\leq\alpha\leq 1
	\end{equation}
	for  all symmetric and positive-definite positive operators $A$ and $B$  that
    \begin{small}
    \begin{eqnarray*}
    &&
    \|(L_{K,D}+t^{-1})^{1/2}(g_{t,\beta}(L_{K,D})L_{K,D}-I)L_K^{r-1/2}\| \\
    &\leq&\!\!   \!\!
    \|(L_{K,D}\!\!+\!\!t^{-1})^{1/2}(g_{t,\!\beta}(L_{K,\!D})L_{K,D}\!\!-\!\!I)(L_{K,D}+t^{-1}I)^{r-1/2}\| \mathcal Q_{D,t}^{2r-1} \\
    &\leq&
      \!\!\!\! 2^r\!\mathcal Q_{D,t}^{2r-1}\!\left(\|L_{K,D}^r(g_{t,\!\beta}(L_{K,\!D})L_{K,D}\!\!-\!\!I)\| 
    \!\!+\!\! t^{-r}\!\|g_{t,\beta}(L_{K,\!D})L_{K,\!D}\!\!-\!\!I\|\right) \\
    &\leq&\!\!   \!\!
    2^{r}\!\mathcal Q_{D,t}^{2r-1}\!\left((r/\beta)^r+1\right)\!t^{-r}.
    \end{eqnarray*}
     \end{small}
	If $r>1$, we obtain from Lemma \ref{Lemma:Spectral property 1} again, the bounds $\|L_{K,D}\|\leq\kappa^2$, $\|L_K\|\leq\kappa^2$, and the following   Lipschitz inequality \citep{blanchard2016convergence}
        \begin{small}
\begin{eqnarray}\label{Operator difference big r}
 &&\|L_{K,D}^{r-1/2}\!-\!L_K^{r-1/2} \| 
 \leq  \|L_{K,D}^{r-1/2}\!-\!L_K^{r-1/2} \|_{HS} \nonumber\\
&\leq& \max\{1,(r\!-\!1/2)\kappa^{2r-3}\}\|L_{K,D}-L_K\|_{HS}^{\min\{1,r-1/2\}} 
\end{eqnarray}
\end{small}
	that
 \begin{small}
\begin{eqnarray*}
    &&
    \|(L_{K,D}\!+\!t^{-1})^{1/2}(g_{t,\beta}(L_{K,D})L_{K,D}\!-\!I)L_K^{r-1/2}\|\\
    &\leq&
    \|(L_{K,D}\!+\!t^{-1})^{1/2}L_{K,D}^{r-1/2}(g_{t,\beta}(L_{K,D})L_{K,D}\!-\!I)\| \\
    && \!+\!\|(L_{K,D}\!+\!t^{-1})^{1/2}(g_{t,\beta}(L_{K,D})L_{K,D}\!-\!I)\|\|L_K^{r-1/2}\!-\!L_{K,D}^{r-1/2}\|\\
    &\leq&
    \left((r/\beta)^r\!+\!((r\!-\!1/2)/\beta)^{r-1/2}\right)
    t^{-r} \!+\! t^{-1/2}((1/(2\beta))^{1/2}\!+\!1) \\
    && \quad \times \max\{1,(r\!-\!1/2)\kappa^{2r-3}\}\mathcal R_D^{\min\{1,r\!-\!1/2\}}.
\end{eqnarray*}
\end{small}
	Plugging the above two estimates into (\ref{case1:bias}), we have
	$$
	\mathcal B_{t,\beta,D}
	\leq
	C'_1\left\{\begin{array}{cc}
	\mathcal Q_{D,t}^{2r-1}t^{-r},& \textnormal{if } 1/2\leq r\leq 1,\\
	t^{-r}+t^{-1/2}\mathcal R_D^{\min\{1,r-1/2\}},& \textnormal{if } r>1.
	\end{array}
	\right.
	$$
	Then, it follows from Lemma \ref{Lemma:Q1}  that with confidence $1-\delta$, there holds
	$
	\mathcal Q_{D,t}^{2r-1}\leq 2^{(2r-1)/2},\qquad \forall \ t\leq T
	$ 
	and
	$
	\mathcal R_D \leq \frac{4 \kappa^2}{\sqrt{|D|}}\log\frac8{\delta}.
	$
	This proves (28).
	To bound the variance, it follows from (27), (56), (48) and Lemma \ref{Lemma:Spectral property 1} that
    \begin{small}
    \begin{eqnarray*}
    &&\mathcal V_{t,\beta,D}\! \leq
   \! 2\|(L_{K,D}\!+\!t^{-1}I)^{1/2} \!\!g_{t,\beta}(L_{K,D})(L_{K,D}f_\rho\!\!-\!\!S_D^{\top}y_D)\|\!_K \\
    &\leq&
    2\|g_{t,\beta}(L_{K,D})(L_{K,D}\!+\!t^{-1}I)\|\mathcal Q_{D,t}\mathcal P_{D,t} \\
    &\leq&
    2(1\!+\!\beta)\mathcal Q_{D,t}\mathcal P_{D,t}.
\end{eqnarray*}
\end{small}
	Then it follows from Lemma \ref{Lemma:Q1} and Lemma \ref{Lemma:important for stopping}  that with confidence $1-\delta$, there holds
\begin{eqnarray*}
    \mathcal V_{t,\beta,D}&\leq&
     2\sqrt{2}(\kappa M +\gamma)(1+\beta) \mathcal A_{D,\lambda} \log
		\frac{8}{\delta}  
        \leq
	2(1+\beta)\mathcal W_{D,t}\log^2\frac{16}\delta.
\end{eqnarray*}
	This completes the proof of Proposition \ref{Proposition:bias--variance-via-t}.
\end{proof}
To prove Corollary \ref{Colloary:generalization-error-abc}, we need the
  following lemma derived in \citep{zhang2015divide}. 

\begin{lemma}\label{Lemma:eigen-to-effective}
If there exists $L\in\mathbb N$ such that  $\sigma_\ell=0$ for all $\ell\geq L$, then for any $t\in\mathbb N$, there holds $\mathcal N(t^{-1})\leq L$.
If there exists an $s\in(0,1]$ such that $\sigma_\ell\leq \ell^{-1/s}$, then $\mathcal N(t^{-1})\leq C_0t^s$ for some absolute constant $C_0>0$. If there exists $c_1,c_2\geq0$ such that $\sigma_\ell\leq c_1e^{-c_2\ell^2}$ for all $\ell=1,2,\dots,$ then $\mathcal N(t^{-1})\leq C_1\sqrt{\log t}$.
\end{lemma}
We then use the above lemma and Proposition \ref{Proposition:bias--variance-via-t} to prove Corollary \ref{Colloary:generalization-error-abc}.
\begin{proof}[Proof of Corollary \ref{Colloary:generalization-error-abc}]
Combining Lemma \ref{Lemma:eigen-to-effective} with (30), we get for any $t=1,\dots,T$ 
    \begin{small}
\begin{equation}\label{bound-adt}
     \mathcal A_{D,t}\leq
     \left\{\begin{array}{cc}
       \frac{2\sqrt{L}}{\sqrt{|D|}},   & \mbox{if} \ \sigma_\ell=0,\ell\geq L+1, \\
        \frac{C_0t^{s/2}}{\sqrt{|D|}}  & \mbox{if} \ \sigma_\ell\leq c_0\ell^{-1/s},\\
      \frac{C_1 {\log^{1/4} |D|}}{\sqrt{|D|}},&\mbox{if} \ \sigma_\ell\leq c_1e^{-c_2\ell^2}.
        \end{array} 
     \right.
\end{equation}
    \end{small}
Therefore, we get from Lemma 1 and Proposition \ref{Proposition:bias--variance-via-t} that
    \begin{small}
$$ 
    \mathcal V_{t,\beta,D} \leq 
    C'_1 \log
		\frac{8}{\delta} \left\{\begin{array}{cc}
       \frac{ \sqrt{L}}{\sqrt{|D|}},   & \mbox{if} \ \sigma_\ell=0,\ell\geq L+1, \\
        \frac{ t^{s/2}}{\sqrt{|D|}}  & \mbox{if} \ \sigma_\ell\leq c_0\ell^{-1/s},\\
      \frac{  {\log^{1/4} |D|}}{\sqrt{|D|}},&\mbox{if} \ \sigma_\ell\leq c_1e^{-c_2\ell^2},
        \end{array} 
     \right.
$$
    \end{small}
where $C'_1:=2\sqrt{2}(\kappa M +\gamma)(1+\beta)\max\{2,C_0,C_1\}.$  
We obtain from (28) and (31) that
    \begin{small}
$$
   \mathcal{B}_{t,\beta,D}\leq C_2'\log^2\frac{16}\delta
   \left\{
   \begin{array}{cc}
      (|D|/L)^{-1/2},  & \mbox{if} \ \sigma_\ell=0,\ell\geq L+1, \\
      |D|^{-r/(2r+s)}  & \mbox{if} \ \sigma_\ell\leq c_0\ell^{-1/s},\\
      (|D|/\sqrt{\log|D|})^{-1/2},& \mbox{if} \ \sigma_\ell\leq c_1e^{-c_2\ell^2},
   \end{array}\right.
$$
    \end{small}
where $C_2':=\max\{2^{r-1/2},1+4\kappa^2\}$. Combining the above two estimates, we get from Lemma 1 and (31) that (32) holds  with confidence $1-\delta$ and 
 $\tilde{C}:=C'_1+C_2'$.
 This completes the proof of Corollary \ref{Colloary:generalization-error-abc}.
\end{proof}

\begin{proof}[Proof of Proposition \ref{proposition:iterative error}]
	Due to  (\ref{Gradient Descent algorithm111}) and (\ref{Spectral gradient}), we have
    \begin{small}
    \begin{eqnarray}\label{error-dec-for-succ}	
    &&\|f_{t+1,\beta,D}\!-\!f_{t,\beta,D}\|_{D}+t^{-1/2}\|f_{t+1,\beta,D}\!-\!f_{t,\beta,D}\|_K \nonumber\\	
    &\leq& 2\beta\|(L_{K,D}\!+\!t^{-1}I)^{1/2}(L_{K,D}g_{t,\beta}(L_{K,D})\!-\!I)S_{D}^{\top}y_{D}\|_K \nonumber\\
    &\leq&\!\! 2\beta \|(L_{K,\!D}\!\!+\!\!t^{-1}I)^{1/2}(L_{K,\!D}g_{t,\beta}(L_{K,D})\!\!-\!\!I)(L_{K,\!D}f_\rho\!\!-\!\!S_D^{\top}y_D)\|\!_K \nonumber\\ 
    && +  2\beta\|(L_{K, D} \!\!+\!\! t^{-1}I)^{1/2} (L_{K, D}g_{t, \beta}(L_{K,D}) \!\!-\!\! I)L_{K,D}f_\rho\|_K.
\end{eqnarray}
\end{small}
	But  (\ref{spectral property})  yields
\begin{small}
\begin{eqnarray}\label{Varian-for-succ}	
    &&\|(L_{K,D}\!+\!t^{-1}I)^{1/2}(L_{K,D}g_{t,\beta}(L_{K,D})\!\!-\!\!I)(L_{K,D}f_\rho\!\!-\!\!S_D^{\top}y_D)\|_K \nonumber\\
    &\leq& \mathcal P_{D,t}\mathcal Q_{D,t}\| (L_{K,D}\!+\!t^{-1} I)(L_{K,D}g_{t,\beta}(L_{K,D})\!-\!I)\| \nonumber\\
    &\leq& \mathcal P_{D,t}\mathcal Q_{D,t} \left(\|L_{K,D}(L_{K,D}g_{t,\beta}(L_{K,D})\!-\!I)\| \right.  + \left. t^{-1}\|(L_{K,D}g_{t,\beta}(L_{K,D})\!-\!I)\|\right) \nonumber\\
    &\leq& (1\!+\!1/\beta)t^{-1} \mathcal P_{D,t}\mathcal Q_{D,t}.
\end{eqnarray}
\end{small}
	If (16) holds with $\frac12\leq r\leq 1$, then it follows from (\ref{spectral property}) and  (\ref{Cordes inequality}) that
    \begin{small}
    \begin{eqnarray}\label{Bias-for-succ-small}
    && \|(L_{K,D}\!+\!t^{-1}I)^{1/2}(L_{K,D}g_{t,\beta}(L_{K,D})\!-\!I)L_{K,D}f_\rho\|_K \nonumber\\
    &\!\leq\!&\! \|(L_{K,D}\!\!+\!\!t^{-1}I)^{1/2}L_{K,D}(L_{K,D}g_{t,\beta}(L_{K,D})\!\!-\!\!I)L_K^{r-1/2}\|\|h_\rho\|_\rho \nonumber \\
    &\leq& \|h_\rho\|_\rho\mathcal Q_{D,t}^{2r-1} 
    \|L_{K,D}(L_{K,D}g_{t,\beta}(L_{K,D})\!-\!I)(L_{K,D}\!+\!t^{-1} I)^{r}\| \nonumber \\
    &\leq& 2^{r}\|h_\rho\|_\rho\mathcal Q_{D,t}^{2r-1} \left(
    \|L_{K,D}^{r+1}(L_{K,D}g_{t,\beta}(L_{K,D})\!-\!I)\| \right. + \left. t^{-r}\|L_{K,D}(L_{K,D}g_{t,\beta}(L_{K,D})\!-\!I)\|\right) \nonumber \\
    &\leq& 2^{r}\|h_\rho\|_\rho\mathcal Q_{D,t}^{2r-1} \left(\left(\frac{r+1}{\beta}\right)^{r+1}\!+\!\frac{1}{\beta}\right) t^{-r-1}.
\end{eqnarray}
    \end{small}
	If (16) holds with $ r>1$,  then it follows from (\ref{spectral property}), (\ref{Operator difference big r}) and (25) that
    \begin{small}
\begin{eqnarray}\label{Bias-for-succ-large}
    && \|(L_{K,D}\!+\!t^{-1}I)^{1/2}(L_{K,D}g_{t,\beta}(L_{K,D})\!-\!I)L_{K,D}f_\rho\|_K \nonumber\\
    &\leq& \|(L_{K,D}\!+\!t^{-1}I)^{1/2}L_{K,D}(L_{K,D}g_{t,\beta}(L_{K,D})\!-\!I)L_K^{r-1/2}\|\|h_\rho\|_\rho \nonumber\\
    &\leq& \|(L_{K,D}\!+\!t^{-1}I)^{1/2}L_{K,D}(L_{K,D}g_{t,\beta}(L_{K,D})\!-\!I)L_{K,D}^{r-1/2}\|\|h_\rho\|_\rho \nonumber\\ 
    && + \|(L_{K,D}\!+\!t^{-1}I)^{1/2}L_{K,D}(L_{K,D}g_{t,\beta}(L_{K,D})\!-\!I) \times (L_K^{r-1/2}\!-\!L_{K,D}^{r-1/2})\| 
    \|h_\rho\|_\rho \nonumber\\
    &\leq& \|h_\rho\|_\rho \left[\left(\frac{r+1}{\beta}\right)^{r+1}\!+\!\left(\frac{r+1/2}{\beta}\right)^{r+1/2}\right]
    t^{-r-1} \nonumber\\
    && + \|(L_{K,D}\!+\!t^{-1}I)^{1/2}L_{K,D}(L_{K,D}g_{t,\beta}(L_{K,D})\!-\!I)\|  \|L_K^{r-1/2}\!-\!L_{K,D}^{r-1/2}\|_{HS} \|h_\rho\|_\rho \nonumber\\
    &\leq& \|h_\rho\|_\rho \left[
    \left(\frac{r+1}{\beta}\right)^{r+1}\!+\!\left(\frac{r+1/2}{\beta}\right)^{r+1/2}\right]
    t^{-r-1} \nonumber\\
    && + \max\{(r\!-\!1/2)\kappa^{2r-3},1\} \left[\left(\frac{3}{2\beta}\right)^{3/2}\!+\!\frac{1}{\beta}\right] 
    t^{-3/2} \mathcal R_D^{\min\{1,r\!-\!1/2\}}.
\end{eqnarray}
\end{small}
	Plugging (\ref{Bias-for-succ-large}), (\ref{Bias-for-succ-small}) and (\ref{Varian-for-succ}) into (\ref{error-dec-for-succ}), we obtain
    \begin{small}
	\begin{eqnarray}\label{err-est-succ}
	&&\|f_{t+1,\beta,D}-f_{t,\beta,D}\|_{D}
	\leq 
	(2+2\beta)t^{-1} \mathcal P_{D,t}\mathcal Q_{D,t}\nonumber\\
	&&+
	c_1\left\{\begin{array}{cc}
	\mathcal Q_{D,t}^{2r-1}( t)^{-r-1}, & \textnormal{if}\ \frac12\leq r\leq
	1,\\
	t^{-r-1}+t^{-3/2}\mathcal R_D^{\min\{1,r-1/2\}},& \textnormal{if}\ r>1.
	\end{array}\right.
	\end{eqnarray}
    \end{small}
	If, in addition, (15) holds and $t\leq T$, then Lemmas \ref{Lemma:Q1} and \ref{Lemma:important for stopping} yield directly that (33) holds with confidence $1-\delta$.
	This completes the proof of Proposition \ref{proposition:iterative error}.
\end{proof}

\begin{proof}[Proof of Proposition \ref{Prop:bound-t-semi}]
	It follows from Proposition \ref{proposition:iterative error}
	that
    \begin{small}
	$$ 
	\mathcal W_{D,\hat{t}}
	\leq  
	C_2'\left\{\begin{array}{cc}
	2^{r-1/2}  \hat{t}^{-r}, & \textnormal{if}\ \frac12\leq r\leq
	1,\\
	\hat{t}^{-r}+4\kappa^2\hat{t}^{-1/2}|D|^{-\min\{1/2,r/2-1/4\}},& \textnormal{if}\ r>1.
	\end{array}\right. 
	$$
     \end{small}
	Therefore, we have from
	(35) that 
	$\hat{t} \leq t^*$.
Then it follows from Lemma \ref{Lemma:eigen-to-effective} and Lemma \ref{Lemma:Q} that for any $t\leq T$, with confidence $1-\delta$, there holds
	\begin{eqnarray*}
		&& \frac{\sqrt{t}}{|D|}+\frac{\sqrt{\max\{\mathcal N_D(t^{-1}),1\}}(1+8\sqrt{t/|D|})}{\sqrt{|D|}}\\
		&\leq& \left(\frac{\sqrt{t}}{|D|}+\frac{(1+4(1+t/|D|))h(t)(1+8\sqrt{t/|D|})}{\sqrt{|D|}}\right)
		\log^2\frac{16}\delta,
	\end{eqnarray*}
where
\begin{equation}\label{def.ht}
    h(t):=\left\{\begin{array}{cc}
       \sqrt{L},  & \mbox{if} \ \sigma_\ell=0,\ell\geq L+1, \\
       C_0t^{s/2},  & \mbox{if} \ \sigma_\ell\leq c_0\ell^{-1/s},\\
       \sqrt{C_1}\log^{1/4}t, & \mbox{if} \ \sigma_\ell\leq c_1'e^{-c_2\ell^2}.\\
    \end{array}\right.
\end{equation}
	Then, (6) yields that with confidence $1-\delta$, there holds
	\begin{eqnarray}\label{WD-bound-population}
	\mathcal W_{D,t}\leq c_2'\left(\frac{\sqrt{t}}{|D|}+\frac{81h(t)}{\sqrt{|D|}}\right)\log^2\frac{16}{\delta} 
	=:\mathcal T_{D,t,\delta}.\nonumber
	\end{eqnarray}
	where $c_2':= 4\sqrt{2}(\kappa M+\gamma)$.
	Define $t^{**}:=t^{**}_{D,\beta}$ as the  largest integer satisfying
    \begin{small}
	\begin{equation}\label{inverse-t}
	\mathcal T_{D,t,\delta}
	\leq C_2'\left\{\begin{array}{cc}
	2^{r-1/2}  t^{-r}, \!&\! \textnormal{if}\ \frac12\leq \!r\!\leq
	1,\\
	\!t^{-r}\!+\!\!4\kappa^2t^{-1/2}|D|^{-\min\{1\!/\!2,r/2\!-\!1/4\}},\!&\! \textnormal{if}\ \!r\!>\!1.
	\end{array}\right.
	\end{equation}
        \end{small}
	Then, it follows from (35) that
	$$
	t^{**}\leq t^*.
	$$
	It is obvious that $\mathcal T_{D,t,\delta}$ is non-decreasing with respect to $t$ but the right-hand side of (\ref{inverse-t}) is non-increasing with $t$. 
    Therefore, for any $t\leq t^{**}$, \eqref{inverse-t} holds.
Setting 
$$
  t_0={c}_\delta\left\{\begin{array}{cc}
      (|D|/L)^{\frac1{2r}},    &  \mbox{if} \ \sigma_\ell=0,\ell\geq L+1,\\
       |D|^{\frac{1}{2r+s}},   & \mbox{if} \ \sigma_\ell\leq c_0\ell^{-1/s},\\
      (|D|/\sqrt{\log|D|})^{\frac{1}{2r}}, &  \mbox{if} \ \sigma_\ell\leq c_1e^{-c_2\ell^2},
     \end{array}\right.  
$$
with $0<c_\delta \leq 1$ depending on $\delta$ that will be determined below, we have
	\begin{eqnarray*}
		\mathcal T_{D,t_0,\delta}\leq c_2'  H(|D|)\log^2\frac{16}\delta
\end{eqnarray*}
	and
    \begin{small}
	$$
	c_3'c_\delta^{-r} H(|D|) \!\!\leq \!\!\left\{\begin{array}{cc}
	\!\!\!\!\!\!2^{r-1/2}  t_0^{-r}, &\!\!\!\!\!\!\!\!\!\!\!\!\!\! \textnormal{if}\ \frac12\!\leq\! r\leq
	1,\\
	\!\!\!t_0^{-r}+4\kappa^2t_0^{-1/2}|D|^{-\min\{1/2,r/2-1/4\}},& \!\!\!\!\textnormal{if}\ \!r\!>\!1,
	\end{array}\right.
	$$
    \end{small}
   where
\begin{equation}\label{def.Hd}
    H(|D|):= \left\{\begin{array}{cc}
      (|D|/L)^{-1/2},  & \mbox{if} \ \sigma_\ell=0,\ell\geq L+1,    \\
      |D|^{-r/(2r+s)} & \mbox{if} \ \sigma_\ell\leq c_0\ell^{-1/s},\\
      (|D|/\sqrt{\log|D|})^{-1/2},& \mbox{if} \ \sigma_\ell\leq c_1e^{-c_2\ell^2},
  \end{array}\right.
\end{equation}  
    $c_2'$ and $c_3'$ are constants depending only on $c_1'$, $\kappa$ and $c_0,c_1,c_2$. It is easy to derive that there exists $c_\delta$ such that \eqref{inverse-t} holds. Taking $\sigma_\ell\leq c_0\ell^{-1/2}$ for example, if we set
	$$
	c_\delta=\min\{1,c_3'/c_2'\}^{\frac{2}{2r+s}}\left(\log\frac{16}{\delta}\right)^{-\frac4{2r+s}} \leq 1,
	$$
then
	$$
	c_2' |D|^{-\frac{r}{2r+s}}c_\delta^{s/2}\log^2\frac{16}\delta
	\leq 
	c_3'c_\delta^{-r} |D|^{-\frac{r}{2r+s}},
	$$
	which implies that (\ref{inverse-t}) holds for $t_0$. The other two cases can be derived similarly. We remove the detailed proof for the sake of brevity. 
Therefore, we have  $t^*\geq  t^{**}\geq t_0$.
	This proves (38). 
	Then, it follows from Proposition \ref{Proposition:bias--variance-via-t}, Lemma 1, the definition of $t^*$ and (38)  that 
    \begin{small}
	\begin{eqnarray*}
	  &&\max\{\|f_{D,t^*}-f_\rho\|_\rho,  \|f_{D,t^*}-f_\rho\|_D, (t^*)^{1/2}\|f_{D,t^*}-f_\rho\|_K\}\\ 
		&\leq&
		2(1+\beta)\mathcal W_{D,t^*}\log^2\frac{16}\delta\\
		\!\!\!&+&\!\!\!\!
		C_1'\log^2\frac{16}\delta
		\left\{\begin{array}{cc}
			\!\!\!2^{r-1/2} \! t^{-r},&\!\! \!\!\!\!\!\!\!\!\!\!\!\textnormal{if }\frac12\leq r\leq 1,\\
			\!\!\!(\!t^*\!)^{-r}\!\!+\!\!4\kappa^2\!(\!t^*\!)^{-1/2}|D|^{\!-\!\min\{\!1\!/\!2\!,r/2-1/4\}},&\!\!\!\!\!\!\textnormal{if } r>1
		\end{array}
		\right.\\
		&\leq&
		\!\!\!\!C_3'\log^2\frac{16}\delta
		\left\{\begin{array}{cc}
			\!\!\!2^{r-1/2}  t_0^{-r},& \!\!\!\!\!\!\!\!\textnormal{if }\frac12\leq r\leq 1,\\
			\!\!\!t_0^{-r}\!\!+\!\!4\kappa^2t_0^{-1/2}|D|^{-\min\{1/2,r/2-1/4\}},&\!\!\!\!\!\!\textnormal{if } r>1
		\end{array}
		\right.\\
		&\leq&
		\tilde{C}' \log^2\frac{16}\delta  \left\{\begin{array}{cc}
      (|D|/L)^{-1/2},  & \mbox{if} \ \sigma_\ell=0,\ell\geq L+1,    \\
      |D|^{-r/(2r+s)} & \mbox{if} \ \sigma_\ell\leq c_0\ell^{-1/s},\\
      (|D|/\sqrt{\log|D|})^{-1/2},& \mbox{if} \ \sigma_\ell\leq c_1e^{-c_2\ell^2},
  \end{array}\right.
	\end{eqnarray*}
    \end{small}
	where $C_3'$ and $\tilde{C}'$ are constants independent of $\delta,|D|$ or $t$.
	This completes the proof of Proposition \ref{Prop:bound-t-semi}.
\end{proof}

\section*{Appendix C: Proof  of Theorem 6}
\setcounter{lemma}{0}
\renewcommand{\thelemma}{C.\arabic{lemma}}
\renewcommand{\theequation}{C\arabic{equation}} 

With the help of  Proposition \ref{Proposition:bias--variance-via-t}, Proposition \ref{proposition:iterative error} and Proposition \ref{Prop:bound-t-semi}, we can prove Theorem 6 as follows.
    \vspace{-7pt}

\begin{proof}[Proof of Theorem 6]
	  We divide the proof into two cases: 1) $\hat{t}< t_0$; 2)  $\hat{t}\geq t_0$, where $t_0$ is specified in (37). For the first case, i.e., $\hat{t}< t_0$, it follows from the triangle inequality that
	\begin{equation}\label{tri-rho}
	\|f_{\hat{t},\beta,D}-f_\rho\|_*
	\leq\|f_{\hat{t},\beta,D}-f_{t_0,\beta,D}\|_*+\|f_{t_0,\beta,D}-f_\rho\|_*,
	\end{equation}
where $\|\cdot\|_*$ denotes either $\|\cdot\|_\rho,t_0^{-1/2}\|\cdot\|_K$ or $\|\cdot\|_D$.
	But Corollary \ref{Colloary:generalization-error-abc} shows that
\begin{align}\label{t0-rho}
	&\max(  \|f_{t_0,\beta,D}\!\!-\!\!f_\rho\|_\rho,\|f_{t_0,\beta,D}\!\!-\!\!f_\rho\|_D,  
    t_0^{-\frac12} \|f_{t_0,\beta,D}\!\!-\!\!f_\rho\|_K) \nonumber
 \\
    &\leq
    \tilde{C}
  \log^2\frac{16}{\delta}\left\{\begin{array}{cc}
      (|D|/L)^{-1/2},  & \mbox{if} \ \sigma_\ell=0,\ell\geq L+1,    \\
      |D|^{-r/(2r+s)} & \mbox{if} \ \sigma_\ell\leq c_0\ell^{-1/s},\\
      (|D|/\sqrt{\log|D|})^{-1/2},& \mbox{if} \ \sigma_\ell\leq c_1e^{-c_2\ell^2},
  \end{array}\right.  
\end{align}
	holds with confidence $1-\delta$.
	Therefore, it suffices to bound $\|f_{\hat{t},\beta,D}-f_{t_0,\beta,D}\|_\rho$. From the triangle inequality again, we get
	$$
	\|f_{\hat{t},\beta,D}-f_{t_0,\beta,D}\|_*
	\leq
	\sum_{k=\hat{t}}^{t_0-1}\|f_{k+1,\beta,D}-f_{k,\beta,D}\|_*.
	$$
For an arbitrary $k=\hat{t},\dots,t_0-1$,  Lemma 1   shows that with confidence $1-\delta$,
there holds
\begin{eqnarray*}
    \| f_{k+1,\beta,D}\!\!-\!\!f_{k,\beta,D}\|_*
    \leq 2(1+\beta) \bar{c}_1 k^{-1} \mathcal W_{D,k} \log^4\frac{16}{\delta}.
\end{eqnarray*}
Since Lemma \ref{Lemma:eigen-to-effective} and   Lemma \ref{Lemma:Q}
yield  that with confidence $1-\delta$, 
\begin{small}
\begin{eqnarray*}
		&& \frac{\sqrt{t}}{|D|}+\frac{\sqrt{\max\{\mathcal N_D(t^{-1}),1\}}(1+8\sqrt{t/|D|})}{\sqrt{|D|}}\\
		&\leq&\!\!\!\! \left(\frac{\sqrt{t}}{|D|}\!+\!\frac{(1+4(1+t/|D|))h(t)(1+8\sqrt{t/|D|})}{\sqrt{|D|}}\right)
		\log^2\frac{16}\delta,
\end{eqnarray*}
\end{small}
where $h(t)$ is given by \eqref{def.ht},
	then, (6) implies that with confidence $1-\delta$, there holds
    \begin{small}
	\begin{eqnarray}\label{WD-bound-population2}
	\mathcal W_{D,t}\!\leq\! c_4'\left(\frac{\sqrt{t}}{|D|}\!+\!\frac{(1\!+\!4(1+t/|D|))h(t)(1\!+\!8\sqrt{t/|D|})}{\sqrt{|D|}}\!\right) 
	\end{eqnarray}
    \end{small}
	where $c_4':= 4\sqrt{2}(\kappa M+\gamma)$. Hence,
	for arbitrary $k\leq t_0\leq {\tilde{c}}|D|$,
	with confidence $1-\delta$, we get from \eqref{def.ht} that
    \begin{small}
	\begin{eqnarray*}
		&&\mathcal W_{D,k}\leq \bar{c}_2 \frac{h(t)}{\sqrt{|D|}}\left(1+\frac{k^{1/2}(h(k))^{-1}}{\sqrt{|D|}}\right)
        \leq 
       2 \bar{c}_2  \frac{h(t)}{\sqrt{|D|}} ,
	\end{eqnarray*}
    \end{small}
	where $\bar{c}_2:=c_4'(1+{\tilde{c}}) (\sqrt{\tilde{c}+1}+(5+4{\tilde{c}})( C_0+1+\sqrt{C_1})(1+8{\tilde{c}}))^2$.
If $\sigma_\ell\leq c_0\ell^{-1/s}$,	then {for any $k=\hat{t},\dots,t_0-1$}, $r\geq 1/2$ yields 
\begin{small}
\begin{eqnarray*} 
	{\| f_{t_0,\beta,D}-f_{\hat{t},\beta,D}\|_*}\nonumber
 &\leq& 8\bar{c}_1\bar{c}_2(1+\beta) \log^4\frac{16}{\delta} 
\sum_{k=\hat{t}}^{t_0-1} k^{-1} {\frac{h(k)}{\sqrt{|D|}}}\nonumber\\
\!\!\!\!\!\!\!\!\!\!\!\!\!\!\!\!&\leq&\!\! 8\bar{c}_1\bar{c}_2(1\!+\!\beta)(2s\!+\!1) \frac{t_0^{s/2}}{\sqrt{|D|}} \log^4\frac{16}{\delta} \nonumber\\
\!\!\!\!\!\!\!\!\!\!\!\!\!\!\!\!&\leq& \!\!\bar{c}_3 |D|^{-r/(2r+s)} \log^4\frac{16}{\delta}.
\end{eqnarray*}
\end{small}
	where
	$\bar{c}_3:=4\bar{c}_1\bar{c}_2 {\tilde{c}}^{s/2}(1+\beta)(2s+1)(1+{\tilde{c}}^{(1+s)/2}).$ 
If $\sigma_\ell=0,\ell\geq L+1,$ the same approach together with $h(k)=L$ yields
\begin{small}
$$
    {\| f_{t_0,\beta,D}-f_{\hat{t},\beta,D}\|_*}
    \leq
    \bar{c}_3 \frac{\sqrt{L}\log |D| }{|D|} \log^4\frac{16}{\delta}.
$$
\end{small}
Similar, for $\sigma_\ell\leq c_1e^{-c_2\ell^2}$, we get 
\begin{small}
$$
 {\| f_{t_0,\beta,D}-f_{\hat{t},\beta,D}\|_*}
    \leq
    \bar{c}_3 \frac{ \log |D| }{|D|} \log^4\frac{16}{\delta}. 
$$
\end{small}
Plugging the above three estimates and (\ref{t0-rho}) into (\ref{tri-rho}), we get with confidence $1-\delta$, there holds
\begin{small}
\begin{eqnarray*}
    &&\|f_{\hat{t},\beta,D}-f_\rho\|_*\leq
    \bar{c}_4  \log^4\frac{16}\delta  \left\{\begin{array}{cc}
       \frac{\sqrt{L}\log |D|}{\sqrt{|D|}},  & \mbox{if} \ \sigma_\ell=0,\ell\geq L+1,    \\
      |D|^{-r/(2r+s)} & \mbox{if} \ \sigma_\ell\leq c_0\ell^{-1/s},\\
      \frac{\log|D|}{\sqrt{|D|}},& \mbox{if} \ \sigma_\ell\leq c_1e^{-c_2\ell^2} \end{array}\right.
\end{eqnarray*}
\end{small}
	with $\bar{c}_4= \bar{c}_3+\tilde{C}.$
	Now, we turn to the second case: $\hat{t}\geq t_0$.    Under this circumstance, we get from  the definitions of $\hat{t}$, $t^*$ and Proposition \ref{Prop:bound-t-semi} that 
    \begin{small}
    \begin{eqnarray*}
   && \|f_{\hat{t},\beta,D}-f_\rho\|_\rho\\
    &\leq& 2(1\!+\!\beta) \mathcal W_{D,\hat{t}} \log^2\frac{16}{\delta} \!\! + \log^2\frac{16}{\delta} \left\{\begin{array}{ll}
       \!\! 2^{r-1/2} \hat{t}^{-r}, &\!\!\! \textnormal{if } \frac{1}{2} \leq r \leq 1, \\
       \!\! \hat{t}^{-r}\!+\!4\kappa^2\hat{t}^{-1/2}|D|^{-\min\{1/2, r/2\!-\!1/4\}}, & \!\!\!\textnormal{if } r\!>\!1
    \end{array}\right. \\
    &\!\!\leq& \!\!\!\bar{c}_8 \log^2\frac{16}{\delta} \left\{\begin{array}{ll}
       \!\! 2^{r-1/2} \hat{t}^{-r}, & \!\!\!\textnormal{if } \frac{1}{2} \leq r \leq 1, \\
       \!\! \hat{t}^{-r}\!+\!4\kappa^2\hat{t}^{-1/2}|D|^{-\min\{1/2, r/2\!-\!1/4\}}, & \!\!\!\textnormal{if } r\!>\!1
    \end{array}\right. \\
    &\!\!\leq&\!\!\! \bar{c}_8 \log^2\frac{16}{\delta} \left\{\begin{array}{ll}
       \!\! 2^{r-1/2} t_0^{-r}, &\!\!\! \textnormal{if } \frac{1}{2} \leq r \leq 1, \\
       \!\! t_0^{-r}\!+\!4\kappa^2t_0^{-1/2}|D|^{-\min\{1/2, r/2\!-\!1/4\}}, & \!\!\!\textnormal{if } r\!>\!1
    \end{array}\right. \\
    &\leq& \bar{c}_9 \log^2\frac{16}{\delta}\left\{\begin{array}{cc}
      (|D|/L)^{-1/2},  & \mbox{if} \ \sigma_\ell=0,\ell\geq L+1,    \\
      |D|^{-r/(2r+s)} & \mbox{if} \ \sigma_\ell\leq c_0\ell^{-1/s},\\
      (|D|/\sqrt{\log|D|})^{-1/2},& \mbox{if} \ \sigma_\ell\leq c_1e^{-c_2\ell^2},
  \end{array}\right.  
\end{eqnarray*}
    \end{small}
	where $\bar{c}_8,\bar{c}_9$ are constants independent  of $|D|,\delta$ or $t$.
	The bound of $\|f_{\hat{t},\beta,D}-f_\rho\|_D$ and $\|f_{\hat{t},\beta,D}-f_\rho\|_K$ can be derived by the same method. We remove the details for the sake of brevity.
	This completes the proof of Theorem 6.
\end{proof}

\end{document}